\documentclass{article}

\PassOptionsToPackage{numbers}{natbib}

\usepackage[final]{neurips_2020}

\usepackage[utf8]{inputenc} 
\usepackage[T1]{fontenc}    
\usepackage{hyperref}       
\usepackage{url}            
\usepackage{booktabs}       
\usepackage{amsfonts}       
\usepackage{nicefrac}       
\usepackage{microtype}      

\usepackage{amsmath}
\usepackage{amssymb}
\usepackage{amsthm}
\usepackage{enumerate}
\usepackage{graphicx}
\usepackage{xparse}
\usepackage{xcolor}

\DeclareFontFamily{U}{mathx}{\hyphenchar\font45}
\DeclareFontShape{U}{mathx}{m}{n}{
	<5> <6> <7> <8> <9> <10>
	<10.95> <12> <14.4> <17.28> <20.74> <24.88>
	mathx10
}{}
\DeclareSymbolFont{mathx}{U}{mathx}{m}{n}
\DeclareFontSubstitution{U}{mathx}{m}{n}
\DeclareMathAccent{\widecheck}{0}{mathx}{"71}
\DeclareMathAccent{\wideparen}{0}{mathx}{"75}

\ExplSyntaxOn
\keys_define:nn { miguel/label }
{
	label   .tl_set:N = \l_miguel_label_tl,
	unknown .code:n   = \clist_put_right:Nx \l_miguel_label_clist
	{ \l_keys_key_tl = \exp_not:n { #1 } }
}
\clist_new:N \l_miguel_label_clist
\box_new:N \l_miguel_label_box
\box_new:N \l_miguel_label_image_box
\NewDocumentCommand{\xincludegraphics}{O{}m}
{
	\tl_clear:N \l_miguel_label_tl
	\clist_clear:N \l_miguel_label_clist
	\keys_set:nn { miguel/label } { #1 }
	\tl_if_empty:NTF \l_miguel_label_tl
	{
		\miguel_includegraphics:Vn \l_miguel_label_clist { #2 }
	}
	{
		\hbox_set:Nn \l_miguel_label_image_box
		{
			\miguel_includegraphics:Vn \l_miguel_label_clist { #2 }
		}
		\hbox_set:Nn \l_miguel_label_box
		{
			\skip_horizontal:n { -10pt }
			\fcolorbox{white}{white}{\footnotesize \tl_use:N \l_miguel_label_tl}
		}
		\leavevmode
		\box_use:N \l_miguel_label_image_box
		\skip_horizontal:n { -\box_wd:N \l_miguel_label_image_box }
		\hbox_overlap_right:n
		{
			\box_move_up:nn
			{
				\box_ht:N \l_miguel_label_image_box - 
				\box_ht:N \l_miguel_label_box - -3pt
			}
			{ \box_use_drop:N \l_miguel_label_box }
		}
		\skip_horizontal:n { \box_wd:N \l_miguel_label_image_box }
	}
}
\cs_new_protected:Nn \miguel_includegraphics:nn
{
	\includegraphics[#1]{#2}
}
\cs_generate_variant:Nn \miguel_includegraphics:nn { V }
\ExplSyntaxOff

\newtheorem{theorem}{Theorem}[section]
\newtheorem{lemma}[theorem]{Lemma}
\newtheorem{proposition}[theorem]{Proposition}
\newtheorem{corollary}[theorem]{Corollary}
\theoremstyle{definition}
\newtheorem{definition}[theorem]{Definition}
\newtheorem{assumption}[theorem]{Assumption}

\newcommand{\R}{\mathbb{R}}
\newcommand{\C}{\mathbb{C}}
\newcommand{\N}{\mathcal{N}}
\newcommand{\E}{\mathbb{E}}
\newcommand{\Id}{\operatorname{Id}}
\newcommand{\CK}{\text{CK}}
\newcommand{\NTK}{\text{NTK}}
\newcommand{\MP}{\text{MP}}
\newcommand{\prev}{{\text{prev}}}

\newcommand{\Tr}{\operatorname{Tr}}
\newcommand{\tr}{\operatorname{tr}}

\newcommand{\x}{\mathbf{x}}
\newcommand{\y}{\mathbf{y}}
\newcommand{\z}{\mathbf{z}}
\renewcommand{\u}{\mathbf{u}}
\renewcommand{\v}{\mathbf{v}}
\newcommand{\w}{\mathbf{w}}
\renewcommand{\r}{\mathbf{r}}
\newcommand{\s}{\mathbf{s}}
\newcommand{\q}{\mathbf{q}}
\renewcommand{\a}{\alpha}
\renewcommand{\b}{\beta}
\newcommand{\eps}{\varepsilon}
\newcommand{\veps}{\widecheck{\varepsilon}}
\newcommand{\vd}{\check{d}}
\newcommand{\vB}{\widecheck{B}}
\newcommand{\wM}{\widetilde{M}}
\newcommand{\vX}{\widecheck{X}}
\newcommand{\wX}{\widetilde{X}}
\newcommand{\bX}{\mathbf{X}}

\renewcommand{\P}{\mathbb{P}}
\newcommand{\vx}{\widecheck{\mathbf{x}}}
\newcommand{\wx}{\tilde{\mathbf{x}}}
\newcommand{\1}{\mathbf{1}}

\newcommand{\Var}{\operatorname{Var}}

\newcommand{\diag}{\operatorname{diag}}
\renewcommand{\vec}{\operatorname{vec}}

\newcommand{\limspec}{\operatorname{lim\;spec}}

\renewcommand{\Im}{\operatorname{Im}}
\renewcommand{\Re}{\operatorname{Re}}
\newcommand{\cE}{\mathcal{E}}

\title{Spectra of the Conjugate Kernel and Neural Tangent Kernel for
	Linear-Width Neural Networks}

\author{%
	Zhou Fan\\
	Department of Statistics and Data Science\\
	Yale University\\
	\texttt{zhou.fan@yale.edu}\\
	\And
	Zhichao Wang\\
	Department of Mathematics\\
	University of California, San Diego\\
	\texttt{zhw036@ucsd.edu}\\
}

\begin{document}
	
\maketitle

\begin{abstract}
We study the eigenvalue distributions of the Conjugate Kernel and Neural
Tangent Kernel associated to multi-layer feedforward neural networks. In an
asymptotic regime where network width is increasing linearly in sample size,
under random initialization of the weights, and for input samples
satisfying a notion of approximate pairwise orthogonality, we show that the
eigenvalue distributions of the CK and NTK converge to deterministic limits. The
limit for the CK is described by iterating the Marcenko-Pastur map across the
hidden layers. The limit for the NTK is equivalent to that of a linear
combination of the CK matrices across layers, and may be described by
recursive fixed-point equations that extend this Marcenko-Pastur map.
We demonstrate the agreement of these asymptotic predictions with the observed
spectra for both synthetic and CIFAR-10 training data, and we perform a small
simulation to investigate the evolutions of these spectra over training.
\end{abstract}

\section{Introduction}\label{sec:introduction}

Recent progress in our theoretical understanding of neural networks has
connected their training and generalization to
two associated kernel matrices. The first is the \emph{Conjugate Kernel (CK)}
or the equivalent Gaussian process kernel
\cite{neal1995bayesian,williams1997computing,cho2009kernel,daniely2016toward,poole2016exponential,schoenholz2017deep,lee2018deep,matthews2018gaussian}.
This is the gram matrix of the derived features produced by the
final hidden layer of the network. The network predictions are linear in
these derived features, and the CK governs
training and generalization in this linear model.

The second is the \emph{Neural Tangent Kernel (NTK)}
\cite{jacot2018neural,du2019gradienta,allen2019convergence}.
This is the gram matrix of the Jacobian of in-sample predictions with
respect to the network weights, and was introduced to study full network
training. Under gradient-flow training dynamics, the in-sample predictions 
follow a differential equation governed by the NTK. We provide a brief
review of these matrices in Section \ref{sec:model}.

The spectral decompositions of these kernel matrices are related to
training and generalization properties of the underlying network.
Training occurs most rapidly along the
eigenvectors of the largest eigenvalues \cite{advani2017high}, and the
eigenvalue distribution may determine the trainability of the model and the
extent of implicit bias towards simpler
functions \cite{xiao2019disentangling,yang2019fine}. It is thus of interest to
understand the spectral properties of these matrices, both at random
initialization and over the course of training.

\subsection{Summary of contributions}

In this work, we apply techniques of random matrix theory to derive an exact
asymptotic characterization of the eigenvalue distributions of the CK and NTK at
random initialization, in a multi-layer feedforward network architecture. We
study a ``linear-width'' asymptotic regime, where each hidden layer has width
proportional to the training sample size. We impose an assumption
of approximate pairwise orthogonality for the training samples, which
encompasses general settings of independent samples that need not have
independent entries.

We show that the eigenvalue distributions
for both the CK and the NTK converge to deterministic limits, depending on the
limiting eigenvalue distribution of the training data. The limit distribution
for the CK at each intermediate hidden layer is a Marcenko-Pastur map of a
linear transformation of that of the previous layer. The NTK can be approximated
by a linear combination of CK matrices, and its limiting eigenvalue distribution can be described by a recursively defined sequence of fixed-point equations that extend
this Marcenko-Pastur map. We demonstrate the agreement of these asymptotic
limits with the observed spectra on both synthetic and CIFAR-10 training data of moderate size.

In this linear-width asymptotic regime, feature learning occurs, and both the
CK and NTK evolve over training.
Although our theory pertains only to their spectra at random
initialization of the weights, we conclude with an empirical examination of
their spectral evolutions during training,
on simple examples of learning a single neuron and learning a
binary classifier for two classes in CIFAR-10. In these examples, the bulk
eigenvalue distributions of the CK and NTK undergo elongations, and
isolated principal components emerge that are highly predictive of the 
training labels. Recent theoretical work has studied the evolution of the NTK 
in an entrywise sense \cite{huang2019dynamics,dyer2019asymptotics}, and
we believe it is an interesting open question to translate this understanding
to a more spectral perspective.

\subsection{Related literature}

Many properties of the CK and NTK have been established in the limit of infinite
width and fixed sample size $n$. In this limit, both the CK
\cite{neal1995bayesian,williams1997computing,daniely2016toward,lee2018deep,matthews2018gaussian}
and the NTK \cite{jacot2018neural,lee2019wide,yang2019scaling} at random
initialization converge to fixed $n \times n$ kernel matrices. The associated
random features regression models converge to kernel linear regression in the
RKHS of these limit kernels. Furthermore, network training occurs
in a ``lazy'' regime \cite{chizat2019lazy}, where the NTK remains constant
throughout training
\cite{jacot2018neural,du2019gradienta,du2019gradientb,allen2019convergence,lee2019wide,arora2019exact}.
Spectral properties of the CK, NTK, and Hessian of the
training loss have been previously studied in this infinite-width limit in
\cite{poole2016exponential,sagun2018empirical,xiao2019disentangling,karakida2019universal,geiger2019jamming,jacot2019asymptotic}.
Limitations of lazy training and these equivalent kernel 
regression models have been studied theoretically and empirically in
\cite{chizat2019lazy,arora2019exact,yehudai2019power,ghorbani2019limitations,ghorbani2019linearized,liang2019risk},
suggesting that trained neural networks of practical width are not fully
described by this type of infinite-width kernel equivalence. The asymptotic
behavior is different in the linear-width regime of this work:
For example, for a linear activation $\sigma(x)=x$, the infinite-width limit
of the CK for random weights is the input Gram matrix
$X^\top X$, whereas its limit spectrum under linear-width asymptotics has an
additional noise component from iterating the Marcenko-Pastur map.

Under linear-width asymptotics, the limit CK spectrum for
one hidden layer was characterized in
\cite{pennington2017nonlinear} for training data with i.i.d.\ Gaussian entries.
For activations satisfying $\E_{\xi \sim \N(0,1)}[\sigma'(\xi)]=0$,
\cite{pennington2017nonlinear} conjectured that this limit is a
Marcenko-Pastur law also in multi-layer networks, and this was proven under a
subgaussian assumption as part of the results of \cite{benigni2019eigenvalue}.
\cite{louart2018random} studied the one-hidden-layer CK
with general training data, and
\cite{liao2018spectrum} specialized this
to Gaussian mixture models. These works \cite{louart2018random,liao2018spectrum}
showed that the limit spectrum is a Marcenko-Pastur map of the inter-neuron
covariance.
We build on this insight by analyzing this covariance across multiple layers,
under approximate orthogonality of the training samples. 
This orthogonality condition is similar to that of \cite{adlam2019random},
which recently studied the one-hidden-layer CK with a bias term.
This condition is also more general than the assumption of i.i.d.\ entries,
and we describe in Appendix \ref{appendix:penningtonreduction}
the reduction to the one-hidden-layer result
of \cite{pennington2017nonlinear} for i.i.d.\ Gaussian inputs,
as this reduction is not immediately clear. \cite{peche2019note} provides
another form of the limit distribution in
\cite{pennington2017nonlinear}, which is equivalent to our form in
Appendix \ref{appendix:penningtonreduction} via the relation described in
\cite{benaych2010surprising}.

The limit NTK spectrum for a one-hidden-layer network with i.i.d.\ Gaussian
inputs was recently characterized in parallel work of
\cite{adlam2020neural}. In particular, \cite{adlam2020neural} applied the same
idea as in Lemma \ref{lemma:NTKapprox} below to study the Hadamard product
arising in the NTK. \cite{pennington2017geometry,pennington2018spectrum}
previously studied the equivalent spectrum of 
a sample covariance matrix derived from the network Jacobian, which is one of
two components of the Hessian of the training loss, in a slightly different
setting and also for one hidden layer.

The spectra of the kernel matrices $X^\top X$ that we study are equivalent
(up to the addition/removal of 0's) to the spectra of the sample covariance
matrices in linear regression using the features $X$. As developed in a line
of recent literature including \cite{dicker2016ridge,pennington2017nonlinear,dobriban2018high,louart2018random,liao2018dynamics,hastie2019surprises,mei2019generalization,adlam2020neural,dascoli2020double}, this spectrum and the associated Stieltjes transform and resolvent are
closely related to the training and generalization errors in this linear
regression model. These works have collectively
provided an asymptotic understanding of training and generalization error for
random features regression models derived from the CK and NTK of
one-hidden-layer neural networks, and related 
qualitative phenomena of double and multiple descent in the generalization
error curves. 


\section{Background}\label{sec:background}

\subsection{Neural network model and kernel matrices}\label{sec:model}

We consider a fully-connected, feedforward neural network with input dimension
$d_0$, hidden layers of dimensions $d_1,\ldots,d_L$, and a scalar output.
For an input $\x \in \R^{d_0}$, we parametrize the network as
\begin{equation}\label{eq:NNfunc}
f_\theta(\x)=\w^\top \frac{1}{\sqrt{d_L}} \sigma\bigg(
W_L \frac{1}{\sqrt{d_{L-1}}}\sigma \Big(\ldots
\frac{1}{\sqrt{d_2}}\sigma\Big(W_2\frac{1}{\sqrt{d_1}} \sigma(W_1 \x)
\Big)\Big)\bigg) \in \R.
\end{equation}
Here, $\sigma:\R \to \R$ is the activation function (applied entrywise) and
\[W_\ell \in \R^{d_\ell \times d_{\ell-1}} \quad
\text{ for } 1 \leq \ell \leq L, \qquad
\w \in \R^{d_L}\]
are the network weights. We denote by $\theta=(\vec(W_1),\ldots,\vec(W_L),\w)$ the weights across all layers. The scalings by
$1/\sqrt{d_\ell}$ reflect the ``NTK-parametrization'' of the network \cite{jacot2018neural}. We discuss
alternative scalings and an extension to 
multi-dimensional outputs in Section \ref{sec:extensions}.

Given $n$ training samples $\x_1,\ldots,\x_n \in \R^{d_0}$, we
denote the matrices of inputs and post-activations by
\[X \equiv X_0
=\begin{pmatrix} \x_1 & \dots & \x_n \end{pmatrix} \in \R^{d_0 \times n},
\quad X_\ell=\frac{1}{\sqrt{d_\ell}}
\sigma\left(W_\ell X_{\ell-1}\right) \in
\R^{d_\ell \times n}\quad \text{ for } 1\leq\ell\leq L.\]
Then the in-sample predictions of the network are given by
$f_\theta(X)=(f_\theta(\x_1),\ldots,f_\theta(\x_n))
=\w^\top X_L \in \R^{1 \times n}$. 
The {\bf Conjugate Kernel (CK)} is the matrix
\[K^\CK=X_L^\top X_L \in \R^{n \times n}.\]
More generally, we will call $X_\ell^\top X_\ell$ the conjugate kernel at the
intermediate layer $\ell$. Fixing the matrix $X_L$, the CK governs
training 
and generalization in the linear regression model $\y=\w^\top X_L$. For very
wide networks, $K^\CK$ may be viewed as an approximation of its infinite-width
limit,\footnote{In this paper, we
use ``conjugate kernel'' and ``neural tangent kernel'' to refer to these
matrices for a finite-width network, rather than their infinite-width limits.}
and regression using $X_L$ is an approximation of regression in the RKHS
defined by this limit kernel \cite{rahimi2008random}.

We denote the Jacobian matrix
of the network predictions with respect to the weights $\theta$ as
\[J=\nabla_\theta f_\theta(X)=
\begin{pmatrix} \nabla_\theta f(\x_1) & \cdots & \nabla_\theta f(\x_n)
\end{pmatrix} \in \R^{\dim(\theta) \times n}.\]
The {\bf Neural Tangent Kernel (NTK)} is the matrix
\begin{equation}\label{eq:NTK}
K^{\NTK}=J^\top J=\big(\nabla_\theta f_\theta(X)\big)^\top
\big(\nabla_\theta f_\theta(X)\big) \in \R^{n \times n}.
\end{equation}
Under gradient-flow training of the network weights $\theta$ with
training loss $\|\y-f_\theta(X)\|^2/2$,
the time evolutions of residual errors and in-sample predictions are given by
\begin{equation}\label{eq:training}
\frac{d}{dt}\Big(\y-f_{\theta(t)}(X)\Big)=-K^\NTK(t) \cdot
\Big(\y-f_{\theta(t)}(X)\Big), \quad
\frac{d}{dt}f_{\theta(t)}(X)
=K^\NTK(t) \cdot \Big(\y-f_{\theta(t)}(X)\Big)
\end{equation}
where $\theta(t)$ and $K^\NTK(t)$ are the parameters and NTK at
training time $t$ \cite{jacot2018neural,du2019gradienta}.
Denoting the eigenvalues and eigenvectors of $K^\NTK(t)$
by $(\lambda_\a(t),\v_\a(t))_{\a=1}^n$, and the spectral components of the
residual error by $r_\a(t)=\v_\a(t)^\top (\y-f_{\theta(t)}(X))$,
these training dynamics are expressed spectrally as
\[\v_\a(t)^\top \frac{d}{dt}
\Big(\y-f_{\theta(t)}(X)\Big)=-\lambda_\a(t)r_\a(t), \qquad
\frac{d}{dt}f_{\theta(t)}(X)=
\sum_{\a=1}^n \lambda_\a(t)r_\a(t) \cdot \v_\a(t).\]
Note that these relations hold instantaneously at each training time $t$,
regardless of whether $K^\NTK(t)$ evolves or remains approximately constant over
training. Hence, $\lambda_\a(t)$ controls the instantaneous rate of decay
of the residual error in the direction of $\v_\a(t)$.

For very wide networks, 
$K^\NTK$, $\lambda_\a$, and $\v_\a$ are all approximately constant over the
entirety of training
\cite{jacot2018neural,du2019gradienta,du2019gradientb,allen2019convergence,chizat2019lazy}.
This yields the closed-form solution
$r_\a(t) \approx r_\a(0)e^{-t \lambda_\a}$,
so that the in-sample predictions $f_{\theta(t)}(X)$ converge 
exponentially fast to the observed training labels $\y$, with a different
exponential rate $\lambda_\a$ along each eigenvector $\v_\a$ of $K^\NTK$.

\subsection{Eigenvalue distributions, Stieltjes transforms,
and the Marcenko-Pastur map}\label{sec:stieltjes}

We will derive almost-sure weak limits for the empirical eigenvalue
distributions of random symmetric kernel matrices $K \in \R^{n \times n}$
as $n \to \infty$. Throughout this paper, we will denote this as
\[\limspec K=\mu\]
where $\mu$ is the limit probability distribution on $\R$.
Letting $\{\lambda_\a\}_{\a=1}^n$ be the eigenvalues of $K$, this means
\begin{equation}\label{eq:weakconverge}
\frac{1}{n}\sum_{\a=1}^n f(\lambda_\a) \to \E_{\lambda \sim \mu} [f(\lambda)]
\end{equation}
a.s.\ as $n \to \infty$, for any continuous bounded
function $f:\R \to \R$.
Intuitively, this may be understood as the convergence
of the ``bulk'' of the eigenvalue distribution
of $K$.\footnote{We caution that this does not imply convergence of the
largest and smallest eigenvalues of $K$ to the support of $\mu$, which is a
stronger notion of convergence than what we study in this work.}
We will also show that $\|K\| \leq C$ a.s., for a constant $C>0$ and all large
$n$. Then (\ref{eq:weakconverge}) in fact holds for any continuous
function $f:\R \to \R$, as such a function must be bounded on $[-C,C]$.

We will characterize the probability distribution $\mu$ and the empirical
eigenvalue distribution of $K$ by their Stieltjes
transforms. These are defined, respectively, for a spectral
argument $z \in \C^+$ as\footnote{Note that some authors use a negative sign
convention and define $m_\mu(z)$ as $\int 1/(z-x)d\mu(x)$.}
\[m_\mu(z)=\int \frac{1}{x-z}d\mu(x), \qquad
m_K(z)=\frac{1}{n}\sum_{\a=1}^n \frac{1}{\lambda_\a-z}
=\frac{1}{n}\Tr (K-z\Id)^{-1}.\]
The pointwise convergence $m_K(z) \to m_\mu(z)$ a.s.\ over $z \in \C^+$
implies $\limspec K=\mu$.
For $z=x+i\eta \in \C^+$, the value $\pi^{-1}\Im m_\mu(z)$ is
the density function of the convolution of $\mu$ with the distribution
$\operatorname{Cauchy}(0,\eta)$ at $x \in \R$. Hence, the function
$m_\mu(z)$ uniquely defines $\mu$, and evaluating $\pi^{-1}\Im m_\mu(x+i\eta)$
for small $\eta>0$ yields an approximation for the density function
of $\mu$ (provided this density exists at $x$).

An example of this type of characterization is given by the
\emph{Marcenko-Pastur map}, which describes the spectra of sample covariance
matrices \cite{marchenko1967distribution}:
Let $X \in \R^{d \times n}$ have i.i.d.\ $\N(0,1/d)$ entries, let
$\Phi \in \R^{n \times n}$ be deterministic and
positive semi-definite, and let $n \to \infty$
such that $\limspec \Phi=\mu$ and $n/d \to \gamma \in (0,\infty)$. Then
the sample covariance matrix $\Phi^{1/2}X^\top X\Phi^{1/2}$
has an almost sure spectral limit,
\begin{equation}\label{eq:MPmap}
\limspec\,\Phi^{1/2}X^\top X\Phi^{1/2}=\rho^\MP_\gamma \boxtimes \mu.
\end{equation}
We will call this limit $\rho^\MP_\gamma \boxtimes \mu$
the Marcenko-Pastur map of $\mu$ with aspect ratio $\gamma$.
This distribution $\rho^\MP_\gamma \boxtimes \mu$ may be defined by its
Stieltjes transform $m(z)$, which solves the Marcenko-Pastur fixed point
equation \cite{marchenko1967distribution}
\begin{equation}\label{eq:MPeq}
m(z)=\int \frac{1}{x(1-\gamma-\gamma zm(z))-z}\,d\mu(x).
\end{equation}

\section{Main results}\label{sec:results}

\subsection{Assumptions}

We use Greek indices $\a$, $\b$, etc.\ for samples in $\{1,\ldots,n\}$, and
Roman indices $i$, $j$, etc.\ for neurons in $\{1,\ldots,d\}$.
For a matrix $X \in \R^{d \times n}$,
we denote by $\x_\a$ its $\a^\text{th}$ column and by
$\x_i^\top$ its $i^\text{th}$ row. $\|\cdot\|$ is the $\ell_2$-norm for
vectors and $\ell_2 \to \ell_2$ operator norm for matrices. $\Id$ is the
identity matrix.

\begin{definition}\label{def:orthogonal}
Let $\eps,B>0$. A matrix $X \in \R^{d \times n}$ is
{\bf $(\eps,B)$-orthonormal} if its columns satisfy,
for every $\a \neq \b \in \{1,\ldots,n\}$,
\[\big|\|\x_\a\|^2-1\big| \leq \eps,
\qquad \big|\x_\a^\top \x_\b \big| \leq \eps,
\qquad \|X\| \leq B, \qquad
\sum_{\a=1}^n (\|\x_\a\|^2-1)^2 \leq B^2.\]
\end{definition}

\begin{assumption}\label{assump:asymptotics}
The number of layers $L \geq 1$ is fixed,
and $n,d_0,d_1,\ldots,d_L \to \infty$, such that
\begin{enumerate}[(a)]
\item The weights $\theta=(\vec(W_1),\ldots,\vec(W_L),\w)$ are i.i.d.\ and distributed
as $\N(0,1)$.
\item The activation $\sigma(x)$ is twice differentiable, with $\sup_{x \in \R}
|\sigma'(x)|,|\sigma''(x)| \leq \lambda_\sigma$ for some
$\lambda_\sigma<\infty$. For $\xi \sim \N(0,1)$, we have
$\E[\sigma(\xi)]=0$ and $\E[\sigma^2(\xi)]=1$.
\item The input $X \in \R^{d_0 \times n}$
is $(\eps_n,B)$-orthonormal in the sense of Definition
\ref{def:orthogonal}, where $B$ is a constant, and
$\eps_n n^{1/4} \to 0$ as $n \to \infty$.
\item As $n \to \infty$, $\limspec X^\top X=\mu_0$
for a probability distribution $\mu_0$ on $[0,\infty)$,
and $\lim n/d_\ell=\gamma_\ell$ for constants $\gamma_\ell \in (0,\infty)$
and each $\ell=1,2,\ldots,L$.
\end{enumerate}
\end{assumption}

Part (c) quantifies our assumption of approximate pairwise orthogonality of
the training samples. Although not completely general, it encompasses many
settings of independent samples with input dimension $d_0 \asymp n$,
including:
\begin{itemize}
\item Non-white Gaussian inputs $\x_\a \sim \N(0,\Sigma)$,
for any $\Sigma$ satisfying $\Tr \Sigma=1$ and $\|\Sigma\| \lesssim 1/n$.
\item Inputs $\x_\a$ drawn from certain multi-class Gaussian mixture models,
in the high-dimensional asymptotic regimes that were studied in \cite{couillet2016kernel,louart2018random,liao2018spectrum,liao2018dynamics,liao2019inner}.
\item Inputs that may be expressed as
$\sqrt{d_0} \cdot \x_\a=f(\z_\a)$, where $\z_\a \in \R^m$ has independent
entries satisfying a log-Sobolev inequality, and $f:\R^m \to \R^{d_0}$ is any
Lipschitz function.
\end{itemize}
In particular, the limit spectral law $\mu_0$ in Assumption
\ref{assump:asymptotics}(d) can be very different
from the Marcenko-Pastur spectrum that would correspond to $X$ having
i.i.d.\ entries. This approximate orthogonality is implied by the following more
technical convex concentration property, which is discussed further
in \cite{vu2015random,adamczak2015note}. We prove this result in Appendix
\ref{appendix:inputisgood}.

\begin{proposition}\label{prop:inputisgood}
Let $X=(\x_1,\ldots,\x_n) \in \R^{d_0 \times n}$,
where $\x_1,\ldots,\x_n$ are independent training samples satisfying
$\E[\x_\a]=0$ and $\E[\|\x_\a\|^2]=1$. Suppose, for some
constant $c_0>0$, that $d_0 \geq c_0n$, and each vector $\sqrt{d_0}\cdot \x_\a$
satifies the convex concentration property
\[\P\Big[\big|\varphi(\sqrt{d_0} \cdot \x_\a)-\E \varphi(\sqrt{d_0} \cdot \x_\a)
\big| \geq t\Big] \leq 2e^{-c_0t^2}\]
for every $t>0$ and every 1-Lipschitz convex function $\varphi:\R^{d_0} \to \R$.
Then for any $k>0$, with probability $1-n^{-k}$,
$X$ is $(\sqrt{\frac{C\log n}{d_0}},B)$-orthonormal
for some $C,B>0$ depending only on $c_0,k$.
\end{proposition}

In Assumptions \ref{assump:asymptotics}(a) and (b), the scaling of $\theta$
and the conditions $\E[\sigma(\xi)]=0$ and $\E[\sigma^2(\xi)]=1$, together with
the parametrization (\ref{eq:NNfunc}),
ensure that all pre-activations have approximate mean 0 and variance 1.
This may be achieved in practice by batch
normalization \cite{ioffe2015batch}.
For $\xi \sim \N(0,1)$, we define the following constants associated
to $\sigma(x)$. We verify in Proposition \ref{prop:sigmaproperties} that under
Assumption \ref{assump:asymptotics}(b), we have
$b_\sigma^2 \leq 1 \leq a_\sigma$.
\begin{equation}\label{eq:qr}
b_\sigma=\E[\sigma'(\xi)], \quad a_\sigma=\E[\sigma'(\xi)^2],
\quad q_\ell=(b_\sigma^2)^{L-\ell}, \quad r_\ell=a_\sigma^{L-\ell},
\quad r_+=\sum_{\ell=0}^{L-1} r_\ell-q_\ell.
\end{equation}

\subsection{Spectrum of the Conjugate Kernel}\label{sec:CK}

Recall the Marcenko-Pastur map (\ref{eq:MPmap}).
Let $\mu_1,\mu_2,\mu_3,\ldots$ be the sequence of probability distributions on
$[0,\infty)$ defined recursively by
\begin{equation}\label{eq:muell}
\mu_\ell=\rho^\MP_{\gamma_\ell} \boxtimes \Big((1-b_\sigma^2)
+b_\sigma^2 \cdot \mu_{\ell-1}\Big).
\end{equation}
Here, $\mu_0$ is the input limit spectrum in Assumption
\ref{assump:asymptotics}(d), $b_\sigma$ is defined in (\ref{eq:qr}),
and $(1-b_\sigma^2)+b_\sigma^2 \cdot \mu$ denotes the translation and rescaling
of $\mu$ that is the
distribution of $(1-b_\sigma^2)+b_\sigma^2 \lambda$ when $\lambda \sim \mu$.

The following theorem shows that these distributions $\mu_1,\mu_2,\mu_3,\ldots$
are the asymptotic limits of the empirical eigenvalue distributions of the CK
across the layers. Thus, the limit distribution for each layer $\ell$
is a Marcenko-Pastur map of a translation and rescaling of that of the preceding
layer $\ell-1$.

\begin{theorem}\label{thm:CK}
Suppose Assumption \ref{assump:asymptotics} holds, and
define $\mu_1,\ldots,\mu_L$ by (\ref{eq:muell}). Then (marginally) for each
$\ell=1,\ldots,L$, we have $\limspec X_\ell^\top X_\ell=\mu_\ell$. In
particular,
\[\limspec K^\CK=\mu_L.\]
Furthermore, $\|K^\CK\| \leq C$ a.s.\ for a constant $C>0$ and all large $n$.
\end{theorem}

If $\sigma(x)$ is such that $b_\sigma=0$, then each distribution
$\mu_\ell$ is simply the Marcenko-Pastur law $\rho_{\gamma_\ell}^\MP$. This
special case was previously conjectured in \cite{pennington2017nonlinear}
and proven in \cite{benigni2019eigenvalue}, for input data $X$ with
i.i.d.\ entries. Note that for such non-linearities, the limiting CK spectrum
does not depend on the spectrum $\mu_0$ of the input data, and furthermore
$\mu_1=\ldots=\mu_L$ if the layers have the same width $d_1=\ldots=d_L$.
Implications of this for the network discrimination ability in classification
tasks and for learning performance have been discussed previously in \cite{couillet2016kernel,pennington2017nonlinear,louart2018random,liao2019inner,adlam2019random}.

To connect Theorem \ref{thm:CK} to our next result on the NTK, let us describe
the iteration (\ref{eq:muell})
more explicitly using a recursive sequence of fixed-point equations derived from
the Marcenko-Pastur equation (\ref{eq:MPeq}):
Let $m_\ell(z)$ be the Stieltjes transform of $\mu_\ell$, and define
\[\tilde{t}_\ell(z_{-1},z_\ell)=\lim_{n \to \infty} \frac{1}{n}
\Tr (z_{-1}\Id+z_\ell X_\ell^\top X_\ell)^{-1}
=\frac{1}{z_\ell}m_\ell\left(-\frac{z_{-1}}{z_\ell}\right).\]
Applying the Marcenko-Pastur equation (\ref{eq:MPeq}) to
$m_\ell(-z_{-1}/z_\ell)$, and introducing
$\tilde{s}_\ell(z_{-1},z_\ell)=[z_\ell(1-\gamma_\ell+\gamma_\ell z_{-1}
\tilde{t}_\ell(z_{-1},z_\ell))]^{-1}$, one may check that (\ref{eq:muell})
may be written as the pair of equations
\begin{align}
\tilde{t}_\ell(z_{-1},z_\ell)&=\tilde{t}_{\ell-1}\bigg(
z_{-1}+\frac{1-b_\sigma^2}{\tilde{s}_\ell(z_{-1},z_\ell)},\;
\frac{b_\sigma^2}{\tilde{s}_\ell(z_{-1},z_\ell)}\bigg),\label{eq:tildet}\\
\tilde{s}_\ell(z_{-1},z_\ell)
&=(1/z_\ell)+\gamma_\ell\Big(\tilde{s}_\ell(z_{-1},z_\ell)-z_{-1}
\tilde{s}_\ell(z_{-1},z_\ell)\tilde{t}_\ell(z_{-1},z_\ell)\Big),\label{eq:tildes}
\end{align}
where (\ref{eq:tildes}) is a rearrangement of the definition of
$\tilde{s}_\ell$. Applying (\ref{eq:tildet}) to substitute
$\tilde{t}_\ell(z_{-1},z_\ell)$ in (\ref{eq:tildes}), the equation (\ref{eq:tildes}) is a fixed-point equation that defines
$\tilde{s}_\ell$ in terms of $\tilde{t}_{\ell-1}$. Then (\ref{eq:tildet})
defines $\tilde{t}_\ell$ in terms of
$\tilde{s}_\ell$ and $\tilde{t}_{\ell-1}$. The limit Stieltjes transform for
$K^\CK$ is the specialization $m_\CK(z)=\tilde{t}_L(-z,1)$.

\subsection{Spectrum of the Neural Tangent Kernel}\label{sec:NTK}

In the neural network model (\ref{eq:NNfunc}), an
application of the chain rule yields an explicit form
\[K^\NTK=X_L^\top X_L+\sum_{\ell=1}^L (S_\ell^\top S_\ell) \odot
(X_{\ell-1}^\top X_{\ell-1})\]
for certain matrices $S_\ell \in \R^{d_\ell \times n}$,
where $\odot$ is the Hadamard (entrywise) product. We refer to Appendix
\ref{appendix:NTKapprox}
for the exact expression; see also \cite[Eq.\ (1.7)]{huang2019dynamics}.
Our spectral analysis of $K^\NTK$ relies on the following approximation,
which shows that the limit spectrum of $K^\NTK$ is equivalent to a linear
combination of the CK matrices $X_0^\top X_0,\ldots,X_L^\top X_L$
and $\Id$. We prove this in Appendix \ref{appendix:NTKapprox}.

\begin{lemma}\label{lemma:NTKapprox}
Under Assumption \ref{assump:asymptotics}, letting $r_+$ and $q_\ell$ be as
defined in (\ref{eq:qr}),
\[\limspec K^\NTK=\limspec
\Big(r_+\Id+X_L^\top X_L+\sum_{\ell=0}^{L-1}
q_\ell X_\ell^\top X_\ell\Big).\]
\end{lemma}

By this lemma, if $b_\sigma=0$, then $q_0=\ldots=q_{L-1}=0$ and
the limit spectrum of $K^\NTK$ reduces to
the limit spectrum of $r_+\Id+X_L^\top X_L$ which is a translation of
$\rho_{\gamma_L}^{\MP}$
described in Theorem \ref{thm:CK}. Thus we assume in the following that
$b_\sigma \neq 0$. Our next result provides an analytic description of the limit spectrum of $K^\NTK$, by extending
(\ref{eq:tildet},\ref{eq:tildes}) to characterize
the trace of rational functions of $X_0^\top X_0,\ldots,X_L^\top X_L$ and $\Id$.

Denote the closed lower-half complex plane with 0 removed as
$\C^*=\overline{\C^-} \setminus \{0\}$.
For $\ell=0,1,2,\ldots$, we define recursively
two sequences of functions
\begin{align*}
t_\ell&:(\C^- \times \R^\ell \times \C^*) \times \C^{\ell+2} \to \C,
& (\z,\w) \mapsto t_\ell(\z,\w)\\
s_\ell&:\C^- \times \R^\ell \times \C^* \to \C^+,
& \z \mapsto s_\ell(\z).
\end{align*}
where $\z=(z_{-1},z_0,\ldots,z_\ell) \in \C^- \times \R^\ell \times \C^*$
and $\w=(w_{-1},w_0,\ldots,w_\ell) \in \C^{\ell+2}$.
We will define these functions such that $t_\ell(\z,\w)$ will be the value of
\[\lim_{n \to \infty} n^{-1}\Tr
(z_{-1}\Id+z_0 X_0^\top X_0+\ldots+z_\ell
X_\ell^\top X_\ell)^{-1}(w_{-1}\Id+w_0 X_0^\top X_0+\ldots+w_\ell
X_\ell^\top X_\ell).\]
For $\ell=0$, we define the first function $t_0$ by
\begin{equation}\label{eq:t0}
t_0\Big((z_{-1},z_0),(w_{-1},w_0)\Big)
=\int \frac{w_{-1}+w_0x}{z_{-1}+z_0x} d\mu_0(x)
\end{equation}
For $\ell \geq 1$, we then define the functions
$s_\ell$ and $t_\ell$ recursively by
\begin{align}
s_\ell(\z)&=(1/z_\ell)
+\gamma_\ell t_{\ell-1}\big(\z_\prev(s_\ell(\z),\z),\,(1-b_\sigma^2,0,
\ldots,0,b_\sigma^2)\big),\label{eq:sl}\\
t_\ell(\z,\w)&=(w_\ell/z_\ell)
+t_{\ell-1}\big(\z_\prev(s_\ell(\z),\z),\,\w_\prev\big)\label{eq:tl}
\end{align}
where we write as shorthand
\begin{align}
\z_\prev(s_\ell(\z),\z)
& \equiv \left(z_{-1}+\frac{1-b_\sigma^2}{s_\ell(\z)},
z_0,\ldots,z_{\ell-2},z_{\ell-1}+\frac{b_\sigma^2}{s_\ell(\z)}\right)
\in \C^- \times \R^{\ell-1} \times \C^*,\label{eq:zprev}\\
\w_\prev& \equiv (w_{-1},\ldots,w_{\ell-1})-(w_\ell/z_\ell) \cdot 
(z_{-1},\ldots,z_{\ell-1}) \in \C^{\ell+1}.\label{eq:wprev}
\end{align}

\begin{proposition}\label{prop:swelldefined}
Suppose $b_\sigma\neq 0$. For each $\ell \geq 1$ and any $\z \in \C^- \times
\R^\ell \times \C^*$, there is a unique solution $s_\ell(\z) \in
\C^+$ to the fixed-point equation (\ref{eq:sl}).
\end{proposition}

Hence, (\ref{eq:sl}) defines the function $s_\ell$ in terms of the function
$t_{\ell-1}$, and this is then used in
(\ref{eq:tl}) to define $t_\ell$. This is illustrated diagrammatically as
\[\begin{matrix} 
t_0 & \rightarrow & t_1 & \rightarrow & t_2 &\rightarrow &\cdots\\
\downarrow& \large\nearrow & \downarrow& \large\nearrow & \downarrow& \large\nearrow & \\
s_1 &  & s_2 &  & s_3 & & \\
\end{matrix}\]


Specializing the function $t_L$ for the last layer $L$ to the values
$(z_{-1},z_0,\ldots,z_{L-1},z_L)=(r_+,q_0,\ldots,q_{L-1},1)$ and
$(w_{-1},w_0,\ldots,w_L)=(1,0,\ldots,0)$, we obtain an analytic description for
the limit spectrum of $K^\NTK$ via its Stieltjes transform.

\begin{theorem}\label{thm:NTK}
Suppose $b_\sigma\ne 0$. Under Assumption \ref{assump:asymptotics}, for any fixed values
$z_{-1},z_0,\ldots,z_L \in \R$ where $z_L \neq 0$, we have
$\limspec (z_{-1}\Id+z_0X_0^\top X_0+\ldots+z_L X_L^\top X_L)=\nu$
where $\nu$ is the probability distribution with Stieltjes transform
$m_\nu(z)=t_L((-z+z_{-1},z_0,\ldots,z_L),(1,0,\ldots,0))$.

In particular, $\limspec K^\NTK$ is the probability distribution with Stieltjes transform
\[m_\NTK(z)=t_L\Big((-z+r_+,q_0,\ldots,q_{L-1},1),(1,0,\ldots,0)\Big).\]
Furthermore, $\|K^\NTK\| \leq C$ a.s.\ for a constant $C>0$ and all large $n$.
\end{theorem}

We remark that Theorem \ref{thm:NTK} encompasses the previous result
in Theorem \ref{thm:CK} for $K^\CK=X_L^\top X_L$, by specializing to
$(z_0,\ldots,z_{L-1},z_L)=(0,\ldots,0,1)$. Under this specialization,
$s_\ell(z_{-1},0,\ldots,0,z_\ell)=\tilde{s}_\ell(z_{-1},z_\ell)$,
$t_\ell((z_{-1},0,\ldots,0,z_\ell),(1,0,\ldots,0))=\tilde{t}_\ell(z_{-1},z_\ell)$, and (\ref{eq:sl},\ref{eq:tl}) reduce to 
(\ref{eq:tildet},\ref{eq:tildes}).

\subsection{Extension to multi-dimensional outputs and rescaled
parametrizations}\label{sec:extensions}

Theorem \ref{thm:NTK} pertains to a network with scalar outputs,
under the ``NTK-parametrization'' of network weights in (\ref{eq:NNfunc}). As
neural network models used in practice often have multi-dimensional outputs and
may be parametrized differently for backpropagation, we state here the
extension of the preceding result to a network with $k$-dimensional output and a
general scaling of the weights.

Consider the model
\begin{equation}\label{eq:NNfuncmulti}
f_\theta(\x)=W_{L+1}^\top \frac{1}{\sqrt{d_L}} \sigma\bigg(
W_L \frac{1}{\sqrt{d_{L-1}}}\sigma \Big(\ldots
\frac{1}{\sqrt{d_2}}\sigma\Big(W_2\frac{1}{\sqrt{d_1}} \sigma(W_1 \x)
\Big)\Big)\bigg) \in \R^k
\end{equation}
where $W_{L+1}^\top \in \R^{k \times d_L}$. We write the coordinates of
$f_\theta$ as $(f_\theta^1,\ldots,f_\theta^k)$, and the vectorized output
for all training samples $X \in \R^{d_0 \times n}$ as
$f_\theta(X)=(f_\theta^1(X),\ldots,f_\theta^k(X)) \in \R^{nk}$. We consider
the NTK
\begin{equation}\label{eq:NTKmulti}
K^\NTK=\sum_{\ell=1}^{L+1}
\tau_\ell \Big(\nabla_{W_\ell} f_\theta(X)\Big)^\top
\Big(\nabla_{W_\ell} f_\theta(X)\Big) \in \R^{nk \times nk}.
\end{equation}
For $\tau_1=\ldots=\tau_{L+1}=1$, this is a flattening of the NTK
defined in \cite{jacot2018neural}, and we recall briefly its derivation 
from gradient-flow training in Appendix \ref{appendix:NTKmultiderivation}.
We consider general constants $\tau_1,\ldots,\tau_{L+1}>0$ to allow for
a different learning rate for each weight matrix $W_\ell$, which may arise from
backpropagation in the model (\ref{eq:NNfuncmulti}) using a
parametrization with different scalings of the weights.

\begin{theorem}\label{thm:NTKmulti}
Fix any $k \geq 1$. Suppose Assumption \ref{assump:asymptotics} holds, and $b_\sigma \neq 0$. Then $\|K^\NTK\| \leq C$ a.s.\ for a constant $C>0$
and all large $n$, and $\limspec K^\NTK$ is the probability
distribution with Stieltjes transform
\[m_\NTK(z)=t_L\Big((-z+\tau \cdot r_+,\;\tau_1q_0,\ldots,\tau_Lq_{L-1},\tau_{L+1}),(1,0,\ldots,0)\Big),
\quad \tau \cdot r_+ \equiv \sum_{\ell=0}^{L-1} \tau_{\ell+1}(r_\ell-q_\ell).\]
\end{theorem}

\section{Experiments}\label{sec:experiments}

We describe in Appendix \ref{appendix:computation} an algorithm to numerically
compute the limit spectral densities of Theorem \ref{thm:NTK}.
The computational cost is independent of the dimensions
$(n,d_0,\ldots,d_L)$, and each limit density below was computed within a few
seconds on our laptop computer. Using this procedure, we investigate
the accuracy of the theoretical predictions of Theorems \ref{thm:CK}
and \ref{thm:NTK}. Finally, we conclude by examining the spectra
of $K^\CK$ and $K^\NTK$ after network training.

\subsection{Simulated Gaussian training data}

\begin{figure}
\xincludegraphics[width=0.33\textwidth,label=a)]{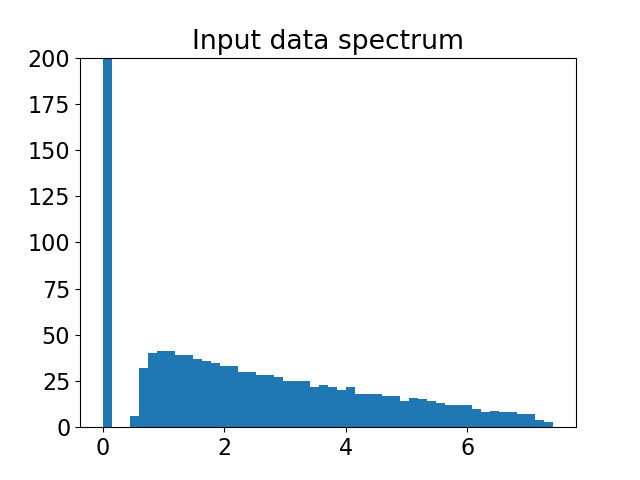}%
\xincludegraphics[width=0.33\textwidth,label=b)]{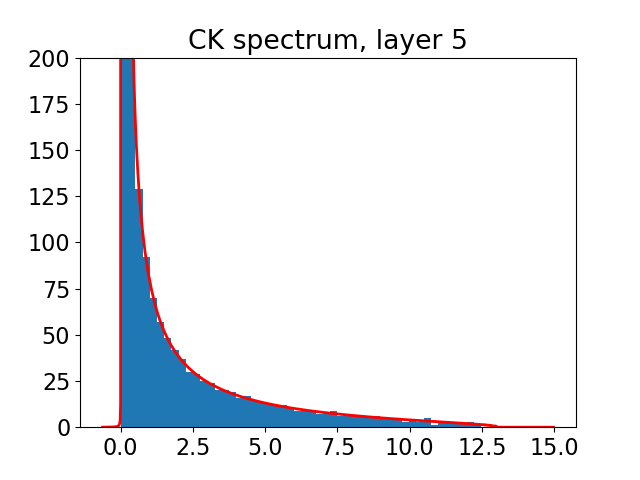}%
\xincludegraphics[width=0.33\textwidth,label=c)]{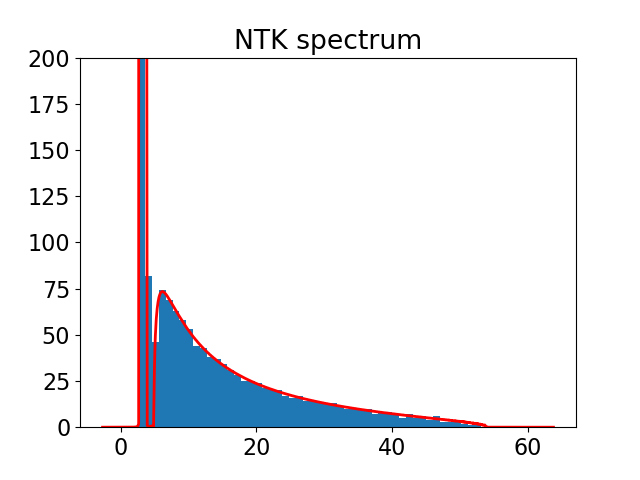}
\caption{Simulated spectra at initialization
for i.i.d.\ Gaussian training samples in a 5-layer
network, for (a) the input gram matrix $X_0^\top X_0$,
(b) $K^\CK=X_5^\top X_5$, and 
(c) $K^\NTK$. Numerical computations of the limit spectra in
Theorems \ref{thm:CK} and \ref{thm:NTK} are superimposed in red.}
\label{fig:gaussian}
\end{figure}

We consider $n=3000$ training samples with i.i.d.\ $\N(0,1/d_0)$ entries,
input dimension $d_0=1000$, and $L=5$ hidden layers of dimensions
$d_1=\ldots=d_5=6000$. We take $\sigma(x) \propto \tan^{-1}(x)$, normalized
so that $\E[\sigma(\xi)^2]=1$. A close agreement between the observed and limit
spectra is displayed in Figure \ref{fig:gaussian},
for both $K^\CK$ and $K^\NTK$ at initialization. The CK spectra for intermediate layers are depicted in Appendix
\ref{appendix:alllayers}.

We highlight two qualitative phenomena:
The spectral distribution of the NTK (at initialization) is separated from 0, as
explained by the $\Id$ component in Lemma \ref{lemma:NTKapprox}. Across layers $\ell=1,\ldots,L$, there is a merging of the spectral bulk
components of the CK, and an extension of its spectral support.

\subsection{CIFAR-10 training data}\label{sec:CIFAR}

\begin{figure}
\xincludegraphics[width=0.33\textwidth,label=a)]{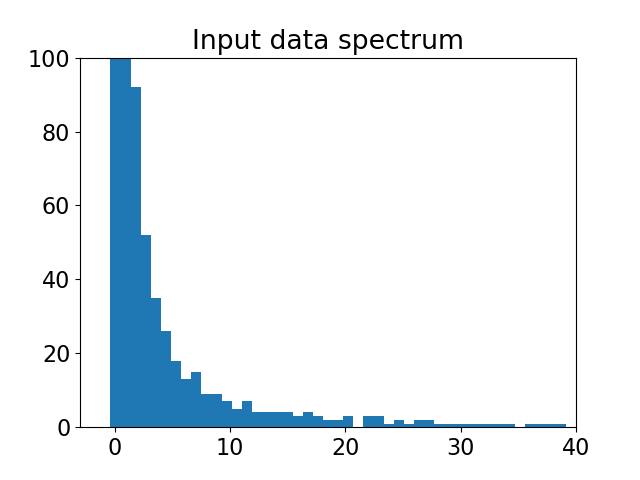}%
\xincludegraphics[width=0.33\textwidth,label=b)]{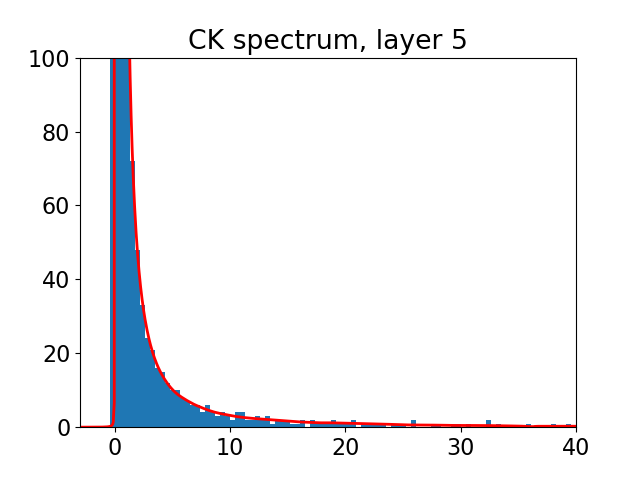}%
\xincludegraphics[width=0.33\textwidth,label=c)]{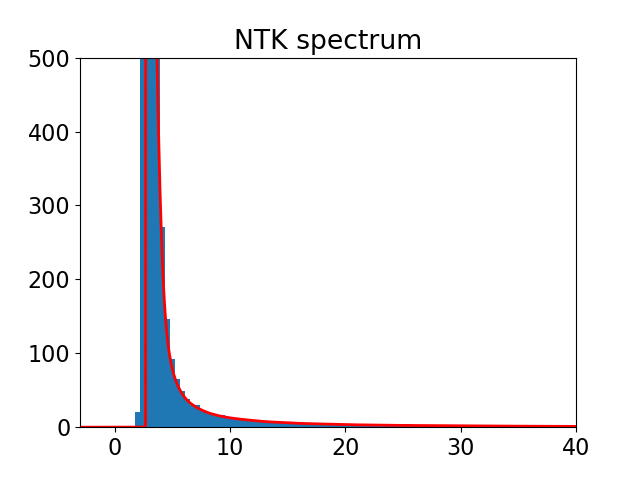}
\caption{Same plots as Figure \ref{fig:gaussian}, for 5000 training 
samples from CIFAR-10 with 10 leading PCs removed.}
\label{fig:CIFAR}
\end{figure}

We consider $n=5000$ samples randomly selected from the CIFAR-10
training set \cite{krizhevsky2009learning}, with
input dimension $d_0=3072$, and $L=5$ hidden layers of
dimensions $d_1=\ldots=d_5=10000$. Strong principal component
structure may cause the training samples to have large pairwise inner-products, which is shown in Appendix \ref{appendix:orthogonality_for_data}.
Thus, we pre-process the training samples by removing the leading 10 PCs---a few
example images before and after this removal are depicted in
Appendix \ref{appendix:CIFAR_images}.
A close agreement between the observed and limit
spectra is displayed in Figure \ref{fig:CIFAR}, for both $K^\CK$ and $K^\NTK$.
Results without removing these leading 10 PCs are
presented in Appendix \ref{appendix:CIFARraw}, where there is close
agreement for $K^\CK$ but a deviation from the theoretical prediction
for $K^\NTK$. This suggests that the approximation in
Lemma \ref{lemma:NTKapprox} is sensitive to large but low-rank
perturbations of $X$. 

\subsection{CK and NTK spectra after training}\label{sec:training}

\begin{figure}
\xincludegraphics[width=0.33\textwidth,label=a)]{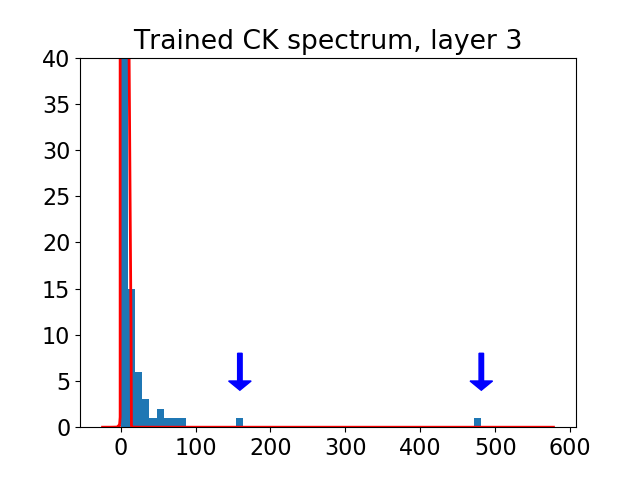}%
\xincludegraphics[width=0.33\textwidth,label=b)]{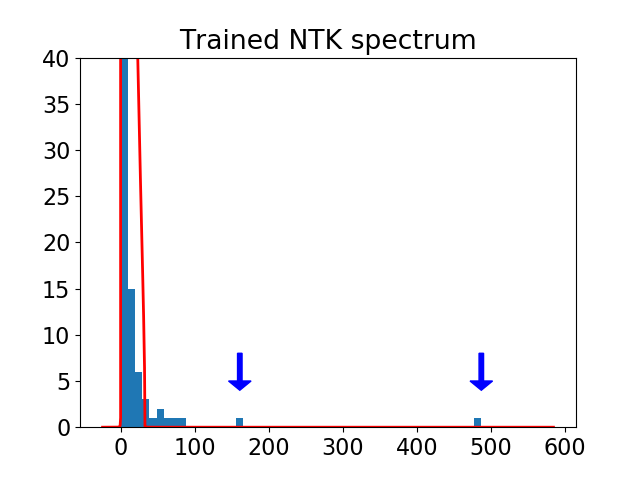}%
\xincludegraphics[width=0.33\textwidth,label=c)]{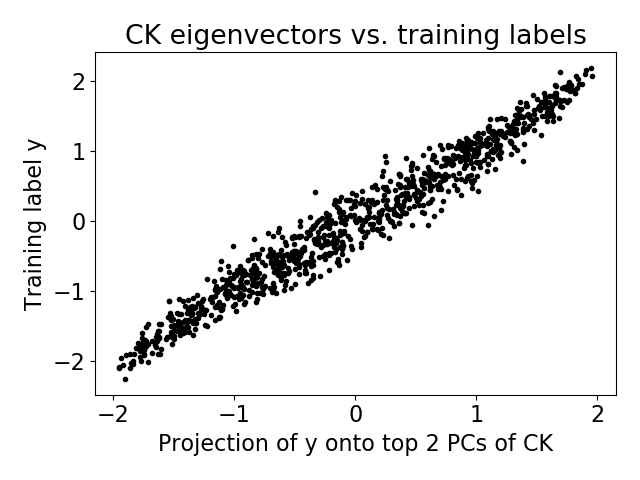}
\caption{Eigenvalues of (a) $K^\CK$ and (b) $K^\NTK$ in a \emph{trained}
network, for training labels $y_\a=\sigma(\x_\a^\top \v)$. The limit spectra at
random initialization of weights are shown in red. Large outlier eigenvalues,
indicated by blue arrows, emerge over training. (c) The projection of training
labels onto the first 2 eigenvectors of the trained matrix $K^\CK$ accounts
for 96\% of the training label variance.}\label{fig:training}
\end{figure}

We consider $n=1000$ training samples $(\x_\a,y_\a)$, with $\x_\a$ uniformly
distributed on the unit sphere of dimension $d_0=800$,
and $y_\a=\sigma(\x_\a^\top \v)$ for a fixed $\v \in \R^{d_0}$ on the sphere of
radius $\sqrt{d_0}$. We train a 3-layer network with widths $d_1=d_2=d_3=800$,
without biases, using the Adam optimizer in Keras with learning rate $0.01$,
batch size 32, and 300 training epochs. The final mean-squared training
error is $10^{-4}$, and the test-sample prediction-$R^2$ is 0.81.

Figure \ref{fig:training} depicts the spectra of $K^\CK$ and $K^\NTK$ for the 
trained weights $\theta$. Intermediate layers are shown in Appendix
\ref{appendix:alllayers}. We observe that the bulk spectra of $K^\CK$ and
$K^\NTK$ are elongated from their random initializations.
Furthermore, large outlier eigenvalues emerge in both $K^\CK$ and $K^\NTK$
over training. The corresponding eigenvectors are highly predictive of the
training labels $\y$, suggesting the emergence of these eigenvectors
as the primary mechanism of training in this example.

We describe in Appendix \ref{appendix:CIFARtraining} a second
training example for a binary classification task on CIFAR-10, where similar
qualitative phenomena are observed for the trained $K^\CK$. This may suggest a
path to understanding the learning process of deep neural networks, for
future study.

\section{Conclusion}

We have provided analytic descriptions of the empirical
eigenvalue distributions of the Conjugate Kernel (CK) and Neural Tangent Kernel
(NTK) of large feedforward neural networks at random initialization,
under a general condition for the input samples. Our work uses techniques of
random matrix theory to provide an asymptotic analysis in a limiting regime
where network width grows linearly with sample size. The resulting limit
spectra exhibit ``high-dimensional noise'' that is not present in analyses
of the infinite-width limit alone. This type of high-dimensional limit has
been previously studied for networks with a single hidden layer, and our work
develops new proof techniques to extend these characterizations to multi-layer
networks, in a systematic and recursive form.

Our results contribute to the theoretical understanding of neural networks
in two ways: First, an increasingly large body of literature studies the
training and generalization errors of linear regression models using random
features derived from the neural network CK and NTK. In the linear-width setting
of our current paper, such results are typically based on asymptotic
approximations for the Stieltjes transforms and resolvents of the associated
kernel and covariance matrices. Our work develops theoretical tools that may
enable the extension of these studies to random features regression models that 
are derived from deep networks with possibly many layers.

Second, the linear-width asymptotic regime may provide a simple
setting for studying feature learning and neural network training
outside of the ``lazy'' regime, and which is arguably closer to the operating
regimes of neural network models in some practical applications. Our
experimental results suggest interesting phenomena in the spectral evolutions of
the CK and NTK that may potentially arise during training in this
regime, and our theoretical characterizations of their spectra for random
weights may provide a first step towards the analysis of these phenomena.

\section*{Broader Impact}

This work performs theoretical analysis that aims to extend our understanding
of training and generalization in multi-layer neural networks. A better
theoretical understanding of training and generalization in these models may
ultimately help us to (1) understand the mechanisms
by which social biases may be propagated by artificial systems, and prevent
this from occurring, and (2) increase the robustness and fault-tolerance of
artificial systems built on such models.

\begin{ack}
This research is supported in part by NSF Grant DMS-1916198. We would like to
thank John Lafferty and Ganlin Song for helpful discussions regarding
the Neural Tangent Kernel.
\end{ack}

{\small

\bibliographystyle{plain}
\bibliography{NTKspectrum}}

\newpage

\appendix

\section{Numerical solution of the fixed-point equations}\label{appendix:computation}

Theorem \ref{thm:NTK} characterizes the limit Stieltjes transform $m(z)$
of matrices such as $K^\CK$ and $K^\NTK$. By the
discussion in Section \ref{sec:stieltjes}, a numerical
approximation to the density functions of the corresponding 
spectral distributions may be obtained by computing $m(z)$ for $z=x+i\eta$,
across a fine grid of values $x \in \R$ and for a fixed small imaginary part
$\eta>0$. We describe here one possible approach for this computation.

To compute the limit spectrum for
$z_{-1}\Id+z_0X_0^\top X_0+\ldots+z_LX_L^\top X_L$
and general values $z_{-1},\ldots,z_L \in \R$, fix the spectral argument
$z=x+i\eta$ and denote
\[\z_L=(-z+z_{-1},z_0,\ldots,z_L),\; \z_{L-1}=\z_\prev(s_L(\z_L),\z_L),
\;\z_{L-2}=\z_\prev(s_{L-1}(\z_{L-1}),\z_{L-1}), \text{ etc.}\]
Here, for $s \in \C^+$ and $\z \in \C^- \times \R^\ell \times \C^*$, the
quantity
\[\z_\prev(s,\z)=\left(z_{-1}+\frac{1-b_\sigma^2}{s},z_0,\ldots,
z_{\ell-2},z_{\ell-1}+\frac{b_\sigma^2}{s}\right)
\in \C^- \times \R^{\ell-1} \times \C^*\]
is as defined in (\ref{eq:zprev}).
Denote $s_\ell \equiv s_\ell(\z_\ell)$ for each $\ell=1,\ldots,L$.
Observe that, if we are given $s_1,\ldots,s_L$, then the value
$t_\ell(\z_\ell,\w)$ may be directly computed from (\ref{eq:tl}),
for any $\ell \in \{0,\ldots,L\}$
and any vector $\w \in \C^{\ell+2}$. This is because
the fixed points needed to compute the arguments
$\z_\prev(s_\ell(\z_\ell),\z_\ell)$,
$\z_\prev(s_{\ell-1}(\z_{\ell-1}),\z_{\ell-1})$, etc.\ for the successive
evaluations of $t_\ell$, $t_{\ell-1}$, etc.\ are provided by this given
sequence $s_1,\ldots,s_L$.

Thus, we apply an iterative procedure of initializing
$s_1^{(0)},\ldots,s_L^{(0)} \in \C^+$, and computing the \emph{simultaneous}
updates $s_1^{(t+1)},\ldots,s_L^{(t+1)}$ using the previous values
$s_1^{(t)},\ldots,s_L^{(t)}$. That is, we iterate the following two steps:
\begin{enumerate}
\item Set $\z_L^{(t)}=\z_L$, and compute
$\z_{L-1}^{(t)}=\z_\prev(s_L^{(t)},\z_L^{(t)})$,
$\z_{L-2}^{(t)}=\z_\prev(s_{L-1}^{(t)},\z_{L-1}^{(t)})$, etc.
\item Compute an update $s_\ell^{(t+1)}$ for the value of $s_\ell(\z_\ell)$
and each $\ell=1,\ldots,L$,
using the right side of (\ref{eq:sl}) with $\z_\ell^{(t)}$
and $\z_{\ell-1}^{(t)} \equiv \z_\prev(s_\ell^{(t)},\z_\ell^{(t)})$
in place of $\z_\ell$ and $\z_\prev(s_\ell(\z_\ell),\z_\ell)$.
\end{enumerate}
After this iteration converges to fixed points $s_1^*,\ldots,s_L^*$,
we then compute $m(z)=t_L(\z_L,(1,0,\ldots,0))$ using (\ref{eq:tl}) and
these fixed points. For each successive value $z=x+i\eta$ along the grid of
values $x \in \R$, we initialize $s_1^{(0)},\ldots,s_L^{(0)}$ by linear
interpolation from the computed fixed points at the preceding two values of $x$
along this grid, for faster computation.

Note that for each value $z=x+i\eta$, if the above iteration converges to fixed
points $s_1^*,\ldots,s_L^* \in \C^+$, then this procedure
computes the correct value for $m(z)$: This is because, denoting
\[\z_{L-1}^*=\z_\prev(s_L^*,\z_L), \quad
\z_{L-2}^*=\z_\prev(s_{L-1}^*,\z_{L-1}^*), \quad \ldots,
\qquad \z_1^*=\z_\prev(s_2^*,\z_2^*),\]
it may be checked iteratively from (\ref{eq:sl},\ref{eq:tl}) and the
uniqueness guarantee of Proposition \ref{prop:swelldefined} that
$s_1^*=s_1(\z_1^*)$, then $s_2^*=s_2(\z_2^*)$, etc., and finally that
$s_L^*=s_L(\z_L)$. This then means that
$\z_{L-1}^*=\z_\prev(s_L(\z_L),\z_L)=\z_{L-1}$, then
$\z_{L-2}^*=\z_\prev(s_{L-1}(\z_{L-1}),\z_{L-1})=\z_{L-2}$, etc.,
and so $s_\ell^*=s_\ell(\z_\ell)$ for each $\ell$. Then this 
method computes the correct value for $m(z)=t_L(\z_L,(1,0,\ldots,0))$.

We have found in practice that the above iteration occasionally converges to
fixed points $s_1,\ldots,s_L$ not belonging to $\C^+$ (i.e.\ this is not a
mapping from $(\C^+)^L$ to $(\C^+)^L$). If this occurs, we randomly
re-initialize $s_1^{(0)},\ldots,s_L^{(0)} \in \C^+$, and we have found that
the method reaches the correct fixed point within a small number of random
initializations.

To clarify this approach, let us illustrate this computation in a
simple example: Consider $L=2$. Fix any grid value $x\in\R$ and $\eta>0$. An
approximate density function for the limit spectrum of $X_2^\top X_2$ at $x$ is
given by $\frac{1}{\pi}\Im t_2\left((-z,0,0,1),(1,0,0,0) \right)$,
where $z=x+i\eta$. Based on equations (\ref{eq:t0},\ref{eq:sl},\ref{eq:tl}),
\begin{align*}
	t_2\left((-z,0,0,1),(1,0,0,0) \right)
	&= t_1\left(\left(-z+\frac{1-b_\sigma^2}{s_2},0,\frac{b_\sigma^2}{s_2}\right),(1,0,0) \right)\\
	&=
t_0\left(\left(-z+\frac{1-b_\sigma^2}{s_2}+\frac{1-b_\sigma^2}{s_1},\frac{b_\sigma^2}{s_1}\right),(1,0)
\right)\\
	&=\int
\left(-z+\frac{1-b_\sigma^2}{s_2}+\frac{1-b_\sigma^2}{s_1}+\frac{b_\sigma^2}{s_1}x\right)^{-1}d\mu_0(x),
\end{align*}
where $s_1,s_2\in\C^+$ satisfy the fixed point equations
\begin{align}
	s_2&=1+\gamma_2s_2+\gamma_2t_0\left(\left(-z+\frac{1-b_\sigma^2}{s_2}+\frac{1-b_\sigma^2}{s_1},\frac{b_\sigma^2}{s_1}\right),(s_2z,0)\right)\label{eq:s_2}\\
	s_1&=\frac{s_2}{b_\sigma^2}+\gamma_1t_0\left(\left(-z+\frac{1-b_\sigma^2}{s_2}+\frac{1-b_\sigma^2}{s_1},\frac{b_\sigma^2}{s_1}\right),(1-b_\sigma^2,b_\sigma^2) \right).\label{eq:s_1}
\end{align}
We randomly initialize $s_1^{(0)},s_2^{(0)}\in\C^+$, and
update $s_1^{(t+1)},s_2^{(t+1)}$ simultaneously by substituting
$s_1=s_1^{(t)}$ and $s_2=s_2^{(t)}$ into the right side of (\ref{eq:s_2}) and
(\ref{eq:s_1}). We iterate this until convergence, and
then substitute into the above expression for $t_2((-z,0,0,1),(1,0,0,0))$ to
approximate the limit spectral density of $X_2^\top X_2$ at $x$.

\section{Proof of $(\eps,B)$-orthonormality for independent input
training samples}\label{appendix:inputisgood}
We prove Proposition \ref{prop:inputisgood}. For convenience, in this section,
we denote the input dimension $d_0$ simply as $d$, and we denote the rescaled
input by $\wX=\sqrt{d}\,X$, with columns $\wx_\a=\sqrt{d} \cdot \x_\a$. 

{\bf Bound for $\|\wx_\a\|^2$:}
Note that $\E[\|\wx_\a\|^2]=d$. Applying the convex concentration property
and \cite[Theorem 2.5]{adamczak2015note} with $A=\Id$, we have for any
$t>0$ that
\begin{equation}\label{eq:xnormconcentration}
\P\Big[\big|\|\wx_\a\|^2-d\big|>t\Big] \leq
2\exp\left(-c\min\left(\frac{t^2}{d},t\right)\right)
\end{equation}
for a constant $c$ depending only on $c_0$. Applying this for
$t=\sqrt{Kd\log n}$ and a union bound, with probability $1-2ne^{-cK\log n}$,
\begin{equation}\label{eq:xgood}
\Big|\|\wx_\a\|^2-d\Big| \leq \sqrt{Kd\log n} \quad
\text{ for all } \a \in [n].
\end{equation}
Rescaling, this shows $|\|\x_\a\|^2-1| \leq \sqrt{(K\log n)/d}$.

{\bf Bound for $\wx_\a^\top \wx_\b$:}
Since $\wx_\a$ and $\wx_\b$ are independent,
conditional on $\wx_\b$, we have $\E[\wx_\a^\top \wx_\b \mid \wx_\b]=0$, and
the map $\wx_\a \mapsto \wx_\a^\top \wx_\b$
is convex and $\|\wx_\b\|$-Lipschitz. Then the convex concentration property
implies, for any $t>0$,
\[\P\Big[|\wx_\a^\top \wx_\b|>t \Big|\wx_\b \Big] \leq
2e^{-c_0t^2/\|\wx_\b\|^2}.\]
On the event (\ref{eq:xgood}), applying this for $t=\sqrt{Kd\log n}$,
this probability is at most $2e^{-cK\log n}$.
Taking a union bound, with probability $1-2n^2e^{-cK\log n}$,
\[\Big|\wx_\a^\top \wx_\b\Big| \leq \sqrt{Kd\log n}
\quad \text{ for all } \a \neq \b \in [n].\]
Rescaling, this shows $|\x_\a^\top \x_\b| \leq \sqrt{(K\log n)/d}$.

{\bf Bound for $\|\wX\|$:} Fix any unit vector $\v=(v_1,\ldots,v_n) \in \R^n$.
By \cite[Lemma C.11]{kasiviswanathan2019restricted}, the random vector $\wX\v$
also satisfies the convex concentration property, with a modified constant
$c_0'$. Note that $\E[\|\wX\v\|^2]=d\|\v\|^2=d$. Then, as in
(\ref{eq:xnormconcentration}), we have
\[\P\Big[|\|\wX \v\|^2-d|>t\Big] \leq
2\exp\left(-c\min\left(\frac{t^2}{d},t\right)\right).\]
Applying this with $t=(B^2/4-1)d$,
and taking a union bound over a $1/2$-net $\N$
of the unit ball $\{\v \in \R^n:\|\v\|=1\}$ with cardinality $5^n$, we have
with probability at least $1-5^n \cdot 2e^{-cB^2d}$ that
\[\|\wX\v\| \leq (B/2)\sqrt{d} \quad \text{ for all } \v \in \N.\]
Since
\[\|\wX\|=\sup_{\v:\|\v\|=1} \|\wX\v\|
\leq \sup_{\v \in \N} \|\wX\v\|+\|\wX\|/2,\]
we have $\|\wX\| \leq B\sqrt{d}$ on this event. Rescaling, this shows $\|X\|
\leq B$.

{\bf Bound for $\sum_{\a=1}^n (\|\wx_\a\|^2-d)^2$:}
Define $\z=(z_1,\ldots,z_n)$ where $\z_\a=\|\wx_\a\|^2-d$.
Fixing any unit vector $\v=(v_1,\ldots,v_n) \in \R^n$, let us first bound
$\v^\top \z$: We have
\[\v^\top \z=\sum_{\a=1}^n v_\a (\|\wx_\a\|^2-d),\]
which has mean 0. Note that integrating the tail bound
(\ref{eq:xnormconcentration}) yields the sub-exponential condition
\[\E\left[\exp\left(\lambda(\|\wx_\a\|^2-d)\right)\right] 
\leq \exp(Cd\lambda^2)
\quad \text{ for all } |\lambda| \leq c'\]
and some constants $C,c'>0$. (See e.g.\ \cite[Theorem
2.3]{boucheron2013concentration}, applied with $(v,c)=(C'd,C')$ and a large enough constant $C'>0$.)
Then, as $\wx_1,\ldots,\wx_n$ are independent and $\|\v\|^2=1$, also
\[\E[e^{\lambda \v^\top \z}]
=\E\left[\exp\left(\lambda \sum_{\a=1}^n v_\a(\|\wx_\a\|^2-d)\right)\right] 
\leq \exp(Cd\lambda^2)
\quad \text{ for all } |\lambda| \leq c'.\]
For any $t>0$, applying this with $\lambda=\min(t/(2Cd),c')$
yields the sub-exponential tail bound
\[\P[\v^\top \z \geq t] \leq e^{-\lambda t}\E[e^{\lambda \v^\top \z}]
\leq \exp\left(-c\min\left(\frac{t^2}{d},t\right)\right).\]
Now applying this for $t=(B/2)d$, and again
taking a union bound over a $1/2$-net $\N$ of the unit ball, we have with
probability $1-5^n \cdot e^{-cBd}$ that
\[\v^\top \z \leq (B/2)d \quad \text{ for all } \v \in \N.\]
On this event, we have as above that $\|\z\| \leq Bd$, so $\|\z\|^2 \leq
B^2d^2$. Rescaling, this shows $\sum_{\a=1}^n (\|\wx_\a\|^2-1)^2 \leq B^2$.

Applying all of the above bounds for sufficiently large constants $K,B>0$, we
obtain that these bounds hold with probability at least $1-n^{-k}$, which yields
Proposition \ref{prop:inputisgood}.

\section{Overview of proofs and preliminary lemmas}\label{appendix:overview}

The proofs of Theorems \ref{thm:CK}, \ref{thm:NTK}, and \ref{thm:NTKmulti} are
contained in the subsequent Appendices
\ref{appendix:orthogonal}--\ref{appendix:NTKmulti}. We provide here an outline
of the argument.

We will apply induction across the layers $\ell=1,\ldots,L$,
analyzing the post-activation matrix $X_\ell$ of each layer conditional on the
previous post-activations $X_0,\ldots,X_{\ell-1}$ (i.e.\ with respect to only
the randomness of $W_\ell$). For the Conjugate Kernel, this will entail
analyzing the Stieltjes transform
\[\frac{1}{n} \Tr (X_L^\top X_L-z\Id)^{-1}\]
conditional on the previous layers. For the Neural Tangent Kernel, given the
approximation in Lemma \ref{lemma:NTKapprox}, this will entail analyzing
the Stieltjes transform
\[\frac{1}{n} \Tr (A+X_L^\top X_L-z\Id)^{-1}\]
conditional on the previous layers, where $A$ is a linear combination of
$X_0^\top X_0,\ldots,X_{L-1}^\top X_{L-1}$, and $\Id$. Note that this matrix $A$
is deterministic conditional on the previous layers.

In Appendix \ref{appendix:orthogonal}, we carry out a non-asymptotic analysis
of $(\eps,B)$-orthonormality. In particular, we show that if the deterministic
input $X \equiv X_0$ is $(\eps,B)$-orthonormal, then $X_1$ is
$(C\eps,CB)$-orthonormal with high probability, for a constant $C>0$ depending
only on $\lambda_\sigma$. Note that we require the fourth technical condition
\[\sum_{\a=1}^n (\|\x_\a\|^2-1)^2 \leq B^2\]
in Definition \ref{def:orthogonal} to ensure that the operator norm $\|X_1\|$
remains of constant order, as otherwise $X_1$ may have a rank-one component
whose norm grows slowly with $n$. Applying this result conditionally for
every layer, Assumption \ref{assump:asymptotics} then implies that
$X_0,\ldots,X_L$ are all $(\tilde{\eps}_n,\tilde{B})$-orthonormal for modified
parameters $(\tilde{\eps}_n,\tilde{B})$ with high probability.

In Appendix \ref{appendix:singlelayer}, we carry out the analysis of the trace
\[\frac{1}{n}\Tr(A+\alpha X_1^\top X_1-z\Id)^{-1}\]
in a single layer, for a deterministic
$(\eps_n,B)$-orthonormal input $X_0$, symmetric matrix $A \in
\R^{n \times n}$, and spectral parameters $\alpha \in \C^* \equiv
\overline{\C^-} \setminus 0$ and $z \in \C^+$. We allow
$\alpha \in \C^*$ (rather than fixing $\alpha=1$), 
as the subsequent induction argument for the NTK will require this extension.
When $A=0$ and $\alpha=1$, this
reduces to the analysis in \cite{louart2018random}, and also mirrors the proof 
of the Marcenko-Pastur equation (\ref{eq:MPeq}). For $A \neq 0$,
this trace will depend jointly on $A$ and the second-moment matrix
$\Phi_1 \in \R^{n \times n}$ for the rows of $X_1$. We derive a 
fixed-point equation in terms of $A$ and $\Phi_1$, which approximates this
trace in the $n \to \infty$ limit. 

In Appendix \ref{appendix:CK}, we prove Theorem \ref{thm:CK} on the CK,
by specializing this analysis to the setting $A=0$ and $\alpha=1$.
The inductive loop
is closed via an entrywise approximation of the second-moment matrix
$\Phi_\ell$ in each layer by a linear combination of
$X_{\ell-1}^\top X_{\ell-1}$ and $\Id$ in the previous layer.
The main argument for this approximation
has been carried out in Appendix \ref{appendix:orthogonal}.

In Appendix \ref{appendix:NTK}, we prove Theorem \ref{thm:NTK} on the NTK.
Our analysis reduces the trace of any
linear combination of $X_0^\top X_0,\ldots,X_L^\top X_L$ and $\Id$ to the trace of a
more general rational function of $X_0^\top X_0,\ldots,X_{L-1}^\top X_{L-1}$ and $\Id$ in the previous layer. In order
to close the inductive loop, we analyze the trace of such a rational function
across layers, and show that it may be characterized by the recursive
fixed-point equations (\ref{eq:sl}) and (\ref{eq:tl}). 
In Appendix \ref{appendix:NTK}, we also establish the
approximation in Lemma \ref{lemma:NTKapprox} and the existence and uniqueness of
the fixed point to (\ref{eq:sl}).

Finally, in Appendix \ref{appendix:NTKmulti}, we prove Theorem
\ref{thm:NTKmulti}, which is a minor extension of Theorem \ref{thm:NTK}.

{\bf Notation.} In the proof, $\v^*$ and $M^*$ denote the conjugate transpose.
For a complex matrix $M \in \C^{n \times n}$,
we denote by
\[\tr M=n^{-1}\Tr M\]
the normalized matrix trace, by
$\|M\|=\sup_{\v \in \C^n:\|\v\|=1} \|M\v\|$ the operator norm, and by
$\|M\|_F=(\Tr M^*M)^{1/2}=(\sum_{\a,\b} |M_{\a\b}|^2)^{1/2}$ the Frobenius 
norm. Note that we have
\[|\tr M| \leq \|M\| \leq \|M\|_F, \qquad \|M\|_F \leq
\sqrt{n}\|M\|, \qquad |\tr AB| \leq n^{-1}\|A\|_F\|B\|_F.\]

Let us collect here a few basic results, which we will use in the subsequent
sections.

\begin{proposition}\label{prop:sigmaproperties}
Under Assumption \ref{assump:asymptotics}(b),
the constants $a_\sigma$ and $b_\sigma$ in (\ref{eq:qr}) satisfy
\[|b_\sigma| \leq 1 \leq \sqrt{a_\sigma} \leq \lambda_\sigma.\]
For a universal constant $C>0$, the activation function $\sigma$ satisfies
\begin{equation}\label{eq:sigmabound}
|\sigma(x)| \leq \lambda_\sigma(|x|+C) \qquad \text{ for all } x \in \R.
\end{equation}
\end{proposition}
\begin{proof}
It is clear from definition that $a_\sigma \leq \lambda_\sigma^2$.
By the Gaussian Poincar\'e inequality,
\[1=\E[\sigma(\xi)^2]=\Var[\sigma(\xi)]
\leq \E[\sigma'(\xi)^2]=a_\sigma.\]
By Gaussian integration-by-parts and Cauchy-Schwarz,
\[|b_\sigma|=|\E[\sigma'(\xi)]|=|\E[\xi \cdot \sigma(\xi)]|
\leq \E[\xi^2]^{1/2}\E[\sigma(\xi)^2]^{1/2}=1.\]
We have
\begin{equation}\label{eq:sigma0bound}
|\sigma(0)| \leq \E[|\sigma(0)-\sigma(\xi)|]+\E[|\sigma(\xi)|]
\leq \lambda_\sigma \E[|\xi|]+\E[\sigma(\xi)^2]^{1/2}
\leq C\lambda_\sigma
\end{equation}
(the last inequality applying $\lambda_\sigma \geq 1$). Then
$|\sigma(x)| \leq |\sigma(0)|+\lambda_\sigma |x| \leq \lambda_\sigma(|x|+C)$.
\end{proof}

\begin{proposition}\label{prop:invertible}
Suppose $M=U+iV \in \C^{n \times n}$, where the real and imaginary parts
$U,V \in \R^{n \times n}$ are symmetric, and $V$ is
invertible with either $V \succeq c_0\Id$ or $V \preceq -c_0\Id$ for a value
$c_0>0$. Then $M$ is invertible, and $\|M^{-1}\| \leq 1/c_0$.
\end{proposition}
\begin{proof}
For any unit vector $\v \in \C^n$,
\[\|M\v\|=\|M\v\| \cdot \|\v\| \geq |\v^* M\v|
=|\v^*U\v+i \cdot \v^*V\v| \geq |\v^*V\v|,\]
the last step holding because $U,V$ are real-symmetric so that
$\v^*U\v$ and $\v^*V\v$ are both real. By the given assumption on $V$, we have
$|\v^*V\v| \geq c_0$, so $\|M\v\| \geq c_0$ for every unit vector $\v \in \C^n$.
Then $M$ is invertible, and $\|M^{-1}\| \leq 1/c_0$.
\end{proof}

\begin{proposition}\label{prop:specapprox}
Let $M,\wM \in \R^{n \times n}$ be any two symmetric matrices satisfying
\[\frac{1}{n}\|M-\wM\|_F^2 \to 0\]
a.s.\ as $n \to \infty$. If $\limspec M=\nu$ for a probability
distribution $\nu$ on $\R$, then also $\limspec \wM=\nu$.
\end{proposition}
\begin{proof}
For fixed $z \in \C^+$, let
$m(z)=\tr(M-z\Id)^{-1}$ and $\tilde{m}(z)=\tr(\wM-z\Id)^{-1}$
be the Stieltjes transforms. Then applying $A^{-1}-B^{-1}=A^{-1}(B-A)B^{-1}$,
we may bound their difference by
\begin{align*}
|m(z)-\tilde{m}(z)|^2&=\frac{1}{n^2}\Big|\Tr
[(M-z\Id)^{-1}-(\wM-z\Id)^{-1}]\Big|^2\\
&=\frac{1}{n^2}\Big|\Tr (M-z\Id)^{-1}(\wM-M)(\wM-z\Id)^{-1}\Big|^2\\
&\leq \frac{1}{n^2}\|\wM-M\|_F^2\|(M-z\Id)^{-1}(\wM-z\Id)^{-1}\|_F^2\\
&\leq \frac{1}{n}\|\wM-M\|_F^2\|(M-z\Id)^{-1}\|^2\|(\wM-z\Id)^{-1}\|^2
\end{align*}
Applying $\|(M-z\Id)^{-1}\| \leq 1/\Im z$ by Proposition \ref{prop:invertible},
and similarly for $\wM$, the given
condition shows that $m(z)-\tilde{m}(z) \to 0$ a.s., pointwise over $z \in
\C^+$. If $\limspec
M=\nu$, then $m(z) \to m_\nu(z) \equiv \int (x-z)^{-1}d\nu(x)$ a.s., and hence
also $\tilde{m}(z) \to m_\nu(z)$ a.s.\ and $\limspec \wM=\nu$.
\end{proof}

\section{Propagation of approximate pairwise
orthogonality}\label{appendix:orthogonal}

In this section, we work in the following (non-asymptotic) setting of a single
layer: Consider any deterministic matrix $X \in \R^{d \times n}$,
let $W \in \R^{\vd \times d}$ have i.i.d.\ $\N(0,1)$ entries, and set
\begin{equation}\label{def:vX}
\vX=\frac{1}{\sqrt{\vd}}\sigma(WX) \in \R^{\vd \times n}.
\end{equation}
Note that $\vX$ has i.i.d.\ rows with distribution $\sigma(\w^\top
X)/\sqrt{\vd}$, where $\w \sim \N(0,\Id)$. Define the second-moment matrix of
$\vX$ by
\begin{equation}\label{def:Phi}
\Phi=\E[\vX^\top \vX]
=\E[\sigma(\w^\top X)^\top \sigma(\w^\top X)] \in \R^{n \times n}
\end{equation}
where the expectations are over the standard Gaussian matrix $W$ and standard
Gaussian vector $\w$. Let $\Phi_{\a\b}$ denote the $(\a,\b)$ entry of $\Phi$ for any $\a,\b\in[n]$. We show in this section the following result.

\begin{lemma}\label{lemma:orthonormalpropagation}
Suppose $X$ is $(\eps,B)$-orthonormal where $\eps<1/\lambda_\sigma$.
Then for universal
constants $C,c>0$, with probability at least $1-2n^2e^{-c\vd
\eps^2}-3e^{-cn}$, the matrix
$\vX$ remains $(\veps,\vB)$-orthonormal with
\[\veps=C\lambda_\sigma^2\eps,
\qquad \vB=C\Big(1+n/\vd\Big)\lambda_\sigma^2 B.\]
\end{lemma}

\begin{corollary}\label{cor:orthonormalinduction}
Under Assumption \ref{assump:asymptotics}, there exist
parameters $(\tilde{\eps}_n,\tilde{B})$ still satisfying
$\tilde{\eps}_n n^{1/4} \to 0$, such that a.s.\ for all large $n$,
every matrix $X_0,\ldots,X_L$ is $(\tilde{\eps}_n,\tilde{B})$-orthonormal.
\end{corollary}
\begin{proof}
Note that increasing $\eps_n$ represents a weaker assumption, so we may assume
without loss of generality that $\eps_n \geq n^{-0.49}$. Then
by Lemma \ref{lemma:orthonormalpropagation}, there is a constant $C_0 \geq 1$
depending on $\lambda_\sigma,\gamma_1,\ldots,\gamma_L$,
such that if $X_{\ell-1}$ is $(C_0^{\ell-1}\eps_n,C_0^{\ell-1}B)$-orthonormal,
then conditional on this event,
$X_\ell$ is $(C_0^\ell \eps_n,C_0^\ell B)$-orthonormal with probability
at least $1-e^{-n^{0.01}}$ for all large $n$. Thus, setting
$\tilde{\eps}_n=C_0^L\eps_n$ and $\tilde{B}=C_0^LB$,
with probability at least $1-Le^{-n^{0.01}}$, every matrix
$X_0,\ldots,X_L$ is $(\tilde{\eps}_n,\tilde{B})$-orthonormal.
The almost sure statement then follows from the Borel-Cantelli Lemma.
\end{proof}

In the remainder of this section, we prove
Lemma \ref{lemma:orthonormalpropagation}.
We divide the proof into Lemmas \ref{lemma:orthog},
\ref{lemma:normbound}, and \ref{lemma:colsqbound} below, which check the
individual requirements for $(\veps,\vB)$-orthonormality of $\vX$. We denote by
$C,C',c,c'>0$ universal constants that may change from instance to instance.

\begin{lemma}\label{lemma:orthog}
If $X$ is $(\eps,B)$-orthonormal where $\eps<1/\lambda_\sigma$,
then for universal constants $C,c>0$:
\begin{enumerate}[(a)]
\item For all $\a \neq \b \in [n]$,
\begin{align}
|\Phi_{\a\b}-b_\sigma^2 \x_\a^\top \x_\b| &\leq C\lambda_\sigma^2 \eps^2\label{eq:concentration_offdiagonal}\\
\Big|\E_{\w \sim \N(0,\Id)}[\sigma(\w^\top \x_\a)]\Big| &\leq C\lambda_\sigma
\Big|\|\x_\a\|^2-1\Big| \leq C\lambda_\sigma \eps\label{eq:wxbound}\\
|\Phi_{\a\a}-1| & \leq  C\lambda_\sigma
\Big|\|\x_\a\|^2-1\Big|\leq C\lambda_\sigma \eps
\label{eq:concentration_diagonal}
\end{align}
\item With probability at least $1-2n^2e^{-c\vd\eps^2}$,
simultaneously for all $\a \neq \b \in [n]$, the columns of $\vX$ satisfy
\[\big|\|\vx_\a\|^2-1\big| \leq C\lambda_\sigma^2 \eps,
\qquad \big|\vx_\a^\top \vx_\b \big| \leq C\lambda_\sigma^2 \eps.\]
\end{enumerate}
\end{lemma}
Note that (\ref{eq:concentration_offdiagonal}) establishes an approximation
which is second-order in $\eps$---this will be important in our later
arguments which approximate $\Phi$ in Frobenius norm.
\begin{proof}
For part (a), observe that $(\zeta_\a,\zeta_\b) \equiv
(\w^\top \x_\a,\w^\top \x_\b)$ is bivariate
Gaussian, with mean 0 and covariance
\[\Sigma=\begin{pmatrix} \|\x_\a\|^2 & \x_\a^\top \x_\b \\
\x_\a^\top \x_\b & \|\x_\b\|^2 \end{pmatrix}=\Id+\Delta\]
where $\Delta$ is entrywise bounded by $\eps$. Then  performing a Gram-Schmidt
orthogonalization procedure, for some independent standard Gaussian
variables $\xi_\a,\xi_\b \sim \N(0,1)$, we have
\begin{equation}\label{eq:gramschmidt}
\zeta_\a=u_\a \xi_\a, \qquad \zeta_\b=u_\b \xi_\b+v_\b \xi_\a
\end{equation}
where $u_\a,u_\b>0$ and $v_\b \in \R$ satisfy
$|u_\a-1|,|u_\b-1|,|v_\b| \leq C\eps$ for a universal constant $C>0$.

By a Taylor expansion of $\sigma(\zeta)$ around $\zeta=\xi$,
there exists a random variable $\eta$ between $\zeta$ and $\xi$ such that
\begin{equation}\label{eq:taylorexpansion}
\sigma(\zeta)=\sigma(\xi)+\sigma'(\xi)(\zeta-\xi)+\frac{1}{2}\sigma''(\eta)(\zeta-\xi)^2.
\end{equation}
For $\a \neq \b$,
applying this for both $\zeta_\a$ and $\zeta_\b$, noting that the product of
leading terms satisfies $\E[\sigma(\xi_\a)\sigma(\xi_\b)]=0$, and applying
also the bounds $|\sigma'(x)|,|\sigma''(x)| \leq \lambda_\sigma$ where
$\lambda_\sigma \geq 1$, it is easy to check that
\[\Phi_{\a\b}=\E[\sigma(\zeta_\a)\sigma(\zeta_\b)]
=\E\Big[\sigma(\xi_\a) \cdot
\sigma'(\xi_\b)(\zeta_\b-\xi_\b)+\sigma(\xi_\b) \cdot
\sigma'(\xi_\a)(\zeta_\a-\xi_\a) \Big]+\text{remainder}\]
where this remainder has magnitude at most $C\lambda_\sigma^2\eps^2$.
For the first term, substituting (\ref{eq:gramschmidt}) and applying
independence of $\xi_\a$ and $\xi_\b$, we have
\begin{align*}
&\E\Big[\sigma(\xi_\a) \cdot
\sigma'(\xi_\b)(\zeta_\b-\xi_\b)+\sigma(\xi_\b) \cdot
\sigma'(\xi_\a)(\zeta_\a-\xi_\a) \Big]\\
&=(u_\b-1) \E[\sigma(\xi_\a)] \cdot \E[\sigma'(\xi_\b)\xi_\b]
+v_\b \E[\sigma(\xi_\a)\xi_\a] \cdot \E[\sigma'(\xi_\b)]
+(u_\a-1) \E[\sigma(\xi_\b)] \cdot \E[\sigma'(\xi_\a)\xi_\a].
\end{align*}
Applying $\E[\sigma(\xi)]=0$ and the integration-by-parts identity
$\E[\sigma(\xi)\xi]=\E[\sigma'(\xi)]=b_\sigma$, this term equals $v_\b
b_\sigma^2$. From (\ref{eq:gramschmidt}), we have $u_\a v_\b=\E[\zeta_\a
\zeta_\b]=\x_\a^\top \x_\b$. Since $|u_\a-1| \leq C\eps$ and $|\x_\a^\top \x_\b|
\leq \eps$, this implies $|v_\b b_\sigma^2-b_\sigma^2 \x_\a^\top \x_\b| \leq
Cb_\sigma^2 \eps^2 \leq C\lambda_\sigma^2 \eps^2$.
Combining these yields (\ref{eq:concentration_offdiagonal}). Similarly,
from a first-order Taylor expansion analogous to
(\ref{eq:taylorexpansion}),
\begin{align*}
\Big|\E[\sigma(\w^\top \x_\a)]\Big|
&=\Big|\E[\sigma(\zeta_\a)]-\E[\sigma(\xi_\a)]\Big|
\leq C\lambda_\sigma \cdot |u_\a-1|,\\
|\Phi_{\a\a}-1|&=\Big|\E[\sigma(\zeta_\a)^2]-\E[\sigma(\xi_\a)^2]\Big|
\leq C\max\Big(\lambda_\sigma \cdot |u_\a-1|,\;\lambda_\sigma^2 \cdot
|u_\a-1|^2\Big).
\end{align*}
The bounds (\ref{eq:wxbound}) and
(\ref{eq:concentration_diagonal}) follow from the observations
$u_\a^2=\E[\zeta_\a^2]=\|\x_\a\|^2$ and
$|u_\a-1| \leq |u_\a-1| \cdot |u_\a+1|=|u_\a^2-1| \leq \eps$.

For part (b), let $\w_k^\top$ be the $k^\text{th}$ row of $W$.
Then by definition of $\vX$, for any $\a,\b \in [n]$ (including $\a=\b$),
\[\vx_\a^\top \vx_\b=\frac{1}{\vd} \sum_{k=1}^{\vd}
\sigma\Big(\w_k^\top \x_\a\Big)
\sigma\Big(\w_k^\top \x_\b\Big).\]
We apply Bernstein's inequality: Denote by $\|\cdot\|_{\psi_2}$ and $\|\cdot\|_{\psi_1}$ the sub-Gaussian and sub-exponential norms of a random variable. For any deterministic vector $\x\in \R^{d}$, the function
$\w \mapsto \sigma(\w^\top \x)$ is
$\lambda_\sigma \|\x\|$-Lipschitz. Then for $\w \sim \N(0,\Id)$ and a universal
constant $C>0$, we have by Gaussian concentration-of-measure
\[\|\sigma(\w^\top \x_\a)-\E[\sigma(\w^\top \x_\a)]\|_{\psi_2}
\leq C\lambda_\sigma \|\x_\a\|.\]
From (\ref{eq:wxbound}),
$|\E[\sigma(\w^\top \x_\a)]| \leq C\lambda_\sigma \eps$. Thus
(recalling that $|\|\x_\a\|-1| \leq \eps$), we have
$\|\sigma(\w^\top \x_\a)\|_{\psi_2} \leq C\lambda_\sigma$ for a
constant $C>0$, and similarly for $\x_\b$. So
\begin{equation}\label{eq:subexpbound}
\|\sigma(\w^\top \x_\a) \sigma(\w^\top \x_\b)\|_{\psi_1}
\leq \|\sigma(\w^\top \x_\a)\|_{\psi_2}
\|\sigma(\w^\top \x_\b)\|_{\psi_2} \leq C\lambda_\sigma^2.
\end{equation}
Applying Bernstein's inequality (see \cite[Theorem 2.8.1]{vershynin2018high}),
for a universal constant $c>0$ and any $t>0$,
\[\P\Big[\big|\vx_\a^\top \vx_\b
-\E\big[\vx_\a^\top \vx_\b\big] \big|>t  \Big]\leq 2\exp\left(-c\vd
\min\left(\frac{t^2}{\lambda_\sigma^4},\frac{t}{\lambda_\sigma^2}\right)\right).\]
Applying this for $t=\lambda_\sigma^2\eps$ and
taking a union bound over all $\a,\b \in [n]$, we get
\begin{align}
\P\Big[\big|\vx_\a^\top \vx_\b-\E\big[\vx_\a^\top \vx_\b\big] \big|\leq
\lambda_\sigma^2 \eps
\text{ for all } \a,\b \in [n] \Big] \geq 1-2n^2\exp\left(-c\vd \cdot \eps^2\right).
\end{align}
Since $\E[\vx_\a^\top \vx_\b]=\Phi_{\a\b}$, part (b) now follows from part (a).
\end{proof}

\begin{lemma}\label{lemma:normbound}
If $X$ is $(\eps,B)$-orthonormal where $\eps<1/\lambda_\sigma$,
then for universal constants $C,c>0$:
\begin{enumerate}[(a)]
\item $\|\Phi\| \leq C\lambda_\sigma^2 B^2$.
\item With probability at least $1-2e^{-cn}$,
$\|\vX\| \leq C\Big(1+\sqrt{n/\vd}\Big)\lambda_\sigma B$.
\end{enumerate}
\end{lemma}
\begin{proof}
For part (a), define
\begin{equation}\label{eq:Sigma}
\Sigma=\E\Big[\sigma(\w^\top X)^\top \sigma(\w^\top X)\Big]
-\E[\sigma(\w^\top X)]^\top\E[\sigma(\w^\top X)]
\end{equation}
where the first term on the right is $\Phi$. Then
\[\|\Sigma\|=\sup_{\v:\|\v\|=1} \v^\top \Sigma \v
=\sup_{\v:\|\v\|=1} \left|\E\Big[\big(\sigma(\w^\top X)
\v\big)^2\Big]-\E\Big[\sigma(\w^\top X)\v\Big]^2\right|
=\sup_{\v:\|\v\|=1} \Var\big[\sigma(\w^\top X)\v\big].\]
We bound this variance using the Gaussian Poincar\'e inequality:
Let us fix $\v\in\R^n$ with $\|\v\|=1$ and define
\[F(\w)=\sigma(\w^\top X)\v=\sum_{\a=1}^n v_\a \sigma(\w^\top \x_\a).\]
Then, letting $\u \in \R^n$ be the vector with entries
$u_\a=v_\a \sigma'(\w^\top \x_\a)$,
\begin{equation}\label{eq:Flipschitz}
\nabla F(\w)=\sum_{\a=1}^n v_\a \sigma'(\w^\top \x_\a) \cdot \x_\a
=X\u, \qquad \|\nabla F(\w)\| \leq \|X\| \cdot \|\u\|
\leq \lambda_\sigma B.
\end{equation}
Then by the Gaussian Poincar\'e inequality,
$\Var[F(\w)] \leq \E[\|\nabla F(\w)\|^2] \leq \lambda_\sigma^2B^2$,
so $\|\Sigma\| \leq \lambda_\sigma^2 B^2$.
In addition, by (\ref{eq:wxbound}), the difference between $\Phi$ and $\Sigma$ is a rank-one perturbation controlled by
\begin{equation}\label{eq:rankone_perturbation_bound}
\|\Phi-\Sigma\|=\|\E[\sigma(\w^\top X)]\|^2=\sum_{\a=1}^n
\E[\sigma(\w^\top \x_\a)]^2 \leq C\lambda_\sigma^2
\sum_{\a=1}^n (\|\x_\a\|^2-1)^2 \leq C \lambda_\sigma^2B^2,
\end{equation}
the last inequality using the final condition of $(\eps,B)$-orthonormality in
Definition \ref{def:orthogonal}. This establishes part (a).

For part (b), we apply the concentration result of
\cite[Eq.\ (5.26)]{vershynin2010introduction} for matrices with independent
sub-Gaussian rows. For any fixed unit vector $\v \in \R^n$, recall from
(\ref{eq:Flipschitz}) that $F(\w)=\sigma(\w^\top X)\v$ is $\lambda_\sigma
B$-Lipschitz. Then by Gaussian concentration-of-measure,
\[\|F(\w)-\E[F(\w)]\|_{\psi_2} \leq C\lambda_\sigma B.\]
We have $|\E[F(\w)]| \leq \|\E[\sigma(\w^\top X)]\| \leq C\lambda_\sigma B$ by
(\ref{eq:rankone_perturbation_bound}), so also
$\|F(\w)\|_{\psi_2} \leq C\lambda_\sigma B$.
This holds for any unit vector $\v \in \R^n$, hence
$\|\sigma(\w^\top X)\|_{\psi_2} \leq C\lambda_\sigma B$ for the vector
sub-Gaussian norm. Thus, $\sqrt{d}\vX/(\lambda_\sigma B)$ has i.i.d.\
rows whose sub-Gaussian norm is at most a universal constant. Recalling
$\Phi=\E[\vX^\top \vX]$ and applying
\cite[Eq.\ (5.26)]{vershynin2010introduction} with
$A=\sqrt{d}\vX/(\lambda_\sigma B)$, we obtain for some universal
constants $C,c>0$ that
\[\P\left[\|\vX^\top \vX-\Phi\|>
\max(\delta,\delta^2)\|\Phi\|\right] \leq 2e^{-ct^2},
\qquad \delta=C\sqrt{n/\vd}+t/\sqrt{\vd}.\]
Note that the complementary event $\|\vX^\top \vX-\Phi\|
\leq \max(\delta,\delta^2)\|\Phi\|$ implies
\[\|\vX\| \leq \sqrt{(1+\max(\delta,\delta^2))\|\Phi\|}
\leq (1+C'\delta)\sqrt{\|\Phi\|}\]
for a constant $C'>0$. Then choosing $t=\sqrt{n}$ and applying part (a) yields
part (b).
\end{proof}

\begin{lemma}\label{lemma:colsqbound}
If $X$ is $(\eps,B)$-orthonormal where $\eps<1/\lambda_\sigma$,
then for universal constants $C,c>0$,
with probability at least $1-e^{-cn}$, the columns of $\vX$ satisfy
\[\sum_{\a=1}^n (\|\vx_\a\|^2-1)^2 \leq C\Big(1+n^2/\vd^2\Big)\lambda_\sigma^4
B^2.\]
\end{lemma}

Let us remark that in settings where $\eps \gg 1/\sqrt{n}$, applying
Lemma \ref{lemma:orthog}(b) to bound each term $(\|\vx_\a\|^2-1)^2$ separately
would not yield a constant-order bound for this sum. The proof below performs
a more careful analysis of the combined fluctuations of $(\|\vx_\a\|^2-1)^2$.

\begin{proof}
Let $\z=(z_1,\ldots,z_n) \in \R^n$ and $\r=(r_1,\ldots,r_n) \in \R^n$ be defined
as
\[z_\a=\|\vx_\a\|^2-\E[\|\vx_\a\|^2],
\qquad r_\a=\E[\|\vx_\a\|^2]-1.\]
The quantity to be bounded is $\|\z+\r\|^2$. Note that
$\|\z+\r\|^2 \leq 2\|\z\|^2+2\|\r\|^2$. We have
\[\E[\|\vx_\a\|^2]=\E\left[\frac{1}{\vd} \sum_{i=1}^{\vd}
\sigma(\w_i^\top \x_\a)^2 \right]=\Phi_{\a\a},\]
so applying (\ref{eq:concentration_diagonal}) from Lemma \ref{lemma:orthog},
\begin{equation}\label{eq:rsqbound}
\|\r\|^2=\sum_{\a=1}^n (\Phi_{\a\a}-1)^2
\leq C\lambda_\sigma^2 \sum_{\a=1}^n (\|\x_\a\|^2-1)^2 \leq C\lambda_\sigma^2
B^2.
\end{equation}
Thus it remains to bound $\|\z\|^2$.

Let $\N$ be a $1/2$-net of the unit ball $\{\w \in \R^n:\|\w\|=1\}$,
of cardinality $|\N| \leq 5^n$. Then
\[\|\z\|=\sup_{\w:\|\w\| \leq 1}
\w^\top \z \leq \sup_{\v \in \N} \v^\top \z+\|\z\|/2,\]
so $\|\z\| \leq 2\sup_{\v \in \N} \v^\top \z$.
For each fixed vector $\v=(v_1,\ldots,v_n) \in \N$, we have
\begin{align}
\v^\top \z&=\sum_{\a=1}^n v_\a \cdot \frac{1}{\vd}
\sum_{i=1}^{\vd} \Big(\sigma(\w_i^\top \x_\a)^2-\E[\sigma(\w_i^\top \x_\a)^2]
\Big)\nonumber\\
&=\frac{1}{\vd}\sum_{i=1}^{\vd}
\bigg(\sum_{\a=1}^n
\Big(\sigma(\w_i^\top \x_\a)^2-\E[\sigma(\w_i^\top \x_\a)^2]\Big) v_\a\bigg).
\label{eq:vzprod}
\end{align}
We will bound the sub-exponential norm of each summand $i=1,\ldots,\vd$ and
apply Bernstein's inequality.

For $\w \sim \N(0,\Id)$, denote
\[\q \equiv \q(\w)=(q_1,\ldots,q_n)=(\w^\top \x_1,\ldots,\w^\top \x_n),
\qquad F(\q)=\sum_{\a=1}^n \Big(\sigma(q_\a)^2
-\E[\sigma(q_\a)^2] \Big)v_\a.\]
Observe that $\q(\w)=X^\top \w$.
Thus we wish to bound the sub-exponential norm of $F(\q(\w))$ when $\w \sim
\N(0,\Id)$. By the Gaussian Sobolev inequality
(see \cite[Eq.\ (3)]{adamczak2015concentration}), for any $p \geq 2$,
\begin{equation}\label{eq:sobolev}
\|F(\q(\w))\|_{L^p} \leq \sqrt{p} \cdot \Big\|\|\nabla_{\w}
F(\q(\w))\|\Big\|_{L^p}
\end{equation}
where $\|Y\|_{L^p}=\E[|Y|^p]^{1/p}$ denotes the $L^p$-norm of a random variable
(and $\|\nabla_\w F(\q(\w))\|$ is the usual $\ell_2$ vector norm of the
gradient of $F(\q(\w))$ in $\w$). By the chain rule,
\[\nabla_{\w} F(\q(\w))=X \cdot \nabla_\q F(\q),\]
so
\[\|\nabla_{\w} F(\q(\w))\|^2 \leq \|X\|^2 \|\nabla_\q F(\q)\|^2
\leq B^2\|\nabla_\q F(\q)\|^2.\]
We have $(\partial/\partial q_\a) F(\q)=2\sigma(q_\a)\sigma'(q_\a)v_\a$, so
\[\|\nabla_\q F(\q)\|^2
=\sum_{\a=1}^n 4\sigma(q_\a)^2\sigma'(q_\a)^2v_\a^2
\leq 4\lambda_\sigma^2 \sum_{\a=1}^n \sigma(q_\a)^2v_\a^2.\]
Recalling (\ref{eq:subexpbound}), we have $\|\sigma(q_\a)^2\|_{\psi_1}
=\|\sigma(\w^\top \x_\a)^2\|_{\psi_1} \leq C\lambda_\sigma^2$. Then
\[\left\|\sum_{\a=1}^n \sigma(q_\a)^2v_\a^2\right\|_{\psi_1}
\leq C\lambda_\sigma^2 \sum_{\a=1}^n v_\a^2=C\lambda_\sigma^2,\]
so
\[\Big\|\|\nabla_{\w} F(\q(\w))\|^2\Big\|_{\psi_1}
\leq C \lambda_\sigma^4B^2.\]
This implies the bound (see \cite[Proposition 2.7.1]{vershynin2018high}),
for any $p \geq 1$,
\[\Big\|\|\nabla_{\w} F(\q(\w))\|\Big\|_{L^{2p}}^{2p}
=\E\Big[\|\nabla_{\w} F(\q(\w))\|^{2p}\Big]
=\Big\|\|\nabla_{\w} F(\q(\w))\|^2\Big\|_{L^p}^p
\leq (C'\lambda_\sigma^4B^2  \cdot p)^p\]
for a universal constant $C'>0$. Thus, applying this to
(\ref{eq:sobolev}), we obtain for any $p \geq 2$
\[\|F(\q(\w))\|_{L^p} \leq \sqrt{p} \cdot C\lambda_\sigma^2B \sqrt{p}
=C\lambda_\sigma^2B \cdot p.\]
Finally, this implies (see again \cite[Proposition 2.7.1]{vershynin2018high})
$\|F(\q(\w))\|_{\psi_1} \leq C'\lambda_\sigma^2B$
for a universal constant $C'>0$, which is our desired bound on the
sub-exponential norm of $F(\q(\w))$.

Applying this and Bernstein's inequality to (\ref{eq:vzprod}), for any $t>0$,
\[\P[\v^\top \z>t]
\leq \exp\left(-c\vd \min\left(\frac{t^2}{\lambda_\sigma^4B^2},
\frac{t}{\lambda_\sigma^2B}\right)\right).\]
Setting
\[t=C_0\lambda_\sigma^2B \cdot \max(\delta,\delta^2),
\qquad \delta=\sqrt{n/\vd}\]
for a large enough constant $C_0>0$, and
taking the union bound over all $5^n$ vectors $\v \in \N$, we get
\[\P[\|\z\|>2t]
\leq \P\left[\sup_{\v \in \N} \v^\top \z>t\right] \leq e^{-cn}\]
for a constant $c>0$. Combining with the bound on $\|\r\|^2$ in
(\ref{eq:rsqbound}), we obtain the lemma.
\end{proof}

\section{Resolvent analysis for a single layer}\label{appendix:singlelayer}

We consider the same setting of a single layer as in the preceding section.
Let $\vX$ and $\Phi$ be defined by the
deterministic input $X \in \R^{d \times n}$ and Gaussian 
matrix $W \in \R^{\vd \times d}$ as in
(\ref{def:vX}) and (\ref{def:Phi}), and define the ($n$-dependent) aspect ratio
\[\gamma=n/\vd.\]
Consider a deterministic real-symmetric matrix $A \in \R^{n \times n}$, and
two (possibly $n$-dependent) spectral arguments
$\alpha \in \C^*$ and $z \in \C^+$, where
$\C^*=\overline{\C^-} \setminus \{0\}$. We study the matrix
\[A+\alpha \vX^\top \vX-z\Id.\]

We collect here the set of assumptions that we will use in this section.
\begin{assumption}\label{assump:singlelayer}
There are constants $B,C_0,c_0>0$ such that
\begin{enumerate}[(a)]
\item $\alpha \in \C^*$ and $z \in \C^+$, and
$\gamma,|\alpha|,|z|,\Im z \in [c_0,C_0]$.
\item $X$ is $(\eps_n,B)$-orthonormal, where $\eps_n<n^{-0.01}$.
\item $A \in \R^{n \times n}$ is deterministic and symmetric,
satisfying $\|A\| \leq C_0$.
\item $W$ has i.i.d.\ $\N(0,1)$ entries, and $\sigma(x)$ satisfies Assumption
\ref{assump:asymptotics}(b).
\end{enumerate}
\end{assumption}
Throughout this section, $C,C',c,c',n_0>0$ denote constants changing from
instance to instance that may depend on $\lambda_\sigma$ and the above
values $B,C_0,c_0$.

Proposition \ref{prop:invertible} ensures that $A+\alpha\vX^\top\vX-z\Id$
is invertible. Define the resolvent
\begin{equation}\label{eq:resolvent}
R=(A+\alpha \vX^\top \vX-z\Id)^{-1} \in \C^{n \times n}
\end{equation} 
and the deterministic ($n$-dependent) parameter
\begin{equation}\label{eq:sbar}
\bar{s}=\alpha^{-1}+\gamma \cdot \E[\tr R\Phi].
\end{equation}
The goal of this section is to prove the following result, which approximates
this resolvent $R$ by replacing the random matrix $\alpha \vX^\top \vX$ with a
deterministic matrix $\bar{s}^{-1}\Phi$, and provides an approximate fixed-point
equation that defines this parameter $\bar{s}$.

For $A=0$ and $\alpha=1$, we will verify in Appendix \ref{appendix:CK}
that this result reduces to the Marcenko-Pastur equation (\ref{eq:MPeq}).

\begin{lemma}\label{lemma:fixedpoint}
Under Assumption \ref{assump:singlelayer},
there are constants $C,c,c',n_0>0$ such that for all $n \geq n_0$,
any deterministic matrix $M \in \C^{n \times n}$, and any $t \in (n^{-1},c')$,
\begin{enumerate}[(a)]
\item 
$\displaystyle \P\left[\left|\tr RM-\tr \left(A+\bar{s}^{-1}\Phi-z\Id\right)^{-1}M\right|
>\|M\|t\right] \leq Cne^{-cnt^2}$
\item 
$\displaystyle \P\left[\left|\bar{s}-\big(\alpha^{-1}+\gamma
\tr \left(A+\bar{s}^{-1} \Phi-z\Id\right)^{-1}\Phi\big)\right|
>t\right] \leq Cne^{-cnt^2}$
\end{enumerate}
\end{lemma}

\subsection{Basic bounds}
\begin{proposition}\label{prop:basic}
Under Assumption \ref{assump:singlelayer}, deterministically
for some constants $C,c,n_0>0$ and all $n \geq n_0$,
\[\|R\| \leq C, \qquad \|\Phi\| \leq C, \qquad
|\bar{s}| \leq C, \qquad \Im \bar{s} \geq c.\]
Furthermore, with probability at least $1-2e^{-c'n}$ for a constant $c'>0$, 
\[\Im \tr R\Phi \geq c.\]
\end{proposition}
\begin{proof}
We may write $A+\alpha \vX^\top \vX-z\Id=U+iV$ where $U=A+(\Re \alpha)\vX^\top
\vX-(\Re z)\Id$ and $V=(\Im \alpha) \vX^\top \vX^\top-(\Im z)\Id$. Both $U$ and
$V$ are
symmetric, and $V \preceq (-\Im z)\Id$ because $\Im \alpha \leq 0$ and $\Im
z>0$. Then $\|R\| \leq 1/\Im z \leq C$ by Proposition \ref{prop:invertible}.

The bound $\|\Phi\| \leq C$ comes from Lemma \ref{lemma:normbound}(a) and the
$(\eps_n,B)$-orthonormality assumption for $X$.
Then from the definition of $\bar{s}$ in
(\ref{eq:sbar}) and the bounds $\|R\|,\|\Phi\| \leq C$, we have also
$|\bar{s}| \leq C$. For the lower bound for $\Im \bar{s}$ and $\Im \tr R\Phi$,
let us write
\[\tr R\Phi=\tr \left(\frac{R+R^*}{2}\right)\Phi+\tr \left(\frac{R-R^*}{2}\right)\Phi.\]
The first trace is real because $R+R^*$ is Hermitian, so
\[\Im \tr R\Phi=\Im \tr \left(\frac{R-R^*}{2}\right)\Phi.\]
Denoting $Y=A+\alpha \vX^\top \vX-z\Id$ and applying the identity
$A^{-1}-B^{-1}=A^{-1}(B-A)B^{-1}$, we have
\[R-R^*=Y^{-1}-(Y^*)^{-1}
=Y^{-1}(Y^*-Y)(Y^*)^{-1}=R(Y^*-Y)R^*.\]
Then, writing $Y=U+iV$ as above and applying $Y^*-Y=-2iV$, we get
\begin{align*}
\Im \tr R\Phi&=\Im(-i \cdot \tr RVR^*\Phi)\\
&=\Re\left(-(\Im \alpha) \cdot \tr R \vX^\top \vX R^*\Phi
+(\Im z) \cdot \tr RR^*\Phi\right).
\end{align*}
Since $\tr R\vX^\top \vX R^*\Phi=\tr \Phi^{1/2}R\vX^\top \vX R^*\Phi^{1/2}$,
where this matrix is positive semi-definite, this trace is real and
non-negative. Similarly, $\tr RR^*\Phi$ is real and non-negative. Then the above
yields the lower bound
\[\Im \tr R\Phi \geq \Im z \cdot \tr RR^*\Phi
\geq \Im z \cdot \lambda_{\min}(RR^*) \cdot \tr \Phi,\]
where $\lambda_{\min}(RR^*)$ is the smallest eigenvalue of $RR^*$.
By (\ref{eq:concentration_diagonal}) and the condition $\eps_n<n^{-0.01}$, we
have $\tr \Phi \geq c$ for a constant $c>0$ and large enough $n_0$.
Observe that $\lambda_{\min}(RR^*)=1/\|Y\|^2$, and
$\|Y\| \leq \|A\|+|\alpha| \cdot \|\vX\|^2+|z|$.
By Lemma \ref{lemma:normbound}(b), with probability $1-2e^{-c'n}$, we
have $\|\vX\| \leq C$, so putting this together yields $\Im \tr R\Phi \geq c$
with this probability. Finally, for the deterministic
bound $\Im \bar{s} \geq c$, we may apply $\Im \tr R\Phi \geq c$ on the event
where $\|\vX\| \leq C$ holds, and $\Im \tr R\Phi \geq 0$
on the complementary event. Taking an expectation and applying
the definition (\ref{eq:sbar}) yields $\Im \bar{s} \geq c$.
\end{proof}

\subsection{Resolvent approximation}

We recall the result of \cite[Lemma 1]{louart2018random}, which
establishes concentration of quadratic forms in the rows of $\vX$.
The following is its specialization to standard Gaussian matrices $W$,
and stated in our notation.

\begin{lemma}[\cite{louart2018random}]\label{lemma:quadform}
Suppose $\sigma(x)$ is $\lambda_\sigma$-Lipschitz, and let $\vx_i^\top$ be a row
of $\vX$. Then for any deterministic
matrix $Y \in \R^{n\times n}$ with $\|Y\|\le 1$, for some constants
$C,c>0$ (depending on $\lambda_\sigma$), and for any $t>0$,
\begin{equation}\label{eq:concentration_quadratic}
\P\left(\left|\frac{1}{\gamma}\vx_i^\top Y \vx_i-\tr Y\Phi \right|>t\right)
\leq C\exp\left(-\frac{cn}{\|X\|^2}\min\left(\frac{t^2}{t_0^2},
t\right)\right)
\end{equation}
where $t_0=|\sigma(0)|+\lambda_\sigma\|X\|\sqrt{1/\gamma}$.
\end{lemma}

Using this result, we establish the following approximation for the resolvent
$R$ in (\ref{eq:resolvent}).

\begin{lemma}\label{lemma:resolvent_remainder}
Consider any deterministic matrix $M \in \C^{n \times n}$, and set
\[\delta_n=\tr M-\tr R\left(A+\frac{1}{\alpha^{-1}+\gamma \tr
R\Phi}\Phi-z\Id\right)M.\]
Under Assumption \ref{assump:singlelayer}, there exist constants
$C,c,c',n_0>0$ such that for all $n \geq n_0$ and $t \in (n^{-1},c')$,
\[\P[|\delta_n|>\|M\|t] \leq Cne^{-cnt^2}.\]
\end{lemma}
\begin{proof}
By rescaling $M$, we may assume that $\|M\| \leq 1$. We have
$\Id=R(A+\alpha \vX^\top \vX-z\Id)=RA+\alpha R\vX^\top \vX-zR$.
Writing $\vX^\top \vX=\sum_i \vx_i\vx_i^\top$ (where $\vx_i^\top$ is the
$i^\text{th}$ row of $\vX$),
multiplying by $M$, and taking the normalized trace $\tr=n^{-1}\Tr$,
\begin{align*}
\tr M&=\tr RAM+\alpha \tr R\vX^\top \vX M-z\tr RM\\
&=\tr RAM+\frac{\alpha}{n}\sum_{i=1}^{\vd} \vx_i^\top MR\vx_i-z\tr RM.
\end{align*}
Hence
\[\delta_n=\frac{\alpha}{n}\sum_{i=1}^{\vd} \vx_i^\top MR\vx_i
-\frac{\tr R\Phi M}{\alpha^{-1}+\gamma \tr R\Phi}.\]

Let us define the leave-one-out resolvent, for each $1\le i\le \vd$, 
\[R^{(i)}=\left(A+\alpha \sum_{j:j \neq i} \vx_j\vx_j^\top-z\Id\right)^{-1}.\]
We may then decompose $\delta_n$ as $\delta_n=J_1+\gamma J_2$
where (recalling $\gamma=n/\vd$)
\begin{align*}
J_1&=\frac{1}{n}\sum_{i=1}^{\vd}\left(\alpha \vx_i^\top MR\vx_i-\frac{\gamma \tr
R^{(i)}\Phi M}{\alpha^{-1}+\gamma\tr R^{(i)}\Phi}\right),\\
J_2&=\frac{1}{n}\sum_{i=1}^{\vd}\left(\frac{ \tr R^{(i)}\Phi
M}{\alpha^{-1}+\gamma\tr R^{(i)}\Phi}-\frac{ \tr R\Phi M}{\alpha^{-1}
+\gamma\tr R\Phi}\right).
\end{align*}
Let us denote these summands as
\[J_1^{(i)}=\alpha \vx_i^\top MR\vx_i-\frac{\gamma\tr R^{(i)}\Phi M}{
\alpha^{-1}+\gamma\tr R^{(i)}\Phi}\quad\text{and}\quad
J_2^{(i)}=\frac{\tr R^{(i)}\Phi M}{\alpha^{-1}+\gamma\tr R^{(i)}\Phi}-\frac{\tr
R\Phi M}{\alpha^{-1}+\gamma\tr R\Phi}.\]

{\bf Bound for $J_1$.} Momentarily fix the index $i \in
\{1,\ldots,\vd\}$. Applying the Sherman-Morrison identity, we have
\begin{equation}\label{eq:Sherman-Morrison}
  R=R^{(i)}-\frac{\a R^{(i)}\vx_i\vx_i^\top R^{(i)}}{1+\a \vx_i^\top
R^{(i)}\vx_i}.
\end{equation}
Then, introducing
$A_1=\vx_i^\top MR^{(i)}\vx_i$ and $A_2=\vx_i^\top R^{(i)}\vx_i$,
\[\alpha \vx_i^\top MR \vx_i=\alpha A_1-\frac{\alpha^2 A_1A_2}{1+\alpha A_2}
=\frac{A_1}{\alpha^{-1}+A_2}.\]
Recall that the rows of $\vX$ are i.i.d. 
Let $\vX^{(i)}$ be the matrix $\vX$ with the $i^\text{th}$ row $\vx_i$
removed, and let $\E_{\vx_i}[\cdot]$ be the expectation over only $\vx_i$ (i.e.\
conditional on $\vX^{(i)}$). Observe
that $R^{(i)}$ is a function of $\vX^{(i)}$.
Applying Proposition \ref{prop:basic} with $\vX^{(i)}$ in place of $\vX$,
we see that $\|R^{(i)}\|$
and $\|MR^{(i)}\|$ are both bounded by a constant. Then applying Lemma
\ref{lemma:quadform} conditional on $\vX^{(i)}$, and recalling the bound
(\ref{eq:sigmabound}) for $\sigma(0)$, there are constants $C,c>0$ for which
\[\P[|A_k-\E_{\vx_i}[A_k]|>t]
\leq Ce^{-cn\min(t^2,t)} \qquad \text{ for } k=1,2.\]

Note that
\[\E_{\vx_i}[A_1]=\Tr MR^{(i)}\E[\vx_i\vx_i^\top]=\frac{1}{\vd}
\Tr MR^{(i)}\Phi=\gamma \tr R^{(i)}\Phi M.\]
Similarly, $\E_{\vx_i}[A_2]=\gamma \tr R^{(i)}\Phi$, so
\[J_1^{(i)}=\frac{A_1}{\alpha^{-1}+A_2}-\frac{\E_{\vx_i}[A_1]}{\alpha^{-1}+\E_{\vx_i}[A_2]}.\]
Applying Proposition \ref{prop:basic},
we have for some constants $C,c,c'>0$, on an event
$\cE(\vX^{(i)})$ of probability $1-2e^{-c'n}$, that
\[|\E_{\vx_i}[A_1]| \leq C, \qquad
|\alpha^{-1}+\E_{\vx_i}[A_2]| \geq \Im (\alpha^{-1}+\E_{\vx_i}[A_2])
\geq c.\]
Then, for any $t$ such that $t<c/2$, on the event where
$|A_1-\E_{\vx_i}[A_1]| \leq t$, $|A_2-\E_{\vx_i}[A_2]| \leq t$,
and $\cE(\vX^{(i)})$ all hold,
\begin{equation}\label{eq:J1bound}
\left|J_1^{(i)}\right|
\leq \frac{|A_1-\E_{\vx_i}[A_1]|}{|\alpha^{-1}+A_2|}+
|\E_{\vx_i}[A_1]| \cdot \frac{|A_2-\E_{\vx_i}[A_2]|}{|\alpha^{-1}+A_2| \cdot |\alpha^{-1}+\E_{\vx_i}[A_2]|} \leq Ct.
\end{equation}
Thus, for $t<c'$ and a sufficiently small constant $c'>0$,
we have $\P[|J_1^{(i)}| \geq t] \leq Ce^{-cnt^2}$.
Applying a union bound over $i \in \{1,\ldots,\vd\}$, this yields
$\P[|J_1| \geq t] \leq Cne^{-cnt^2}$.

{\bf Bound for $J_2$.} Applying the identity $A^{-1}-B^{-1}=A^{-1}(B-A)B^{-1}$,
\[R^{(i)}-R=R^{(i)}(R^{-1}-(R^{(i)})^{-1})R=\alpha R^{(i)}\vx_i \vx_i^\top R.\]
Then, applying also the bounds $\|R\|,\|R^{(i)}\| \leq C$ from
Proposition \ref{prop:basic},
\[|\tr (R^{(i)}-R)\Phi M|=\frac{1}{n}|\alpha \vx_i^\top R\Phi MR^{(i)}\vx_i|
\leq \frac{C\|\vX\|^2}{n}.\]
Applying Lemma \ref{lemma:normbound}(b), with probability $1-2e^{-cn}$,
this is at most $C/n$ for every $i \in \{1,\ldots,\vd\}$.
Similarly, $|\tr (R^{(i)}-R)\Phi| \leq C/n$ with this probability.
Applying again
$|\tr R\Phi M| \leq C$, $|\alpha^{-1}+\gamma \tr R\Phi| \geq c$,
and an argument similar to (\ref{eq:J1bound}), we obtain $|J_2^{(i)}| \leq
C'/n$ for a constant $C'>0$. 
Taking a union bound over $i \in \{1,\ldots,\vd\}$, this yields
$\P[|J_2|>C/n] \leq C'ne^{-cn}$.
Combining these bounds for $J_1$ and $J_2$, choosing $t>cn^{-1}$, and
re-adjusting the constants yields the lemma.
\end{proof}

\subsection{Proof of Lemma \ref{lemma:fixedpoint}}

We now prove Lemma \ref{lemma:fixedpoint} using Lemma
\ref{lemma:resolvent_remainder}. Define the random $n$-dependent parameter
\[s=\alpha^{-1}+\gamma \tr R\Phi,\]
so that $\bar{s}=\E[s]$. The following establishes concentration of $s$ around
$\bar{s}$.

\begin{lemma}\label{lemma:s-bars_converge}
Under Assumption \ref{assump:singlelayer},
for some constants $c,n_0>0$, all $n \geq n_0$, and any $t>0$,
\[\P\left[\left|s-\bar{s}\right|>t\right]\leq 2e^{-cnt^2}.\]
\end{lemma}
\begin{proof}
Define $F(W)=\gamma \tr R\Phi$, where $R$ and $\vX$ are considered as a
function of $W$. Fix any
matrices $W,\Delta \in \R^{\vd \times n}$ where $\|\Delta\|_F=1$, and define
$W_t=W+t \Delta$. Then, applying $\partial R=-R(\partial (R^{-1}))R$ and
$R=R^\top$,
\begin{align*}
\vec(\Delta)^\top (\nabla F(W))=\frac{d}{dt}\Big|_{t=0} F(W_t)
&=-\gamma \tr R\left(\frac{d}{dt}\Big|_{t=0}R^{-1}\right)R\Phi\\
&=-2\gamma \alpha \tr R\left(\vX^\top \cdot \frac{d}{dt}\Big|_{t=0}\vX
\right)R\Phi\\
&=-\frac{2\gamma \alpha}{\sqrt{\vd}}
\tr R\left(\vX^\top \cdot \left(\sigma'(WX) \odot
(\Delta X)\right)\right)R\Phi,
\end{align*}
where $\odot$ is the Hadamard product, and $\sigma'$ is applied entrywise.
Applying Proposition \ref{prop:basic},
\[\Big|\vec(\Delta)^\top (\nabla F(W))\Big|
\leq \frac{C}{\sqrt{\vd}} \cdot \Big\|R\vX^\top \cdot
(\sigma'(WX) \odot (\Delta X)) \cdot R\Big\|
\leq \frac{C'}{\sqrt{\vd}} \cdot \|R\vX^\top\| \cdot
\|\sigma'(WX) \odot (\Delta X)\|.\]
For the first term,
\begin{align*}
\|R\vX^\top\|^2=\frac{1}{|\alpha|}\|R(\alpha \vX^\top \vX) R^*\|
&\leq \frac{1}{|\alpha|}\left(\|R(A+\alpha \vX^\top \vX-z\Id)R^*\|
+\|R(A-z\Id)R^*\|\right)\\
& \leq \frac{1}{|\alpha|}(\|R\|+\|R\|^2(\|A\|+|z|)) \leq C.
\end{align*}
For the second term,
\[\|\sigma'(WX) \odot (\Delta X)\| \leq \|\sigma'(WX) \odot (\Delta
X)\|_F \leq \lambda_\sigma \|\Delta X\|_F \leq \lambda_\sigma \|\Delta\|_F \cdot
\|X\| \leq C.\]
Thus $|\vec(\Delta)^\top (\nabla F(W))| \leq C/\sqrt{n}$. This holds for
every $\Delta$ such that $\|\Delta\|_F=1$, so
$F(W)$ is $C/\sqrt{n}$-Lipschitz in $W$ with respect to the Frobenius norm.
Then the result follows from Gaussian concentration of measure.
\end{proof}

To conclude the proof of Lemma \ref{lemma:fixedpoint}, we may again assume
$\|M\| \leq 1$ by rescaling $M$. Set
\[\wM=\left(A+\bar{s}^{-1}\Phi-z\Id\right)^{-1}M.\]
Note that $\bar{s}^{-1} \in \C^-$, so $\|\wM\| \leq
\|(A+\bar{s}^{-1}\Phi-z\Id)^{-1}\| \leq C$ by Proposition \ref{prop:invertible}.
Applying Lemma \ref{lemma:resolvent_remainder} with $\wM$,
\begin{equation}\label{eq:applicationM}
\P\left[\Big|\tr \wM-\tr R\left(A+s^{-1}
\Phi-z\Id\right)\wM\Big|>t\right] \leq Cne^{-cnt^2}
\end{equation}
for all $t \in (n^{-1},c')$.
Furthermore, applying the definition of $\wM$,
\begin{align*}
|\tr R\left(A+s^{-1} \Phi-z\Id\right)\wM-\tr RM|
&=\left|\tr R\left(\left(A+s^{-1} \Phi-z\Id\right)
-\left(A+\bar{s}^{-1} \Phi-z\Id\right)\right)\wM\right|\\
&=|s^{-1}-\bar{s}^{-1}| \cdot |\tr R\Phi\wM|
\leq C|s^{-1}-\bar{s}^{-1}|.
\end{align*}
Recall that $|\bar{s}| \geq \Im \bar{s} \geq c$.
Then, on the event where $|s-\bar{s}| \leq t$ and $t<c/2$, we have
$|s^{-1}-\bar{s}^{-1}| \leq Ct$. Then
applying Lemma \ref{lemma:s-bars_converge},
for some constants $c,c'>0$ and all $t \in (0,c')$,
\[\P\left[|\tr R\left(A+s^{-1} \Phi-z\Id\right)\wM-\tr
RM|>t \right] \leq 2e^{-cnt^2}.\]
Combining this with (\ref{eq:applicationM}) yields Lemma
\ref{lemma:fixedpoint}(a). Specializing Lemma \ref{lemma:fixedpoint}(a)
to $M=\Phi$, we obtain
\[\P\left[\left|s-\left(\alpha^{-1}+\gamma
\tr(A+\bar{s}^{-1}\Phi-z\Id)^{-1}\Phi\right)\right|>t\right]
\leq Cne^{-cnt^2}.\]
Applying again Lemma \ref{lemma:s-bars_converge} to bound $|s-\bar{s}|$, we
obtain Lemma \ref{lemma:fixedpoint}(b).

\section{Analysis for the Conjugate Kernel}\label{appendix:CK}

Theorem \ref{thm:CK} is a special case of Theorem \ref{thm:NTK}, but let us
provide here a simpler argument. Define, for each layer, the $n \times n$
matrices
\begin{align}
\Phi_\ell&=\E_{\w}\Big[\sigma(\w^\top X_{\ell-1})^\top
\sigma(\w^\top X_{\ell-1})\Big]\label{eq:Phiell}\\
\tilde{\Phi}_\ell&=b_\sigma^2 X_{\ell-1}^\top X_{\ell-1}+(1-b_\sigma^2)\Id
\label{eq:tildePhiell}
\end{align}
where $\E_\w$ denotes the expectation over only the random vector $\w \sim
\N(0,\Id)$. Here, $\Phi_\ell$ and $\tilde{\Phi}_\ell$ are deterministic
conditional on $X_{\ell-1}$, but are random unconditionally for $\ell \geq 2$.
For each fixed $\ell=1,\ldots,L$, we will show
\begin{equation}\label{eq:Phiapprox}
\limspec \Phi_\ell=\limspec \tilde{\Phi}_\ell.
\end{equation}
Conditional on $X_{\ell-1}$, the spectral limit of $X_\ell^\top X_\ell$ was
shown in \cite{louart2018random} to be a Marcenko-Pastur map of the spectral
limit of $\Phi_\ell$---we reproduce a short proof below under our assumptions,
by specializing Lemma \ref{lemma:fixedpoint} to $\alpha=1$ and $A=0$.
Combining with (\ref{eq:Phiapprox}) and
iterating from $\ell=1,\ldots,L$ yields Theorem \ref{thm:CK}.

\begin{lemma}\label{lemma:difference_Phi}
Under Assumption \ref{assump:asymptotics}, for each $\ell=1,\ldots,L$,
almost surely as $n \to \infty$, 
\[\frac{1}{n}\|\Phi_\ell-\tilde{\Phi}_\ell\|_F^2 \to 0.\]
\end{lemma}
\begin{proof}
By Corollary \ref{cor:orthonormalinduction}, increasing $(\eps_n,B)$ as
needed, we may assume that each matrix $X_0,\ldots,X_L$ is
$(\eps_n,B)$-orthonormal.
Denote by $\Phi_\ell[\a,\b]$ and $\tilde{\Phi}_\ell[\a,\b]$ the $(\a,\b)$
entries of these matrices. Then Lemma \ref{lemma:orthog}(a) shows
for $\a \neq \b$ that
\[|\Phi_\ell[\a,\b]-\tilde{\Phi}_\ell[\a,\b]|
\leq C\eps_n^2.\]
For $\a=\b$, applying $\tilde{\Phi}_\ell[\alpha,\alpha]=1-b_\sigma^2+
b_\sigma^2 \|\x_\a^{\ell-1}\|^2$, we have
\[|\Phi_\ell[\a,\a]-\tilde{\Phi}_\ell[\a,\a]|
\leq |\Phi_\ell[\a,\a]-1|+b_\sigma^2|\|\x^{\ell-1}_\a\|^2-1|
\leq C\eps_n.\]
Then
\[\|\Phi_\ell-\tilde{\Phi}_\ell\|_F^2
\leq Cn(n-1)\eps_n^4+Cn\eps_n^2,\]
and the result follows from the condition $\eps_n n^{1/4} \to 0$.
\end{proof}

\begin{proof}[Proof of Theorem \ref{thm:CK}]
By Corollary \ref{cor:orthonormalinduction},
we may assume that each matrix $X_0,\ldots,X_L$ is $(\eps_n,B)$-orthonormal.
This implies the bounds
$\|X_\ell\| \leq C$ and $\|K^\CK\| \leq C$ for all large $n$.

For the spectral convergence, suppose by induction that
$\limspec X_{\ell-1}^\top X_{\ell-1}=\mu_{\ell-1}$,
where the base case $\limspec X_0^\top X_0=\mu_0$ holds by assumption.
Defining
\[\nu_\ell=(1-b_\sigma^2)+b_\sigma^2 \cdot \mu_{\ell-1},\]
Proposition \ref{prop:specapprox} and Lemma \ref{lemma:difference_Phi}
together show that
\[\limspec \Phi_\ell=\limspec \tilde{\Phi}_\ell=\nu_\ell.\]
Specializing
Lemma \ref{lemma:fixedpoint}(b) to the setting $A=0$, $\alpha=1$,
$X=X_{\ell-1}$, and $\vX=X_\ell$,
and choosing $t \equiv t_n$ such that $t_n \to 0$ and
$nt_n^2 \gg \log n$, we obtain
\begin{equation}\label{eq:sbarMP}
\Big|\bar{s}-1-(n/d_\ell) \tr (\bar{s}^{-1}\Phi_\ell-z\Id)^{-1}\Phi_\ell\Big| \to 0
\end{equation}
a.s.\ as $n \to \infty$, where
\[\bar{s}=1+\frac{n}{d_\ell}
\E_{W_\ell}[\tr (X_\ell^\top X_\ell-z\Id)^{-1}\Phi_\ell].\]
Here, this expectation is taken over only $W_\ell$ (i.e.\ conditional on
$X_0,\ldots,X_{\ell-1}$).

Proposition \ref{prop:basic} verifies that $\bar{s}$
is bounded as $n \to \infty$,
so for any subsequence in $n$, there is a further
sub-subsequence along which $\bar{s} \to s_0$ for a limit $s_0 \equiv s_0(z)
\in \C^+$. Applying $A^{-1}-B^{-1}=A^{-1}(B-A)B^{-1}$ and Propositions
\ref{prop:invertible} and \ref{prop:basic},
\begin{align*}
&\Big|\tr (\bar{s}^{-1}\Phi_\ell-z\Id)^{-1}\Phi_\ell-
\tr (s_0^{-1}\Phi_\ell-z\Id)^{-1}\Phi_\ell\Big|\\
&=|s_0^{-1}-s^{-1}| \cdot \tr \Big|(s_0^{-1}\Phi_\ell-z\Id)^{-1}
\Phi_\ell (\bar{s}^{-1}\Phi_\ell-z\Id)^{-1}\Phi_\ell\Big|\\
&\leq |s_0^{-1}-s^{-1}| \cdot \|(s_0^{-1}\Phi_\ell-z\Id)^{-1}\| \cdot
\|(\bar{s}^{-1}\Phi_\ell-z\Id)^{-1}\|\cdot \|\Phi_\ell\|^2\\
&\leq C|s_0^{-1}-s^{-1}|.
\end{align*}
Thus, along the sub-subsequence where $\bar{s} \to s_0$, we get
\begin{equation}\label{eq:sbys0}
\tr (\bar{s}^{-1}\Phi_\ell-z\Id)^{-1}\Phi_\ell-
\tr (s_0^{-1}\Phi_\ell-z\Id)^{-1}\Phi_\ell \to 0.
\end{equation}
We have also
\begin{equation}\label{eq:Phibynu}
\tr (s_0^{-1}\Phi_\ell-z\Id)^{-1}\Phi_\ell \to
\int \frac{x}{s_0^{-1}x-z}d\nu_\ell(x),
\end{equation}
since the function $x \mapsto x/(s_0^{-1}x-z)$ is continuous and bounded over
$\R$, and $\limspec \Phi_\ell=\nu_\ell$. Thus, taking the limit of (\ref{eq:sbarMP})
along this sub-subsequence, the value $s_0$ must satisfy
\begin{equation}\label{eq:s0}
s_0-1-\gamma_\ell \int \frac{x}{s_0^{-1}x-z}\,d\nu_\ell(x)=0.
\end{equation}

Now applying Lemma \ref{lemma:fixedpoint}(a) with $M=\Id$, and taking the limit
along this sub-subsequence, by a similar argument
we obtain that
\begin{equation}\label{eq:mlz}
\tr (X_\ell^\top X_\ell-z\Id)^{-1} \to
\int \frac{1}{s_0^{-1}x-z}d\nu_\ell(x).
\end{equation}
Denoting this limit by $m_\ell(z)$, and rewriting (\ref{eq:s0}) by applying
\[\int \frac{x}{s_0^{-1}x-z}d\nu_\ell(x)
=s_0\int \left(1+\frac{z}{s_0^{-1}x-z}\right)d\nu_\ell(x)=s_0(1+zm_\ell(z)),\]
we get $s_0^{-1}=1-\gamma_\ell-\gamma_\ell zm_\ell(z)$. Applying this back to
the definition of $m_\ell(z)$ in (\ref{eq:mlz}), this shows that $m_\ell(z)$
satisfies the Marcenko-Pastur equation
\[m(z)=\int \frac{1}{x(1-\gamma_\ell-\gamma_\ell z m(z))-z}d\nu_\ell(x),\]
so $m_\ell(z)$ is the Stieltjes transform of
$\mu_\ell=\rho_{\gamma_\ell}^\MP \boxtimes \nu_\ell
=\rho_{\gamma_\ell}^\MP \boxtimes ((1-b_\sigma^2)+b_\sigma^2
\cdot \mu_{\ell-1})$.

We have shown that $\tr (X_\ell^\top X_\ell-z\Id)^{-1} \to m_\ell(z)$ almost
surely along this sub-subsequence in $n$. Since, for every subsequence in $n$,
there exists such a sub-subsequence, this implies
$\lim_{n \to \infty} \tr (X_\ell^\top X_\ell-z\Id)^{-1}=m_\ell(z)$ almost
surely. Thus $\limspec X_\ell^\top X_\ell=\mu_\ell$,
which completes the induction.
\end{proof}

\section{Analysis for the Neural Tangent Kernel}\label{appendix:NTK}

\subsection{Spectral approximation and
operator norm bound}\label{appendix:NTKapprox}

We first prove the spectral approximation stated in
Lemma \ref{lemma:NTKapprox}, as well as the operator norm
bound $\|K^\NTK\| \leq C$. The following form of $K^\NTK$ is derived also
in \cite[Eq.\ (1.7)]{huang2019dynamics}:
Denote by $\x^\ell_\a$ the $\a^\text{th}$ column of $X_\ell$. For each
$\ell=1,\ldots,L$, define the matrix $S_\ell \in \R^{d_\ell \times n}$ whose
$\a^\text{th}$ column is given by
\begin{equation}\label{eq:Sell}
\s^\ell_\a=D^\ell_\a \frac{W_{\ell+1}^\top}{\sqrt{d_\ell}}
D^{\ell+1}_\a \frac{W_{\ell+2}^\top}{\sqrt{d_{\ell+1}}} D^{\ell+2}_\a \ldots
\frac{W_L^\top}{\sqrt{d_{L-1}}}
D^L_\a\frac{\w}{\sqrt{d_L}},
\end{equation}
where we define diagonal matrices indexed by $\a \in [n]$ and $k \in [L]$ as
\[D^k_\a
\equiv \diag\Big(\sigma'(W_k \x^{k-1}_\a)\Big) \in \R^{d_k \times d_k}.\]
Applying the chain rule, we may verify for each input sample $\x_\a$ that
\[\nabla_\w f_\theta(\x_\a)=\x_\a^L \in \R^{d_L},
\quad \nabla_{W_\ell} f_\theta(\x_\a)
=\s_\a^\ell \otimes \x_\a^{\ell-1} \in \R^{d_\ell d_{\ell-1}}.\]
Then
\begin{align*}
\big(\nabla_{\w} f_\theta(X)\big)^\top
\big(\nabla_{\w} f_\theta(X)\big)&=X_L^\top X_L,\\
\big(\nabla_{W_\ell} f_\theta(X)\big)^\top
\big(\nabla_{W_\ell} f_\theta(X)\big)&=(S_\ell^\top S_\ell) \odot
(X_{\ell-1}^\top X_{\ell-1}),
\end{align*}
where $\odot$ is the Hadamard product. Thus, the NTK is given by
\begin{equation}\label{eq:NTKform}
K^{\NTK}=\Big(\nabla_\theta f_\theta(X)\Big)^\top
\Big(\nabla_\theta f_\theta(X)\Big)
=X_L^\top X_L+\sum_{\ell=1}^L (S_\ell^\top S_\ell) \odot
(X_{\ell-1}^\top X_{\ell-1}).
\end{equation}

\begin{lemma}\label{lemma:onelayerapprox}
Let $X \in \R^{d \times n}$ be $(\eps,B)$-orthonormal, let $W \in \R^{\vd \times
d}$ have i.i.d.\ $\N(0,1)$ entries, and let $\x_\a,\x_\b$ be two columns
of $X$ where $\a \neq \b$. Then for universal constants $C,c>0$ and any $t>0$:
\begin{enumerate}[(a)]
\item With probability at least $1-2e^{-c\vd t^2}$,
\[\left|\frac{1}{\vd}\Tr \Big(\diag\big(\sigma'(W \x_\a)\big)
\diag\big(\sigma'(W \x_\b)\big)\Big)-b_\sigma^2\right| \leq
C\lambda_\sigma^2(\eps+t).\]
\item Let $M \in \R^{d \times d}$ be any deterministic symmetric matrix,
and denote
\[T(\x_\a,\x_\b)=\frac{1}{\vd}\Tr \Big(\diag\big(\sigma'(W\x_\a)\big)WMW^\top
\diag\big(\sigma'(W\x_\b)\big)\Big).\]
With probability at least $1-(2\vd+2)e^{-c\min(t^2\vd,t\sqrt{\vd})}$,
\[\left|T(\x_\a,\x_\b)-b_\sigma^2 \Tr M\right| \leq C\lambda_\sigma^2
\left(\eps\sqrt{d}+t\sqrt{d}+t\sqrt{\vd}\right)\|M\|_F.\]
\end{enumerate}
Furthermore, both (a) and (b) hold with $(\x_\a,\x_\a)$ in place of
$(\x_\a,\x_\b)$, upon replacing $b_\sigma^2$ by $a_\sigma$.
\end{lemma}
\begin{proof}
Write $\w_k^\top \in \R^d$ for the $k^\text{th}$ row of $W$. Then
\[\frac{1}{\vd}\Tr \Big(\diag\big(\sigma'(W \x_\a)\big)
\diag\big(\sigma'(W \x_\b)\big)\Big)
=\frac{1}{\vd}\sum_{k=1}^{\vd} \sigma'(\w_k^\top \x_\a)\sigma'(\w_k^\top \x_\b).\]
Applying $\sigma'(\w_k^\top \x_\a)
\sigma'(\w_k^\top \x_\b) \in [-\lambda_\sigma^2,\lambda_\sigma^2]$ and
Hoeffding's inequality,
\[\P\left[\left|\frac{1}{\vd}\sum_{k=1}^{\vd} \Big(\sigma'(\w_k^\top \x_\a)
\sigma'(\w_k^\top \x_\b)-\E[\sigma'(\w_k^\top \x_\a)
\sigma'(\w_k^\top \x_\b)]\Big)\right|>\lambda_\sigma^2 t\right] \leq 2e^{-c\vd
t^2}.\]
To bound the mean, recall that $(\zeta_\a,\zeta_\b) \equiv
(\w_k^\top \x_\a,\w_k^\top\x_\b)$ is bivariate Gaussian, which we may write as
\[\zeta_\a=u_\a \xi_\a, \qquad \zeta_\b=u_\b \xi_\b+v_\b \xi_\a\]
as in (\ref{eq:gramschmidt}). Here, $\xi_\a,\xi_\b \sim \N(0,1)$ are
independent, $u_\a,u_\b>0$ and $v_\b \in \R$, and these satisfy
$|u_\a-1|,|u_\b-1|,|v_\b| \leq C\eps$. Applying the Taylor expansion
\[\sigma'(\zeta)=\sigma'(\xi)+\sigma''(\eta)(\zeta-\xi)\]
for some $\eta$ between $\zeta$ and $\xi$, and the conditions
$\E[\sigma'(\xi)]=b_\sigma$ and $|\sigma''(x)| \leq
\lambda_\sigma$, it is easy to check that
$|\E[\sigma'(\zeta_\a)\sigma'(\zeta_\b)]-b_\sigma^2| \leq
C\lambda_\sigma^2\eps$. Then part (a) follows. The statement
with $(\x_\a,\x_\a)$ and $a_\sigma$ follows similarly from this Taylor expansion
and the bound $|\E[\sigma'(\zeta_\a)^2]-a_\sigma| \leq C\lambda_\sigma^2\eps$.

For part (b), we write
\[T(\x_\a,\x_\b)=\frac{1}{\vd}\sum_{k=1}^{\vd} \sigma'(\w_k^\top \x_\a)
\sigma'(\w_k^\top \x_\b) \cdot \w_k^\top M \w_k.\]
By the Hanson-Wright inequality (see \cite[Theorem 1.1]{rudelson2013hanson}),
\[\P\Big[|\w_k^\top M \w_k-\Tr M|>\|M\|_F \cdot t\sqrt{\vd}\Big]
\leq 2e^{-c\min(t^2\vd,t\sqrt{\vd})}\]
for a constant $c>0$. Then, applying $|\sigma'(x)| \leq \lambda_\sigma$ and a
union bound over $k=1,\ldots,\vd$, with probability at least $1-2\vd
e^{-c\min(t^2\vd,t\sqrt{\vd})}$,
\[\left|T(\x_\a,\x_\b)-\Tr M \cdot
\frac{1}{\vd}\sum_{k=1}^{\vd} \sigma'(\w_k^\top \x_\a)\sigma'(\w_k^\top \x_\b)\right|
\leq \|M\|_F \cdot \lambda_\sigma^2 t\sqrt{\vd}.\]
Then part (b) follows from combining with part (a), and applying $\Tr M \leq
\sqrt{d}\|M\|_F$.
\end{proof}

\begin{corollary}\label{cor:Sapprox}
Let $\s_\a^\ell$ be as defined in (\ref{eq:Sell}), and let
$q_\ell,r_\ell$ be the constants in (\ref{eq:qr}).
Under Assumption \ref{assump:asymptotics}, for a constant
$C>0$, almost surely for all large $n$ and for all $\ell \in [L]$ and $\a \neq
\b \in [n]$,
\begin{equation}\label{eq:Sabbound}
\Big|{\s_\a^\ell}^\top \s_\b^\ell-q_{\ell-1}\Big| \leq C\max(\eps_n,n^{-0.48}),
\qquad \Big|\|\s_\a^\ell\|^2-r_{\ell-1}\Big| \leq C\max(\eps_n,n^{-0.48}).
\end{equation}
\end{corollary}
\begin{proof}
By Corollary \ref{cor:orthonormalinduction}, we may assume that each matrix
$X_0,\ldots,X_L$ is $(\eps_n,B)$-orthonormal. Since a larger value of $\eps_n$
corresponds to a weaker assumption, we may assume without loss of generality
that $\eps_n \geq n^{-0.48}$.

Fix $\ell \in [L]$ and $\a,\b \in [n]$, and define
\begin{align}
M_\ell&=D_\a^\ell D_\b^\ell\nonumber\\
M_k&=D_\a^k\frac{W_k}{\sqrt{d_{k-1}}}
\ldots D_\a^{\ell+1}\frac{W_{\ell+1}}{\sqrt{d_\ell}}D_\a^\ell D_\b^\ell
\frac{W_{\ell+1}^\top}{\sqrt{d_\ell}}D_\b^{\ell+1}\ldots
\frac{W_k^\top}{\sqrt{d_{k-1}}}D_\b^k \quad \text{ for }
\quad \ell+1 \leq k \leq L.\label{eq:Mell}
\end{align}
Recalling the definition (\ref{eq:Sell}) and applying the
Hanson-Wright inequality conditional on $W_1,\ldots,W_L$,
\begin{equation}\label{eq:HansonWright}
\left|{\s_\a^\ell}^\top \s_\b^\ell-\frac{1}{d_L}\Tr M_L\right|
\leq C\eps_n\sqrt{n} \cdot \frac{1}{d_L}\|M_L\|_F
\end{equation}
with probability $1-e^{-c\min(\eps_n^2n,\eps_n\sqrt{n})}
\geq 1-e^{-n^{0.01}}$. Next, for each
$k=L,L-1,\ldots,\ell+1$,
we apply Lemma \ref{lemma:onelayerapprox}(b) conditional on
$W_1,\ldots,W_{k-1}$, with $t=\eps_n$, $M=M_{k-1}/d_{k-1}$,
$d=d_{k-1}$, and $\vd=d_k$. Note that $k-1 \geq \ell \geq 1$, so that both
$d_{k-1}$ and $d_k$ are proportional to $n$. Then
\[\left|\frac{1}{d_k}\Tr M_k-b_\sigma^2 \cdot
\frac{1}{d_{k-1}}\Tr M_{k-1}\right|
\leq C\eps_n \sqrt{n} \cdot \frac{1}{d_{k-1}}\|M_{k-1}\|_F\]
with probability $1-e^{-n^{0.01}}$.
Finally, for $k=\ell$, applying Lemma \ref{lemma:onelayerapprox}(a) conditional
on $W_1,\ldots,W_{\ell-1}$ and with $t=\eps_n$,
\[\left|\frac{1}{d_\ell}\Tr M_\ell-b_\sigma^2\right| \leq C\eps_n\]
with probability $1-e^{-n^{0.01}}$. Combining these bounds, with probability
$1-C'e^{-n^{0.01}}$,
\[\left|{\s_\a^\ell}^\top \s_\b^\ell-(b_\sigma^2)^{L-\ell+1}\right|
\leq \frac{C\eps_n}{\sqrt{n}}\left(
\|M_L\|_F+\ldots+\|M_\ell\|_F+\sqrt{n}\right).\]
We also have $\|W_k/\sqrt{d_k}\| \leq C$ for each $k=2,\ldots,L$
with probability $1-C'e^{-cn}$, see e.g.\ \cite[Theorem
4.4.5]{vershynin2018high}. 
Then, applying $\|D_k\| \leq \lambda_\sigma$, we have
$\|M_k\|_F \leq C\sqrt{n}\|M_k\| \leq C'\sqrt{n}$ for every $k=1,\ldots,L$.
Then the first bound of (\ref{eq:Sabbound}) follows. The second bound of
(\ref{eq:Sabbound}) is the same, applying
Lemma \ref{lemma:onelayerapprox} for $(\x_\a,\x_\a)$ instead of $(\x_\a,\x_\b)$.
The almost sure statement follows from the Borel-Cantelli Lemma.
\end{proof}

\begin{lemma}\label{lemma:NTKFapprox}
Under Assumption \ref{assump:asymptotics}, almost surely as $n \to \infty$,
\[\frac{1}{n}
\left\|K^\NTK-\left(r_+\Id+X_L^\top X_L+\sum_{\ell=0}^{L-1} q_\ell X_\ell^\top
X_\ell\right)\right\|_F^2 \to 0.\]
Furthermore, for a constant $C>0$, almost surely for all large $n$, 
$\|K^\NTK\| \leq C$.
\end{lemma}
\begin{proof}
By Corollary \ref{cor:orthonormalinduction}, we may assume that each matrix
$X_0,\ldots,X_L$ is $(\eps_n,B)$-orthonormal. Then
\[\Big|{\x_\a^{\ell-1}}^\top\x_\b^{\ell-1}\Big| \leq \eps_n,
\qquad \Big|\|\x_\a^{\ell-1}\|^2-1\Big| \leq \eps_n.\]
Increasing $\eps_n$ if necessary, we may assume $\eps_n \geq n^{-0.48}$.
Combining with (\ref{eq:Sabbound}), we have for the off-diagonal entries of
the Hadamard product that
\[\Big|\big((S_\ell^\top S_\ell) \odot (X_{\ell-1}^\top X_{\ell-1})\big)
[\a,\b]-q_{\ell-1} X_{\ell-1}^\top X_{\ell-1}[\a,\b]\Big| \leq C\eps_n^2,\]
and for the diagonal entries that
\begin{align*}
&\Big|\big((S_\ell^\top S_\ell) \odot (X_{\ell-1}^\top
X_{\ell-1})[\a,\a]-q_{\ell-1}(X_{\ell-1}^\top X_{\ell-1})[\a,\a]
-(r_{\ell-1}-q_{\ell-1})\Big|\\
&\leq \Big|\big((S_\ell^\top S_\ell) \odot (X_{\ell-1}^\top
X_{\ell-1})[\a,\a]-r_{\ell-1}\Big|
+q_{\ell-1}\Big|X_{\ell-1}^\top X_{\ell-1}[\a,\a]-1\Big|
\leq C\eps_n.
\end{align*}
Then applying this to (\ref{eq:NTKform}),
\[\left\|K^\NTK-\left(r_+\Id+X_L^\top X_L+\sum_{\ell=0}^{L-1} q_\ell X_\ell^\top
X_\ell\right)\right\|_F^2 \leq Cn(n-1)\eps_n^4+Cn\eps_n^2.\]
The first statement of the lemma then follows
from the assumption $\eps_n n^{1/4} \to 0$.

For the second statement on the operator norm, we have
\[\|(S_\ell^\top S_\ell) \odot (X_{\ell-1}^\top X_{\ell-1})\|
\leq \max_{\a=1}^n \Big|{\s^\ell_\a}^\top \s^\ell_\a\Big|
\cdot \|X_{\ell-1}^\top X_{\ell-1}\|.\]
See \cite[Eq.\ (3.7.9)]{johnson1990matrix}, applied with $X=Y=S_\ell$.
Then $\|K^\NTK\| \leq C$ follows from (\ref{eq:NTKform}), the
$(\eps_n,B)$-orthonormality of each matrix
$X_{\ell-1}$, and the bound for $\|\s_\a^\ell\|^2$ in (\ref{eq:Sabbound}).
\end{proof}

Combining Lemma \ref{lemma:NTKFapprox} and Proposition
\ref{prop:specapprox}, this proves Lemma \ref{lemma:NTKapprox}.

As a remark, Lemmas \ref{lemma:NTKFapprox} and \ref{lemma:NTKapprox} imply
$\limspec K^\NTK=\limspec(r_+\Id+X_L^\top X_L)$ when $b_\sigma=0$, since every
$q_{\ell}=0$ in this case. Thus, the Stieltjes transform of $\limspec K^\NTK$ is
actually $m_\NTK(z)=m(-r_++z)$ defined by the Stieltjes transform of
$\rho^\MP_\gamma$ in (\ref{eq:MPeq}) with $\gamma=\gamma_L$. Thus in the
following arguments for the limit spectrum of $K^\NTK$, we restrict to the case
$b_\sigma\neq 0$.

\subsection{Unique solution of the fixed-point equation}

Let $A,\Phi \in \R^{n \times n}$ be symmetric matrices, where $\Phi$ is
positive semi-definite. Let $z \in \C^+$, $\alpha \in \C^*$, and $\gamma>0$.
For $s \in \C^+$, define
\[S(s)=(A+s^{-1}\Phi-z\Id)^{-1}, \quad
f_n(s)=\alpha^{-1}+\gamma \tr S(s)\Phi.\]

\begin{lemma}\label{lemma:fixedpointfiniten}
\begin{enumerate}[(a)]
\item For any $s \in \C^+$, setting $S \equiv S(s)$,
\[\Im f_n(s) \geq \Im z \cdot \gamma \tr S\Phi S^* \geq 0.\]
\item For any $s_1,s_2 \in \C^+$, setting $S_1 \equiv S(s_1)$ 
and $S_2 \equiv S(s_2)$,
\begin{align*}
&|f_n(s_1)-f_n(s_2)|\\
& \leq
|s_1-s_2| \cdot \left(\frac{\Im f_n(s_1)-\Im z \cdot \gamma \tr S_1\Phi S_1^*}{\Im s_1}\right)^{1/2}
\left(\frac{\Im f_n(s_2)-\Im z \cdot \gamma \tr S_2\Phi S_2^*}{\Im
s_2}\right)^{1/2}
\end{align*}
\end{enumerate}
\end{lemma}
\begin{proof}
For part (a), let us write
\[S\Phi=S\Phi S^*(A+s^{-1}\Phi-z\Id)^*
=S\Phi S^*A+(1/s^*)S\Phi S^*\Phi-z^* S\Phi S^*.\]
Since $S\Phi S^*$ is Hermitian and positive semi-definite, the quantities
$\tr S\Phi S^*A$, $\tr S\Phi S^* \Phi$, and $\tr S\Phi S^*$ are all real, and
the latter two are nonnegative. Then
\begin{equation}\label{eq:ImPhiS}
\Im f_n(s)=\Im \alpha^{-1}+\gamma \Im \tr S\Phi=
\Im \alpha^{-1}+\frac{\Im s}{|s|^2} \cdot \gamma \tr S\Phi S^*\Phi+\Im z \cdot
\gamma\tr S\Phi S^*.
\end{equation}
Each term on the right side of (\ref{eq:ImPhiS}) is nonnegative, and dropping
the first two of these terms yields (a).

For part (b), applying the identity $A^{-1}-B^{-1}=A^{-1}(B-A)B^{-1}$, we have
\[S_1-S_2=S_1(s_2^{-1}\Phi-s_1^{-1}\Phi)S_2=\frac{s_1-s_2}{s_1s_2}S_1\Phi S_2,\]
so
\[f_n(s_1)-f_n(s_2)=\gamma \tr S_1\Phi-\gamma \tr S_2\Phi
=\frac{\gamma(s_1-s_2)}{s_1s_2}\tr S_1\Phi S_2\Phi.\]
Applying Cauchy-Schwarz to the inner-product $\langle S_1,S_2 \rangle_\Phi=\tr
S_1\Phi S_2^*\Phi$,
\[|\tr S_1\Phi S_2\Phi|^2
=|\langle S_1,S_2^* \rangle_\Phi|^2
\leq \langle S_1,S_1 \rangle_\Phi \cdot \langle S_2^*,S_2^* \rangle_\Phi
= \tr S_1\Phi S_1^*\Phi \cdot \tr S_2\Phi S_2^* \Phi.\]
Then
\[|f_n(s_1)-f_n(s_2)| \leq |s_1-s_2| \cdot
\left(\frac{\gamma \tr S_1\Phi S_1^*\Phi}{|s_1|^2}\right)^{1/2}
\left(\frac{\gamma \tr S_2\Phi S_2^* \Phi}{|s_2|^2}\right)^{1/2}.\]
Dropping $\Im \alpha^{-1}$ in (\ref{eq:ImPhiS}) and applying this
to upper-bound $\gamma \tr S\Phi S^*\Phi/|s|^2$, part (b) follows.
\end{proof}

\begin{corollary}\label{cor:fixedpointunique}
As $n \to \infty$, suppose that $f_n(s) \to f(s)$ pointwise for
each $s \in \C^+$, the empirical spectral distributions of $\Phi$ and $A$
converge weakly to deterministic limits, and the limit 
for $\Phi$ is not the point distribution at 0. Then the fixed-point equation
$s=f(s)$ has at most one solution $s \in \C^+$.
\end{corollary}
\begin{proof}
Let us first show that for each $s \in \C^+$ and a value $c_0(s)>0$ independent
of $n$,
\begin{equation}\label{eq:trSPhiSlowerbound}
\liminf_{n \to \infty} \tr S(s)\Phi S(s)^* \geq c_0(s)>0.
\end{equation}
Denoting $S \equiv S(s)$ and applying the von Neumann trace inequality,
\[\tr S\Phi S^*=\frac{1}{n}\Tr \Phi S^*S
\geq \frac{1}{n}\sum_{\a=1}^n \lambda_\a(\Phi)\lambda_{n+1-\a}(S^*S),\]
where $\lambda_1(\cdot) \geq \ldots \geq \lambda_n(\cdot)$ denote the sorted
eigenvalues. Since
$\Phi$ has a non-degenerate limit spectrum, there is a constant $\eps>0$ for
which $\lambda_{\eps n}(\Phi)>\eps$ for all large $n$. (Throughout the proof,
$\eps n$, $\eps n/2$, etc.\ should be understood as their roundings to the
nearest integer.) Then
\[\tr S\Phi S^* \geq \eps \cdot \frac{1}{n}\sum_{\a=1}^{\eps n}
\lambda_{n+1-\a}(S^*S).\]
Denoting by $\sigma_\a(\cdot)$ the $\a^\text{th}$ largest singular value,
observe that
\[\lambda_{n+1-\a}(S^*S)=\sigma_{n+1-\a}(S)^2
=\sigma_\a(A+s^{-1}\Phi-z\Id)^{-2}.\]
Applying $\sigma_{\a+\b-1}(A+B) \leq \sigma_\a(A)+\sigma_\b(B)$, we have
\[\sigma_\a(A+s^{-1}\Phi-z\Id) \leq
\sigma_{\a/2}(A)+|s|^{-1}\sigma_{\a/2+1}(\Phi)+|z|.\]
Since the spectra of $A$ and $\Phi$ converge to deterministic limits,
this implies that there is a constant
$C(s)>0$ (also depending on $z$ and $\eps$)
such that $\sigma_\a(A+s^{-1}\Phi-z\Id) \leq C(s)$ for
every $\a \in [\eps n/2,\eps n]$ and all large $n$. Thus
\[\tr S\Phi S^* \geq \eps \cdot \frac{\eps n-\eps n/2}{n} \cdot C(s)^{-2}\]
for all large $n$, and this shows the claim (\ref{eq:trSPhiSlowerbound}).

Then, taking the limit $n \to \infty$ in
Lemma \ref{lemma:fixedpointfiniten}(b), we get
\[|f(s_1)-f(s_2)| \leq |s_1-s_2|
\cdot \left(\frac{\Im f(s_1)-\Im z \cdot \gamma c_0(s_1)}{\Im s_1}\right)^{1/2}
\left(\frac{\Im f(s_2)-\Im z \cdot \gamma c_0(s_2)}{\Im s_2}\right)^{1/2}.\]
If $s_1=f(s_1)$ and $s_2=f(s_2)$,
then this yields $|s_1-s_2| \leq |s_1-s_2| \cdot h(s_1,s_2)$ for some
quantity $h(s_1,s_2) \in [0,1)$, where $h(s_1,s_2)<1$ strictly because
$c_0(s_1),c_0(s_2)>0$. This contradiction implies $s_1=s_2$, so the equation $s=f(s)$ has at
most one solution $s \in \C^+$.
\end{proof}

\subsection{Proof of Proposition \ref{prop:swelldefined} and
Theorem \ref{thm:NTK}}

The operator norm bound in Theorem \ref{thm:NTK} was shown in Lemma
\ref{lemma:NTKFapprox}. For the spectral convergence,
note that by Lemma \ref{lemma:NTKapprox}, the limit Stieltjes transform of
$K^\NTK$ at any $z \in \C^+$ is given by
\[m_\NTK(z)=\lim_{n \to \infty} \tr \left((-z+r_+)\Id+X_L^\top X_L
+\sum_{\ell=0}^{L-1} q_\ell X_\ell^\top X_\ell\right)^{-1},\]
provided that this limit exists and defines the Stieltjes transform of a
probability measure. For
\[\z=(z_{-1},\ldots,z_\ell) \in \C^- \times \R^\ell \times \C^*,
\qquad \w=(w_{-1},\ldots,w_\ell) \in \C^{\ell+2},\]
recall the functions
\[\z \mapsto s_\ell(\z), \quad (\z,\w) \mapsto t_\ell(\z,\w)\]
defined recursively by (\ref{eq:sl}) and (\ref{eq:tl}). Proposition \ref{prop:swelldefined} and Theorem \ref{thm:NTK} are immediate
consequences of the following extended result.

\begin{lemma}\label{lemma:NTKextended}
Suppose $b_\sigma\neq 0$. Under Assumption \ref{assump:asymptotics}, for each $\ell=1,\ldots,L$:
\begin{enumerate}[(a)]
\item For every $\z \in \C^- \times \R^\ell \times \C^*$, the
equation (\ref{eq:sl}) has a unique fixed point $s_\ell(\z) \in \C^+$.
\item For every $(\z,\w) \in (\C^- \times \R^\ell \times \C^*) \times
\C^{\ell+2}$, almost surely
\begin{align}
&t_\ell(\z,\w)\nonumber\\
&=\lim_{n \to \infty}
\tr \Big(z_{-1}\Id+z_0 X_0^\top X_0+\ldots+z_\ell
X_\ell^\top X_\ell\Big)^{-1}\Big(w_{-1}\Id+w_0 X_0^\top X_0+\ldots+w_\ell
X_\ell^\top X_\ell\Big).\label{eq:tliscorrect}
\end{align}
In particular, for any $z_{-1},\ldots,z_\ell \in \R$ where $z_\ell \neq 0$,
\[\limspec z_{-1}\Id+z_0 X_0^\top X_0+\ldots+z_\ell
X_\ell^\top X_\ell=\nu\]
where $\nu$ is a probability measure on $\R$ with Stieltjes transform
\[m(z)=t_\ell\Big((-z+z_{-1},z_0,\ldots,z_\ell),(1,0,\ldots,0)\Big).\]
\end{enumerate}
\end{lemma}
\begin{proof}
By Corollary \ref{cor:orthonormalinduction}, we may assume that each matrix
$X_0,\ldots,X_L$ is $(\eps_n,B)$-orthonormal.

Define $\Phi_\ell,\tilde{\Phi}_\ell$ by (\ref{eq:Phiell}) and
(\ref{eq:tildePhiell}). For $\z=(z_{-1},\ldots,z_\ell)$, let us write as
shorthand
\[\z \cdot \bX^\top \bX(\ell)
=z_{-1}\Id+z_0X_0^\top X_0+\ldots+z_\ell X_\ell^\top X_\ell,\]
where the parenthetical $(\ell)$ signifies the index of the last term in this
sum. Let us define similarly $\w \cdot \bX^\top \bX(\ell)$.

Note that part (b) holds for $\ell=0$, by 
the assumption $\limspec X_0^\top X_0=\mu_0$,
the definition of $t_0((z_{-1},z_0),(w_{-1},w_0))$ in (\ref{eq:t0}),
and the fact that the function $x
\mapsto (w_{-1}+w_0x)/(z_{-1}+z_0x)$ is continuous and bounded over the
non-negative real line when $z_{-1} \in \C^-$ and $z_0 \in \C^*$.

We induct on $\ell$. Suppose that part (b) holds for $\ell-1$. 
To show part (a) for $\ell$, fix any
$\z=(z_{-1},\ldots,z_\ell) \in \C^- \times \R^\ell \times \C^*$
(not depending on $n$) and consider the matrix
\begin{equation}\label{eq:RNTK}
R=\Big(\z \cdot \bX^\top \bX(\ell)\Big)^{-1}.
\end{equation}
We apply the analysis of Appendix \ref{appendix:singlelayer}, conditional on
$X_0,\ldots,X_{\ell-1}$, and with the identifications
\[\vX=X_\ell, \qquad X=X_{\ell-1}, \qquad \vd=d_\ell, \qquad d=d_{\ell-1},\]
\[A=z_0X_0^\top X_0+\ldots+z_{\ell-1}X_{\ell-1}^\top X_{\ell-1}, \qquad
\alpha=z_\ell, \qquad z=-z_{-1}.\]
Observe that $\alpha \in \C^*$ and $z \in \C^-$. The
matrix $R$ in (\ref{eq:RNTK}) is exactly
\[R=(A+\alpha \vX^\top \vX-z\Id)^{-1}.\]
Since each $X_0,\ldots,X_{\ell-1}$ is $(\eps_n,B)$-orthonormal, we have
$\|A\| \leq C$ for some constant $C>0$
(depending on $z_{-1},\ldots,z_\ell,\lambda_\sigma$).
Thus Assumption \ref{assump:singlelayer} holds, conditional on
$X_0,\ldots,X_{\ell-1}$. Let us define the $n$-dependent parameter
\[\bar{s}=\frac{1}{\alpha}+\frac{n}{d_\ell}\tr \E_{W_\ell}[R \Phi_\ell]\]
where this expectation is over only the weights $W_\ell$.
Then, applying Lemma \ref{lemma:fixedpoint}(b) with a
value $t \equiv t_n$ such that $t \to 0$ and $nt^2 \gg \log n$, we obtain
\begin{equation}\label{eq:fixedpointNTKbars}
\Big|\bar{s}-\frac{1}{\alpha}-\frac{n}{d_\ell}
\tr (A+\bar{s}^{-1}\Phi_\ell-z\Id)^{-1}\Phi_\ell\Big| \to 0
\end{equation}
almost surely as $n \to \infty$.

Proposition \ref{prop:basic} shows that $|\bar{s}|$ is bounded, so for
any subsequence in $n$, there is a further sub-subsequence where $\bar{s} \to
s_0$ for a limit $s_0 \equiv s_0(\z) \in \C^+$. Let us now replace $\bar{s}$ and
$\Phi_\ell$ above by $s_0$ and $\tilde{\Phi}_\ell$: First we have
\[\tr \left(A+\bar{s}^{-1} \Phi_\ell-z\Id\right)^{-1}\Phi_\ell-\tr
\left(A+s_0^{-1} \tilde{\Phi}_\ell-z\Id\right)^{-1}\Phi_\ell
\to 0\]
by the same argument as (\ref{eq:sbys0}). Then, we have
\begin{align*}
&\left|\tr \left(A+s_0^{-1} \Phi_\ell-z\Id\right)^{-1}\Phi_\ell-\tr
\left(A+s_0^{-1} \tilde{\Phi}_\ell-z\Id\right)^{-1}\Phi_\ell\right|\\
&=\left|s_0^{-1}\tr\left(A+s_0^{-1} \Phi_\ell-z\Id\right)^{-1}
(\tilde{\Phi}_\ell-\Phi_\ell)\left(A+s_0^{-1} \tilde{\Phi}_\ell-z\Id\right)^{-1}
\Phi_\ell\right|\\
&\leq \frac{C}{n}\|\tilde{\Phi}_\ell-\Phi_\ell\|_F
\cdot \left\|(A+s_0^{-1} \tilde{\Phi}-z\Id)^{-1}
\Phi(A+s_0^{-1} \Phi-z\Id)^{-1}\right\|_F\\
&\leq \frac{C}{\sqrt{n}}
\|\tilde{\Phi}_\ell-\Phi_\ell\|_F
\cdot \|(A+s_0^{-1} \tilde{\Phi}-z\Id)^{-1}\|
\cdot \|\Phi\| \cdot \|(A+s_0^{-1} \Phi-z\Id)^{-1}\| \to 0,
\end{align*}
where the convergence to 0 follows from Lemma \ref{lemma:NTKFapprox}.
Finally, we have
\begin{align*}
&\left|\tr \left(A+s_0^{-1} \Phi_\ell-z\Id\right)^{-1}\Phi_\ell-\tr
\left(A+s_0^{-1} \Phi_\ell-z\Id\right)^{-1}\tilde{\Phi}_\ell\right|\\
&\leq \frac{1}{n}\|(A+s_0^{-1} \Phi_\ell-z\Id)^{-1}\|_F
\cdot \|\Phi_\ell-\tilde{\Phi}_\ell\|_F
\leq \frac{1}{\sqrt{n}}\|(A+s_0^{-1} \Phi_\ell-z\Id)^{-1}\|
\cdot \|\Phi_\ell-\tilde{\Phi}_\ell\|_F \to 0.
\end{align*}
Applying these approximations to (\ref{eq:fixedpointNTKbars}), we have almost
surely along this sub-subsequence that
\begin{equation}\label{eq:s0fixedfinal}
\Big|s_0-\frac{1}{\alpha}-\gamma_\ell
\tr (A+s_0^{-1}\tilde{\Phi}_\ell-z\Id)^{-1}\tilde{\Phi}_\ell\Big| \to 0.
\end{equation}

Now observe from the definitions of $A$, $\tilde{\Phi}_\ell$, and $z$ that 
\begin{align*}
A+s_0^{-1}\tilde{\Phi}_\ell-z\Id
&=\Big(z_{-1}+\frac{1-b_\sigma^2}{s_0}\Big)\Id
+\sum_{k=0}^{\ell-2} z_kX_k^\top X_k
+\Big(z_{\ell-1}+\frac{b_\sigma^2}{s_0}\Big)X_{\ell-1}^\top X_{\ell-1},\\
\tilde{\Phi}_\ell&=(1-b_\sigma^2)\Id+b_\sigma^2 X_{\ell-1}^\top X_{\ell-1}.
\end{align*}
Then, applying (\ref{eq:s0fixedfinal}) and the induction hypothesis
that part (b) holds for $\ell-1$, we obtain that the value $s_0$ must satisfy
\[s_0=\frac{1}{\alpha}+\gamma_\ell t_{\ell-1}
\Big(\z_\prev(s_0,\z),(1-b_\sigma^2,0,\ldots,0,b_\sigma^2)\Big),\]
where $\z_\prev$ is defined in (\ref{eq:zprev}). This shows the existence of a
solution (in $\C^+$) to the fixed-point equation (\ref{eq:sl}). Notice that
because $b_\sigma\neq 0$ and $s_0\in\C^+$, the last entry of $\z_\prev(s_0,\z)$
is in $\C^*$ and $(\z_\prev(s_0,\z),(1-b_\sigma^2,0,\ldots,0,b_\sigma^2))$ is in the domain of function $t_{\ell-1}$.

To show uniqueness, we apply Corollary \ref{cor:fixedpointunique}:
For any fixed $s \in \C^+$, defining
\[f_n(s)=\frac{1}{\alpha}+(n/d_\ell)\tr (A+s^{-1}\Phi_\ell-z\Id)^{-1}\Phi_\ell,\]
the same arguments as above establish that
\[\lim_{n \to \infty} f_n(s)=f(s) \equiv 
\frac{1}{\alpha}+\gamma_\ell t_{\ell-1}
\Big(\z_\prev(s,\z),(1-b_\sigma^2,0,\ldots,0,b_\sigma^2)\Big).\]
Part (b) holding for $\ell-1$ implies that both $A$ and
$\Phi_\ell$ have deterministic spectral limits, where
\[\limspec \Phi_\ell=\limspec \tilde{\Phi}_\ell\]
by (\ref{eq:Phiapprox}). This cannot be the point distribution at 0,
because (\ref{eq:concentration_diagonal}) implies that $\tr \Phi_\ell \geq 1/2$ for
all large $n$, and $\|\Phi_\ell\| \leq C$ so at least $n/(2C)$ eigenvalues of
$\Phi_\ell$ exceed $1/2$ for every $n$.
Thus, Corollary \ref{cor:fixedpointunique} implies that the
fixed point $s=f(s)$ is unique. So the fixed point $s_\ell(\z) \in \C^+$ is
uniquely defined by (\ref{eq:sl}), and this shows part (a) for $\ell$.

By the uniqueness of this fixed point, we have also shown that $s_0=s_\ell(\z)$,
where $s_0$ is the limit of $\bar{s}$ along the above sub-subsequence. Since for
any subsequence in $n$, there exists a sub-subsequence for this which holds,
this shows that $\lim_{n \to \infty} \bar{s}=s_\ell(\z)$ almost surely.

Now, to show that part (b) holds for $\ell$, let us also fix
any $\w=(w_{-1},\ldots,w_\ell) \in \C^{\ell+2}$. Using that $z_\ell \neq 0$,
we may write
\[\w \cdot \bX^\top \bX(\ell)
=\frac{w_\ell}{z_\ell}\cdot  \z \cdot \bX^\top \bX(\ell)
+\w_\prev \cdot \bX^\top \bX(\ell-1),\]
where $\w_\prev$ is as defined in (\ref{eq:wprev}). Then
\begin{equation}\label{eq:wreduce}
\Big(\z \cdot \bX^\top \bX(\ell)\Big)^{-1}
\Big(\w \cdot \bX^\top \bX(\ell)\Big)
=\frac{w_\ell}{z_\ell}\Id+
\Big(\z \cdot \bX^\top \bX(\ell)\Big)^{-1}
\Big(\w_\prev \cdot \bX^\top \bX(\ell-1)\Big).
\end{equation}
We now apply Lemma \ref{lemma:fixedpoint}(a) conditional on
$X_0,\ldots,X_{\ell-1}$, with the same identifications as above and with
\[M=\w_\prev \cdot \bX^\top \bX(\ell-1).\]
Note that $M$ is indeed deterministic conditional on $X_0,\ldots,X_{\ell-1}$,
and $\|M\| \leq C$ for a constant $C>0$ (depending on $\z$ and $\w$) since
$X_0,\ldots,X_{\ell-1}$ are $(\eps_n,B)$-orthonormal. Then, applying Lemma
\ref{lemma:fixedpoint}(a),
\[\tr \Big[\Big(\z \cdot \bX^\top \bX(\ell)\Big)^{-1}
\Big(\w_\prev \cdot \bX^\top \bX(\ell-1)\Big)\Big]-
\tr \Big[(A+\bar{s}^{-1}\Phi_\ell-z\Id)^{-1}
\Big(\w_\prev \cdot \bX^\top \bX(\ell-1)\Big)\Big] \to 0.\]
By the same arguments as above, we may replace $\bar{s}$ by $s_0=s_\ell(\z)$
and $\Phi_\ell$ by $\tilde{\Phi}_\ell$. Then, applying this to
(\ref{eq:wreduce}),
\[\tr \Big[\Big(\z \cdot \bX^\top \bX(\ell)\Big)^{-1}
\Big(\w \cdot \bX^\top \bX(\ell)\Big)\Big]
-\frac{w_\ell}{z_\ell}-
\tr \Big[(A+s_\ell(\z)^{-1}\tilde{\Phi}_\ell-z\Id)^{-1}
\Big(\w_\prev \cdot \bX^\top \bX(\ell-1)\Big)\Big] \to 0.\]
Finally, applying that part (b) holds for $\ell-1$, this yields
\[\lim_{n \to \infty} \tr
\Big[\Big(\z \cdot \bX^\top \bX(\ell)\Big)^{-1}
\Big(\w \cdot \bX^\top \bX(\ell)\Big)\Big]
=\frac{w_\ell}{z_\ell}+t_{\ell-1}(\z_\prev(s_\ell(\z),\z),\w_\prev),\]
which is the definition of $t_\ell(\z,\w)$. This establishes
(\ref{eq:tliscorrect}).

For any fixed $z_{-1},\ldots,z_\ell \in \R$ where $z_\ell \neq 0$, and any fixed
$z \in \C^+$, this implies that the Stieltjes transform of $\z \cdot
\bX^\top \bX(\ell)$ has the almost sure limit
\[m(z)=t_\ell\Big((-z+z_{-1},z_0,\ldots,z_\ell),(1,0,\ldots,0)\Big).\]
So $m(z)$ defines the Stieltjes transform of a sub-probability
distribution $\nu$, and the empirical eigenvalue distribution of
$\z \cdot \bX^\top \bX(\ell)$ converges vaguely a.s.\ to $\nu$.
Since $\|\z \cdot \bX^\top \bX(\ell)\|$ is bounded because
$X_0,\ldots,X_L$ are $(\eps_n,B)$-orthonormal, this limit $\nu$ must in
fact be a probability distribution, and the eigenvalue distribution converges
weakly to $\nu$. This concludes the induction and the proof.
\end{proof}

\section{Multi-dimensional outputs and rescaled
parametrizations}\label{appendix:NTKmulti}

In this section, we provide some motivation for the form of the NTK 
in (\ref{eq:NTKmulti}) for networks with a $k$-dimensional output,
and we prove Theorem \ref{thm:NTKmulti} regarding its spectrum.

\subsection{Derivation of (\ref{eq:NTKmulti}) from gradient flow training}
\label{appendix:NTKmultiderivation}

Consider gradient flow training of the network (\ref{eq:NNfuncmulti}), with
training samples
$(\x_\a,\y_\a)_{\a=1}^n$ where $\x_\a \in \R^{d_0}$ and $\y_\a \in \R^k$,
under the general training loss
\[F(\theta)=\sum_{\a=1}^n \mathcal{L}(f_\theta(\x_\a),\y_\a).\]
Here, $\mathcal{L}:\R^k \times \R^k \to \R$ is the loss function.
We denote by $\nabla \mathcal{L}(f_\theta(\x_\a),\y_\a) \in \R^k$ the gradient
of $\mathcal{L}$ with respect to its first argument, and by
$\nabla_{W_\ell} f_\theta(\x_\a)\in \R^{\dim(W_\ell) \times k}$ the Jacobian
of $f_\theta(\x_\a)$ with respect to the weights $W_\ell$.

Consider a possibly reweighted gradient-flow training of $\theta$,
where the evolution of weights $W_\ell$ is given by
\begin{align*}
\frac{d}{dt}W_\ell(t)&=-\tau_\ell \cdot
\nabla_{W_\ell} F(\theta(t))=-\tau_\ell
\sum_{\a=1}^n \nabla_{W_\ell} f_{\theta(t)}(\x_\a) \cdot \nabla
\mathcal{L}(f_{\theta(t)}(\x_\a),\y_\a).
\end{align*}
The learning rate for each weight matrix $W_\ell$ is
scaled by a constant $\tau_\ell$---this may arise, for example,
from reparametrizing the network (\ref{eq:NNfuncmulti}) using
$\widetilde{W}_\ell=\tau_\ell^{-1} \cdot W_\ell$ and considering
gradient flow training for $\widetilde{W}_\ell$. Denoting
the vectorization of all training predictions and its Jacobian by
\[f_\theta(X)=(f_\theta^1(X),\ldots,f_\theta^k(X)) \in \R^{nk},
\qquad \nabla_{W_\ell} f_\theta(X) \in \R^{\dim(W_\ell) \times nk},\]
and the corresponding vectorization of
$(\nabla \mathcal{L}(f_\theta(\x_\a),\y_\a))_{\a=1}^n$
by $\nabla \mathcal{L}(f_\theta(X),\y) \in \R^{nk}$,
this may be written succinctly as
\[\frac{d}{dt}W_\ell(t)=-\tau_\ell \cdot \nabla_{W_\ell} f_{\theta(t)}(X)
\cdot \nabla \mathcal{L}(f_{\theta(t)}(X),\y).\]
Then the time evolution of in-sample predictions is given by
\begin{align*}
\frac{d}{dt} f_{\theta(t)}(X)
&=\Big(\nabla_\theta f_{\theta(t)}(X)\Big)^\top
\cdot \frac{d}{dt}\theta(t)\\
&=-\sum_{\ell=1}^{L+1} \tau_\ell 
\Big(\nabla_{W_\ell} f_{\theta(t)}(X)\Big)^\top
\Big(\nabla_{W_\ell} f_{\theta(t)}(X)\Big) \cdot \nabla
\mathcal{L}(f_{\theta(t)}(X),\y)\\
&=-K^\NTK(t) \cdot \nabla \mathcal{L}(f_{\theta(t)}(X),\y),
\end{align*}
where $K^\NTK$ is the matrix defined in (\ref{eq:NTKmulti}).
For $\tau_1=\ldots=\tau_{L+1}=1$, this matrix is simply
\[K^\NTK=\Big(\nabla_\theta f_\theta(X)\Big)^\top
\Big(\nabla_\theta f_\theta(X)\Big) \in \R^{nk \times nk},\]
which is a flattening of the neural
tangent kernel $K \in \R^{n \times n \times k \times k}$ (identified as a map
$K:\R^{n \times n} \to \R^{k \times k}$) that is defined in
\cite{jacot2018neural}.

\subsection{Proof of Theorem \ref{thm:NTKmulti}}

The matrix $K^\NTK$ in (\ref{eq:NTKmulti})
admits a $k \times k$ block decomposition
\[K^\NTK=\begin{pmatrix} K_{11}^\NTK & \cdots & K_{1k}^\NTK \\
\vdots & \ddots & \vdots \\
K_{k1}^\NTK & \cdots & K_{kk}^\NTK \end{pmatrix},
\qquad K_{ij}^\NTK=\sum_{\ell=1}^{L+1}
\tau_\ell \Big(\nabla_{W_\ell} f_\theta^i(X)\Big)^\top
\Big(\nabla_{W_\ell} f_\theta^j(X)\Big) \in \R^{n \times n}.\]
Writing
\[W_{L+1}=\begin{pmatrix} \w_1^\top \\ \vdots \\ \w_k^\top \end{pmatrix},\]
a computation using the chain rule similar to (\ref{eq:NTKform})
verifies that
\[K_{ij}^\NTK=\1\{i=j\}\tau_{L+1}X_L^\top X_L+\sum_{\ell=1}^L \tau_\ell
({S_\ell^i}^\top S_\ell^j) \odot (X_{\ell-1}^\top X_{\ell-1})\]
where $S_\ell^i \in \R^{d_\ell \times n}$ is the matrix with the same
column-wise definition as in (\ref{eq:Sell}), replacing $\w$ by $\w_i$.

\begin{lemma}\label{lemma:NTKFapproxmulti}
Under the assumptions of Theorem \ref{thm:NTKmulti},
for any indices $i \neq j \in [k]$, almost surely as $n
\to \infty$,
\[\frac{1}{n}\|K_{ij}^\NTK\|_F^2 \to 0.\]
Furthermore, for a constant $C>0$, almost surely for all large $n$,
$\|K_{ij}^\NTK\| \leq C$.
\end{lemma}
\begin{proof}
By Corollary \ref{cor:orthonormalinduction}, we may assume that each
$X_0,\ldots,X_L$ is $(\eps_n,B)$-orthonormal.

Let us fix $i,j,\ell$ and denote the columns of $S_\ell^i$ and $S_\ell^j$ by
$\s_\a^{\ell,i}$ and $\s_\b^{\ell,j}$ for $\a,\b \in [n]$.
We apply the Hanson-Wright inequality conditional on $W_1,\ldots,W_L$, which is
similar to (\ref{eq:HansonWright}). However, since $\w_i$ and $\w_j$ are
independent, there is no trace term, and we obtain instead
\[\Big|{\s_\a^{\ell,i}}^\top \s_\b^{\ell,j}\Big| \leq C\eps_n\sqrt{n}
\frac{1}{d_L}\|M_L\|_F\]
for both $\a=\b$ and $\a \neq \b$ with probability $1-e^{-n^{0.01}}$,
where $M_L$ is the same matrix as defined in (\ref{eq:Mell}).
Applying the bound $\|M_L\|_F \leq C\sqrt{n}$ as
in the proof of Corollary \ref{cor:Sapprox}, this yields
\[\Big|{\s_\a^{\ell,i}}^\top \s_\b^{\ell,j}\Big| \leq C\eps_n\]
almost surely for all $\a,\b \in [n]$ and all large $n$. Combining with the
$(\eps_n,B)$-orthonormality of $X_{\ell-1}$, we get for $\a \neq \b$ that
\[\Big|({S_\ell^i}^\top S_\ell^j) \odot (X_{\ell-1}^\top X_{\ell-1})[\a,\b]\Big|
\leq C\eps_n^2, \qquad
\Big|({S_\ell^i}^\top S_\ell^j) \odot (X_{\ell-1}^\top X_{\ell-1})[\a,\a]\Big|
\leq C\eps_n.\]
Then
\[\|({S_\ell^i}^\top S_\ell^j) \odot (X_{\ell-1}^\top X_{\ell-1})\|_F^2
\leq Cn(n-1)\eps_n^4+Cn\eps_n^2,\]
and the first statement follows from the assumption $\eps_n n^{1/4} \to 0$.
The second statement on the operator norm follows from the bound
\[\|({S_\ell^i}^\top {S_\ell^j}) \odot (X_{\ell-1}^\top X_{\ell-1})\|
\leq \left(\max_{\a=1}^n \Big|{\s^{\ell,i}_{\a}}^\top
{\s^{\ell,i}_{\a}}\Big|\right)^{1/2}
\left(\max_{\a=1}^n \Big|{\s^{\ell,j}_\a}^\top
\s^{\ell,j}_\a\Big|\right)^{1/2} \cdot \|X_{\ell-1}^\top X_{\ell-1}\|.\]
See \cite[Eq.\ (3.7.9)]{johnson1990matrix} applied with $X=S_\ell^i$ and
$Y=S_\ell^j$. The bound $\|K_{ij}^\NTK\| \leq C$ then follows from
the $(\eps_n,B)$-orthonormality of
$X_{\ell-1}$ and Corollary \ref{cor:Sapprox}, applied to $S_\ell^i$ and
$S_\ell^j$.
\end{proof}

Applying this lemma together with Proposition \ref{prop:specapprox},
we obtain
\[\limspec K^\NTK =\limspec \begin{pmatrix} K_{11}^\NTK & & \\
& \ddots & \\ & & K_{kk}^\NTK \end{pmatrix}\]
where the off-diagonal blocks $K_{ij}^\NTK$ may be replaced by 0.
Then the limit spectral distribution of $K^\NTK$ is an equally weighted mixture of those
of $K_{11}^\NTK,\ldots,K_{kk}^\NTK$. For each diagonal block
$K_{ii}^{\NTK}$, the argument of Lemma \ref{lemma:NTKFapprox} shows that
\[\limspec K_{ii}^{\NTK}=\limspec
\left(\tau \cdot r_+\Id+\tau_{L+1}X_L^\top X_L+\sum_{\ell=0}^{L-1}
\tau_{\ell+1}q_\ell X_\ell^\top X_\ell\right).\]
Then by Theorem \ref{thm:NTK},
each diagonal block $K_{ii}^\NTK$ has the same limit spectral distribution, whose Stieltjes
transform is given by the function $m_\NTK(z)$ in Theorem \ref{thm:NTKmulti}.
Furthermore, since $\|K_{ii}^\NTK\| \leq C$ by Lemma \ref{lemma:NTKFapprox} and
$\|K_{ij}^\NTK\| \leq C$ for $i \neq j$ by Lemma \ref{lemma:NTKFapproxmulti},
this shows $\|K^\NTK\| \leq C$. This establishes Theorem \ref{thm:NTKmulti}.

Again, when $b_\sigma=0$, the limit spectrum of each $K_{ii}^\NTK$ reduces to $\limspec(\tau \cdot r_+\Id+\tau_{L+1}X_L^\top X_L)$, which can be computed via the Stieltjes transform of $\rho^\MP_{\gamma_L}$.

\section{Reduction to result of Pennington and Worah
\cite{pennington2017nonlinear} for one hidden
layer}\label{appendix:penningtonreduction}

Consider the one-hidden-layer conjugate kernel
\[K^\CK=X_1^\top X_1=\frac{1}{d_1}\sigma(W_1X)^\top \sigma(W_1X) \in
\R^{n\times n}.\]
Define an associated covariance matrix
\begin{equation}\label{eq:CKcompanion}
M=\frac{1}{n}\sigma(W_1X)\sigma(W_1X)^\top \in \R^{d_1 \times d_1},
\end{equation}
and observe that the eigenvalues of $K^\CK$ are those of $M$
multiplied by $n/d_1$ and padded by $n-d_1$ additional
zeros (or with $d_1-n$ zeros removed, if $n-d_1<0$). 
\cite[Theorem 1]{pennington2017nonlinear} characterizes the
limit spectral distribution of $M$ in terms of a quartic equation in its Stieltjes
transform, under the additional assumptions that $X$ has i.i.d.\ $\N(0,1/d_0)$
entries and $n/d_0 \to \gamma_0 \in (0,\infty)$.\footnote{In \cite{pennington2017nonlinear}, the
$1/\sqrt{d_0}$ scaling is in $W_1$ rather than $X$, but these are clearly the
same. We consider $\sigma_w=\sigma_x=1$ and
$\eta=1$ in the results of \cite{pennington2017nonlinear}.}
By Theorem \ref{thm:CK}, this should be equivalent to the description
\begin{equation}\label{eq:onelayerCK}
\limspec K^\CK
=\rho_{\gamma_1}^\MP \boxtimes \Big((1-b_\sigma^2)+b_\sigma^2 \mu_0\Big)
\end{equation}
for the limit spectrum of $K^\CK$, if we specialize to
$\mu_0=\rho_{\gamma_0}^\MP$ being the Marcenko-Pastur limit
of the input gram matrix $X^\top X$. We derive this equivalence in
this section.

Let $m_K(z)$ and $m_M(z)$ be the \emph{limit} Stieltjes transforms for $K^\CK$
and $M$. For any $z \in \C^+$, by the relation between the eigenvalues of
$K^\CK$ and $M$,
\begin{align*}
\frac{1}{n}\Tr\left(K^\CK-\frac{n}{d_1}z\Id\right)^{-1}&=\frac{n-d_1}{n}\left(-\frac{n}{d_1}z\right)^{-1}
+\frac{1}{n}\Tr\left(\frac{n}{d_1}M-\frac{n}{d_1}z\Id\right)^{-1}\\
&=-\left(1-\frac{d_1}{n}\right)\frac{d_1}{n}\cdot
\frac{1}{z}+\left(\frac{d_1}{n}\right)^2
\cdot \frac{1}{d_1}\Tr(M-z\Id)^{-1}.
\end{align*}
Taking the limit on both sides, we obtain the relation between $m_K(z)$ and
$m_M(z)$, which is
\begin{equation}\label{eq:CKMrelation}
m_K(\gamma_1 z)=-\left(1-\frac{1}{\gamma_1}\right)\frac{1}{\gamma_1
z}+\frac{1}{\gamma_1^2}m_M(z)
=\frac{1}{\gamma_1^2}\left(m_M(z)+\frac{1-\gamma_1}{z}\right).
\end{equation}

Following the notation of \cite{pennington2017nonlinear}, let us set
\begin{equation}\label{eq:PWCKnotation}
\phi=1/\gamma_0,\quad \psi=\gamma_1/\gamma_0,
\quad \eta=1=\E[\sigma(\xi)^2], \quad \zeta=b_\sigma^2.
\end{equation}
\cite[Theorem 1]{pennington2017nonlinear} characterizes $G(z) \equiv -m_M(z)$
as the root of a
quartic equation. Defining three $z$-dependent quantities
$P,P_\phi,P_\psi$ by
\begin{equation}\label{eq:Pequations}
G(z)=\frac{\psi}{z}P+\frac{1-\psi}{z},
\quad P_\phi=1+(P-1)\phi, \quad P_\psi=1+(P-1)\psi,
\end{equation}
this quartic equation is expressed as
\begin{equation}\label{eq:quarticP}
P=1+(1-\zeta)tP_\phi P_\psi+\frac{\zeta tP_\phi P_\psi}{1-\zeta tP_\phi P_\psi} \qquad \text{ where } \qquad t=\frac{1}{z\psi},
\end{equation}
see \cite[Equations (10--12)]{pennington2017nonlinear}.

To verify that (\ref{eq:onelayerCK}) is equivalent to this equation
(\ref{eq:quarticP}), note that (\ref{eq:onelayerCK})
means the Stieltjes transform $m_K(z)$ is defined by the Marcenko-Pastur
equation (\ref{eq:MPeq}) as
\begin{equation}\label{eq:MP}
m_K(z)=\int \frac{1}{[(1-b_\sigma^2)+b_\sigma^2 x][1-\gamma_1-\gamma_1 z
m_K(z)]-z}d\mu_0(x).
\end{equation}
Applying the identity $1-\gamma_1-\gamma_1^2 zm_K(\gamma_1 z)
=-zm_M(z)$ from rearranging (\ref{eq:CKMrelation}), and applying also
$\zeta=b_\sigma^2$ in (\ref{eq:PWCKnotation}),
\begin{equation}\label{eq:CKfixedpoint1}
m_K(\gamma_1 z)=\int \frac{1}{[(1-\zeta)+\zeta x][-zm_M(z)]-\gamma_1 z}
d\mu_0(x).
\end{equation}
When $X$ has i.i.d.\ $\N(0,1/d_0)$ entries, the limit spectral distribution
of $X^\top X$ is the Marcenko-Pastur law $\mu_0=\rho_{\gamma_0}^\MP$.
The Stieltjes transform $m(z)$ of this law $\mu_0=\rho_{\gamma_0}^\MP$ is
characterized by the quadratic equation
\[1=m(z)[1-\gamma_0-\gamma_0zm(z)-z]\]
(which is the specialization of (\ref{eq:MPeq}) when $\mu$ is the point
distribution at 1). Defining
\[g(a,b)=\int \frac{1}{ax-b}d\mu_0(x)=\frac{1}{a}m\left(\frac{b}{a}\right),\]
we obtain then that $g(a,b)$ satisfies the quadratic equation
\begin{align*}
1&=g(a,b)[a-\gamma_0a-\gamma_0bm(b/a)-b]\\
&=g(a,b)[(a-b)-\gamma_0a-\gamma_0ab \cdot g(a,b)].
\end{align*}
Applying this with $a=-\zeta zm_M(z)$ and $b=(1-\zeta)zm_M(z)+\gamma_1 z$,
the quantity (\ref{eq:CKfixedpoint1}) is exactly $g(a,b)$. Thus this equation
holds for $g(a,b)=m_K(\gamma_1z)$ and these settings of $(a,b)$, i.e.\
\begin{equation}\label{eq:CKfixedpoint2}
1=m_K(\gamma_1 z)\Big(-zm_M(z)-\gamma_1 z+\gamma_0\zeta zm_M(z)
+\gamma_0\zeta zm_M(z)[(1-\zeta)zm_M(z)+\gamma_1 z]m_K(\gamma_1 z)\Big).
\end{equation}
From the relation (\ref{eq:CKMrelation}), we see that this is a quartic equation
in $m_M(z)$. Note that the definitions of $P_\psi$ and $P_\phi$ in 
(\ref{eq:Pequations}) may be equivalently written as
\begin{align*}
P_\psi&=\psi P+1-\psi=zG(z)=-zm_M(z),\\
P_\phi&=1+\frac{\phi}{\psi}(zG(z)-1)
=\frac{1}{\gamma_1}(-zm_M(z)-1+\gamma_1)
=-\gamma_1 zm_K(\gamma_1 z)
\end{align*}
where we have used $G(z)=-m_M(z)$, $\psi/\phi=\gamma_1$ from
(\ref{eq:PWCKnotation}),
and the relation (\ref{eq:CKMrelation}).
Applying now $\gamma_1 z=(\psi/\phi)z=1/(\phi t)$ and $\gamma_0=1/\phi$,
the equation (\ref{eq:CKfixedpoint2}) becomes
\begin{align*}
1&=-\phi t P_\phi \left(P_\psi-\frac{1}{\phi t}
-\frac{\zeta}{\phi} P_\psi+\frac{\zeta}{\phi} P_\psi\left[
-(1-\zeta)P_\psi+\frac{1}{\phi t}\right]\phi t P_\phi\right)\\
&=-\phi tP_\phi P_\psi+P_\phi+(1-P_\phi)\zeta tP_\phi P_\psi
+\zeta(1-\zeta)\phi(tP_\phi P_\psi)^2.
\end{align*}
This may be rearranged as
\[(1-P_\phi-\phi)(1-\zeta tP_\phi P_\psi)
=-\phi(1-\zeta tP_\phi P_\psi)-\phi tP_\phi
P_\psi+\zeta(1-\zeta)\phi(tP_\phi P_\psi)^2,\]
and dividing both sides by $-\phi(1-\zeta tP_\phi P_\psi)$ yields
\[\frac{1}{\phi}(P_\phi-1)+1
=1+\frac{tP_\phi P_\psi-\zeta(1-\zeta)(tP_\phi P_\psi)^2}{1-\zeta tP_\phi
P_\psi}
=1+(1-\zeta)tP_\phi P_\psi+\frac{\zeta tP_\phi P_\psi}{1-\zeta tP_\phi P_\psi}.\]
Identifying the left side as $P$ by (\ref{eq:Pequations}), we obtain
(\ref{eq:quarticP}) as desired.

\section{Additional simulation results}

\subsection{Pairwise orthogonality of training samples}\label{appendix:orthogonality_for_data}

\xincludegraphics[width=0.33\textwidth,label=a)]{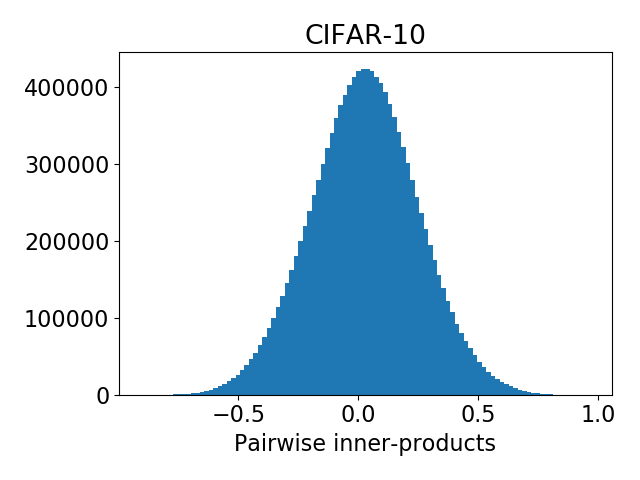}%
\xincludegraphics[width=0.33\textwidth,label=b)]{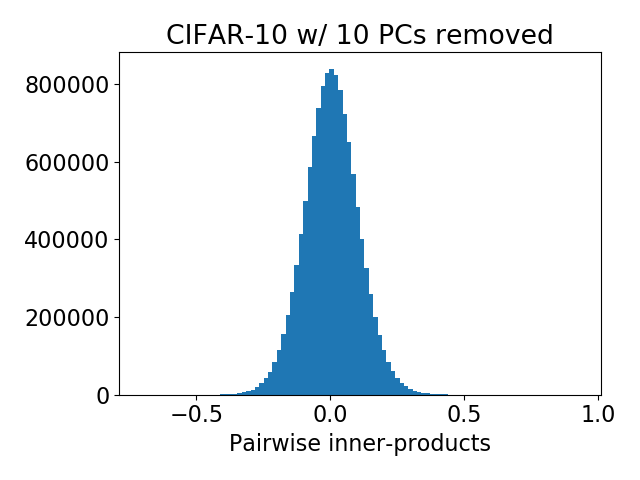}%
\xincludegraphics[width=0.33\textwidth,label=c)]{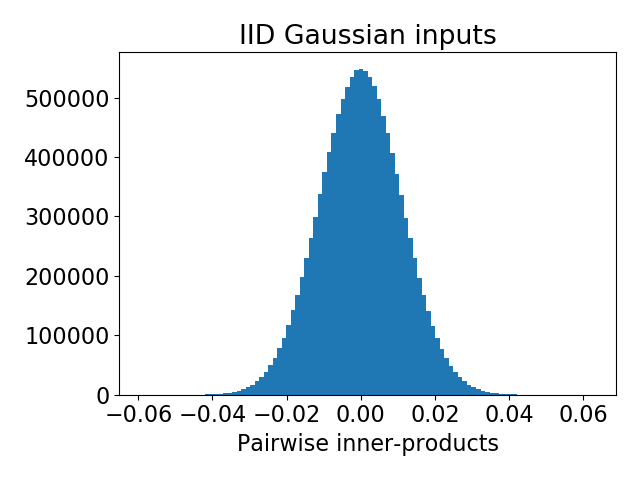}

All pairwise inner-products $\{\x_\a^\top \x_\b:1 \leq \a<\b
\leq n\}$, for (a) 5000 CIFAR-10 training samples, (b) 5000 CIFAR-10 training
samples with the first 10 PCs removed, and (c) i.i.d.\ Gaussian training data
of the same dimensions. Results for (b) were reported in Section
\ref{sec:CIFAR}, and results for (a) are reported below in Appendix
\ref{appendix:CIFARraw}. CIFAR-10 training samples were mean-centered and
normalized to satisfy $\x_\a^\top 1=0$ and $\|\x_\a\|^2=1$ in (a) and (b).

The pairwise inner-products in (a) span a typical range
of $[-0.5,0.5]$. Those in (b) span a range of about
$[-0.2,0.2]$, and those in (c) about $[-0.02,0.02]$. Thus,
with 10 PCs removed, these inner-products for CIFAR-10 are larger than
for i.i.d.\ Gaussian inputs by a factor of 10. We found in Section
\ref{sec:CIFAR} that the inner-products of (b) are sufficiently small for the
observed spectra to match the theoretical limits of
Theorems \ref{thm:CK} and \ref{thm:NTK}.

\subsection{CK and NTK spectra for CIFAR-10 without removal of leading PCs}\label{appendix:CIFARraw}

\xincludegraphics[width=0.33\textwidth,label=a)]{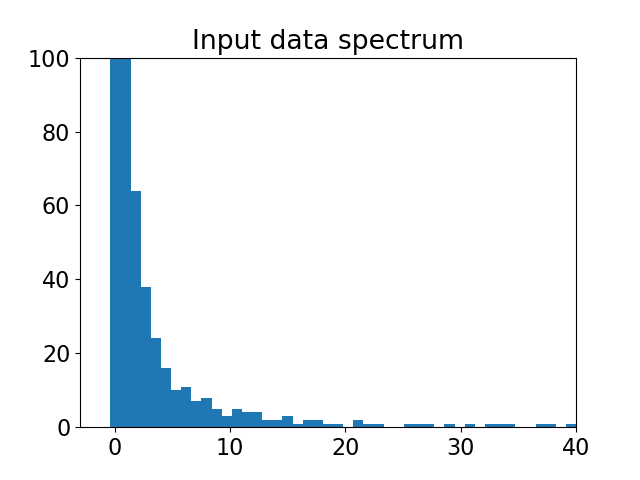}%
\xincludegraphics[width=0.33\textwidth,label=b)]{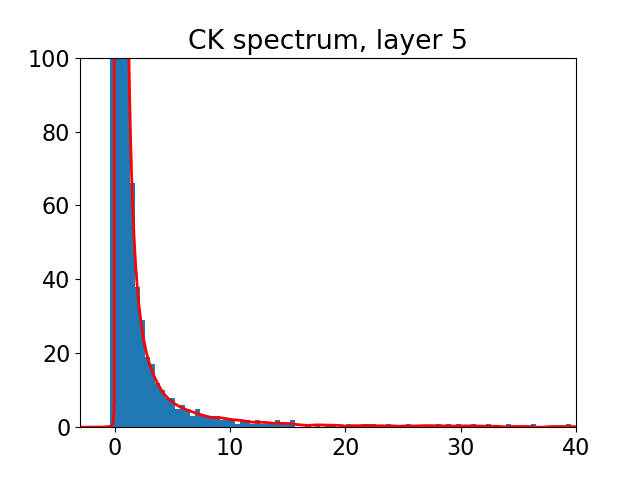}%
\xincludegraphics[width=0.33\textwidth,label=c)]{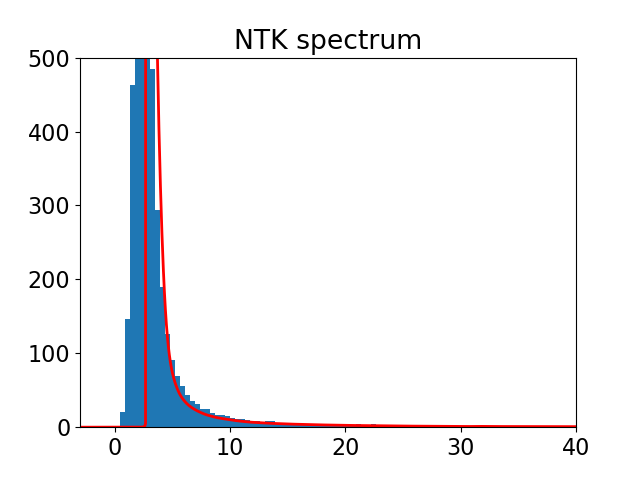}

Same plots as Figure \ref{fig:CIFAR} for CIFAR-10 training samples,
without the removal of the 10 leading PCs. We observe a close agreement of the
observed CK spectrum with the limit spectrum of
Theorem \ref{thm:CK}. However, there is a greater discrepancy of the NTK
spectrum with the limit spectrum of Theorem \ref{thm:NTK} in this setting.

\subsection{Example images of CIFAR-10 with/without leading PCs}\label{appendix:CIFAR_images}

\xincludegraphics[width=0.098\textwidth,label=0)]{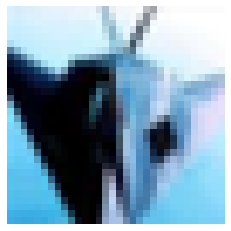}%
\xincludegraphics[width=0.098\textwidth,label=1)]{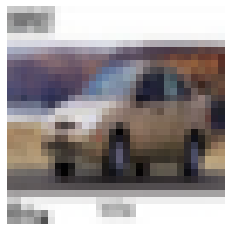}%
\xincludegraphics[width=0.098\textwidth,label=2)]{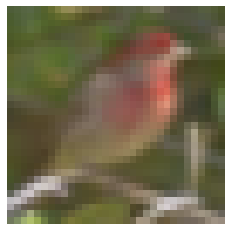}%
\xincludegraphics[width=0.098\textwidth,label=3)]{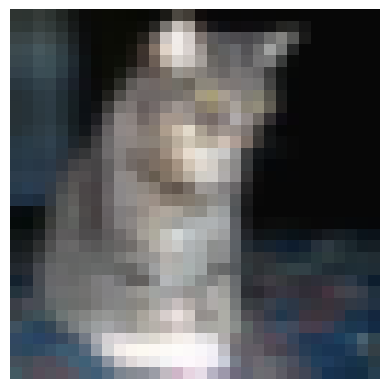}%
\xincludegraphics[width=0.098\textwidth,label=4)]{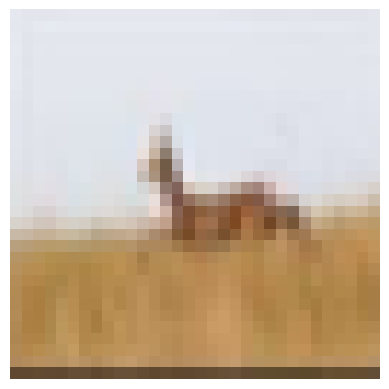}%
\xincludegraphics[width=0.098\textwidth,label=5)]{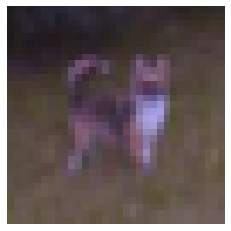}%
\xincludegraphics[width=0.098\textwidth,label=6)]{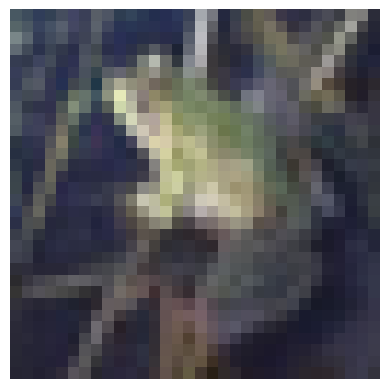}%
\xincludegraphics[width=0.098\textwidth,label=7)]{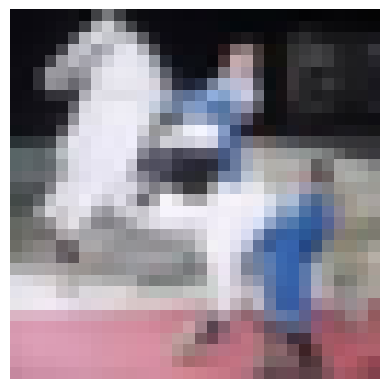}%
\xincludegraphics[width=0.098\textwidth,label=8)]{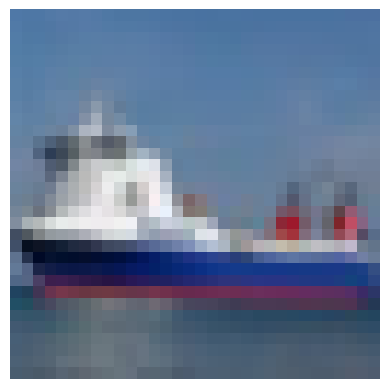}%
\xincludegraphics[width=0.098\textwidth,label=9)]{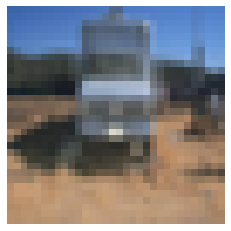}\\
\xincludegraphics[width=0.098\textwidth]{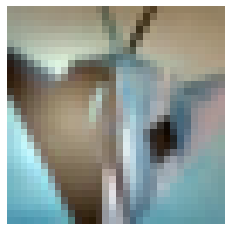}%
\xincludegraphics[width=0.098\textwidth]{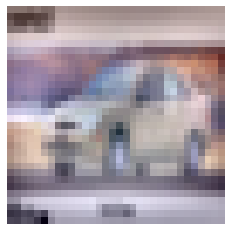}%
\xincludegraphics[width=0.098\textwidth]{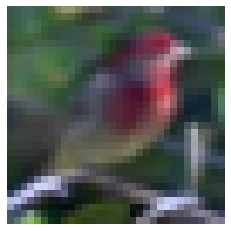}%
\xincludegraphics[width=0.098\textwidth]{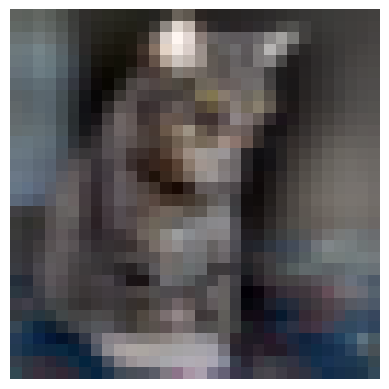}%
\xincludegraphics[width=0.098\textwidth]{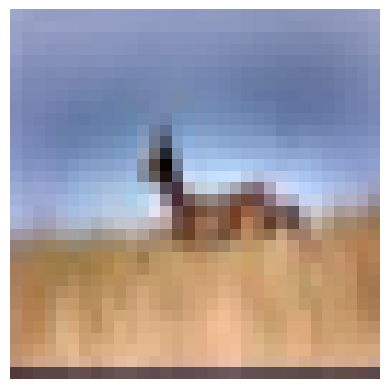}%
\xincludegraphics[width=0.098\textwidth]{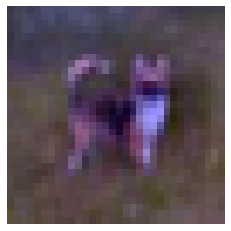}%
\xincludegraphics[width=0.098\textwidth]{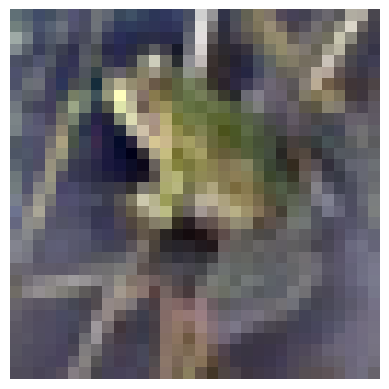}%
\xincludegraphics[width=0.098\textwidth]{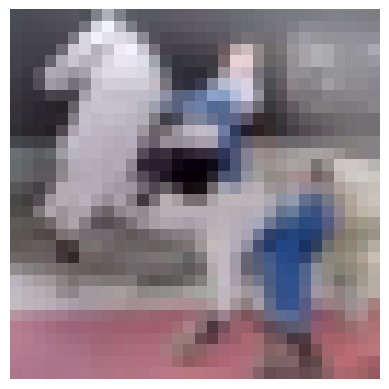}%
\xincludegraphics[width=0.098\textwidth]{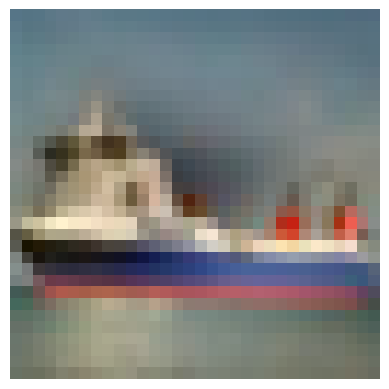}%
\xincludegraphics[width=0.098\textwidth]{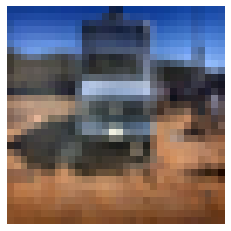}\\

\xincludegraphics[width=0.098\textwidth,label=0)]{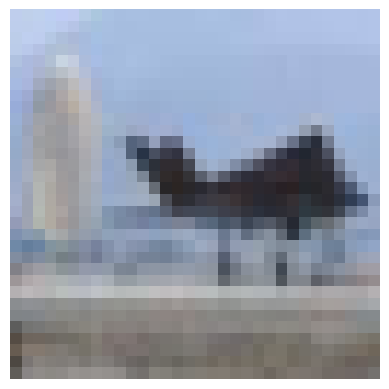}%
\xincludegraphics[width=0.098\textwidth,label=1)]{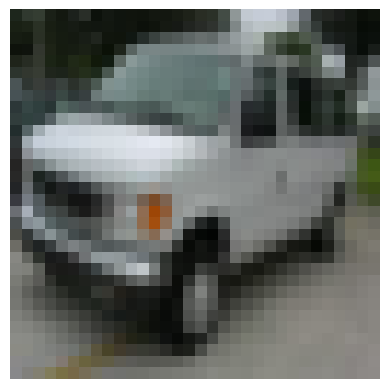}%
\xincludegraphics[width=0.098\textwidth,label=2)]{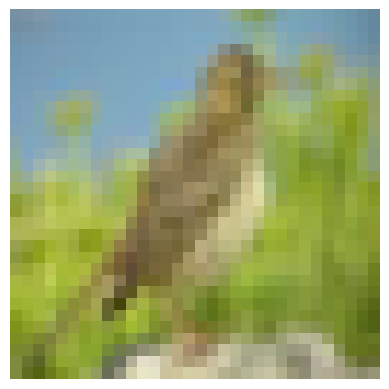}%
\xincludegraphics[width=0.098\textwidth,label=3)]{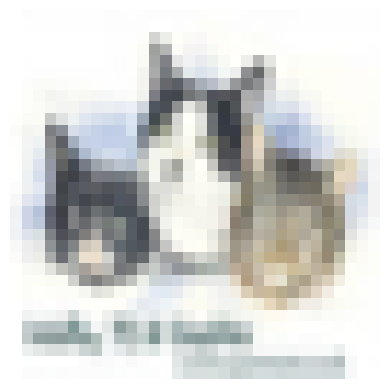}%
\xincludegraphics[width=0.098\textwidth,label=4)]{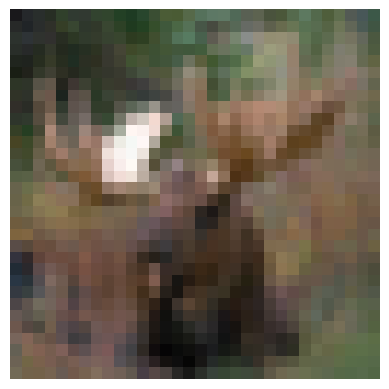}%
\xincludegraphics[width=0.098\textwidth,label=5)]{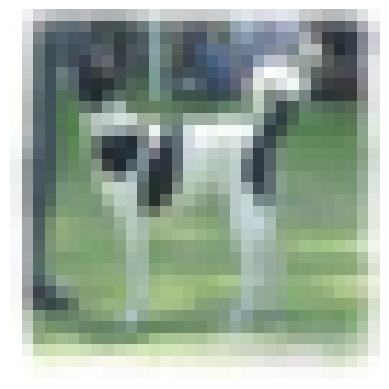}%
\xincludegraphics[width=0.098\textwidth,label=6)]{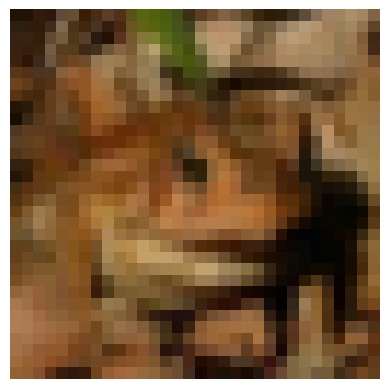}%
\xincludegraphics[width=0.098\textwidth,label=7)]{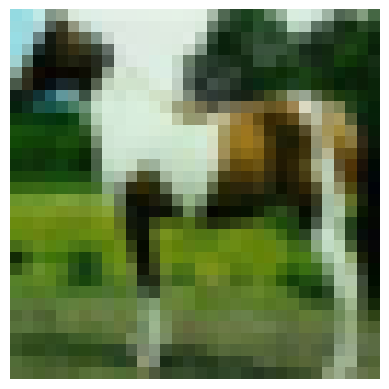}%
\xincludegraphics[width=0.098\textwidth,label=8)]{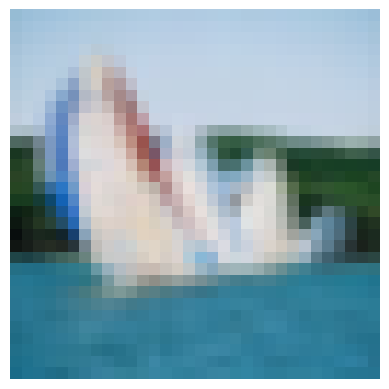}%
\xincludegraphics[width=0.098\textwidth,label=9)]{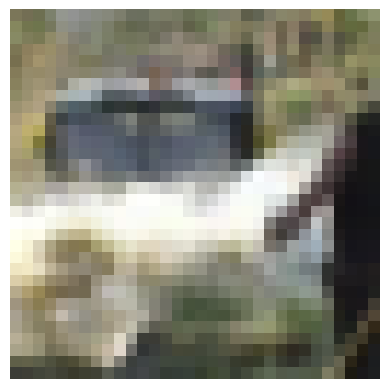}

\xincludegraphics[width=0.098\textwidth]{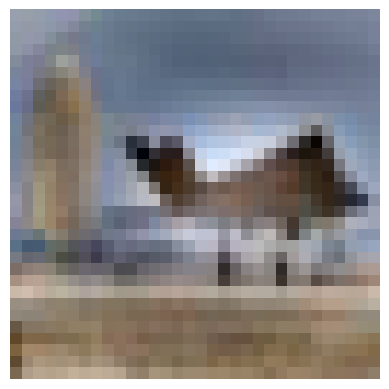}%
\xincludegraphics[width=0.098\textwidth]{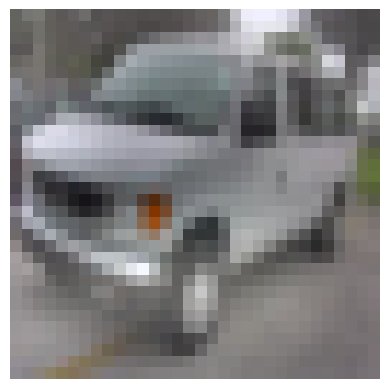}%
\xincludegraphics[width=0.098\textwidth]{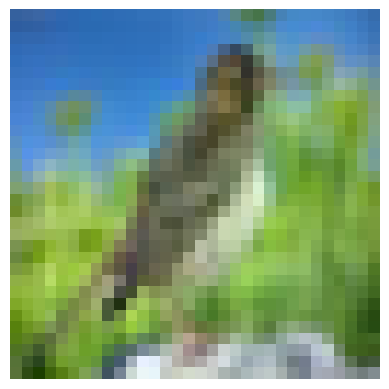}%
\xincludegraphics[width=0.098\textwidth]{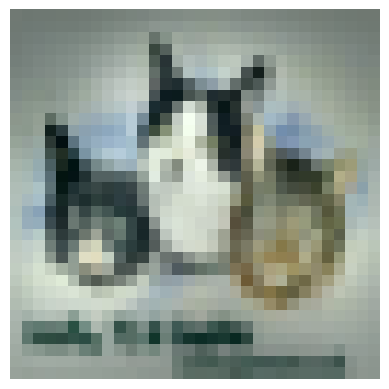}%
\xincludegraphics[width=0.098\textwidth]{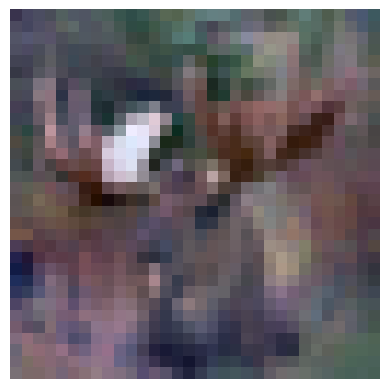}%
\xincludegraphics[width=0.098\textwidth]{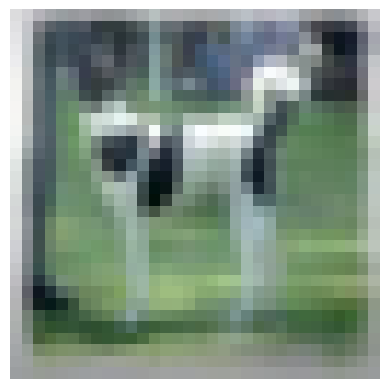}%
\xincludegraphics[width=0.098\textwidth]{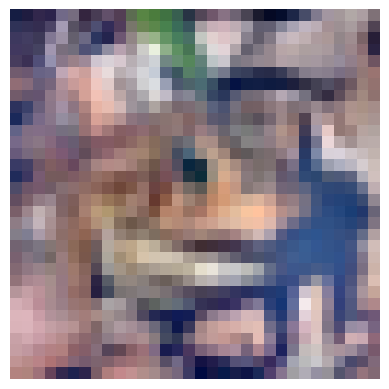}%
\xincludegraphics[width=0.098\textwidth]{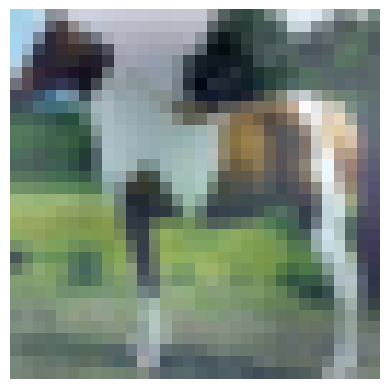}%
\xincludegraphics[width=0.098\textwidth]{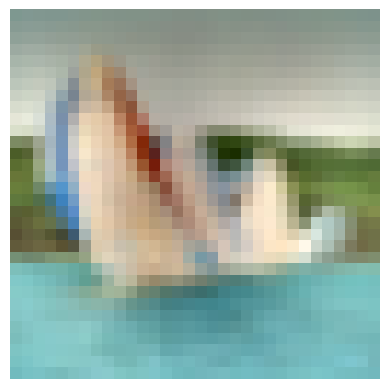}%
\xincludegraphics[width=0.098\textwidth]{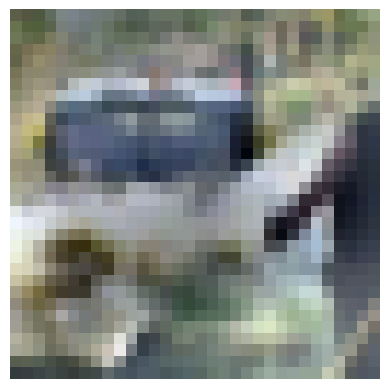}

\xincludegraphics[width=0.098\textwidth,label=0)]{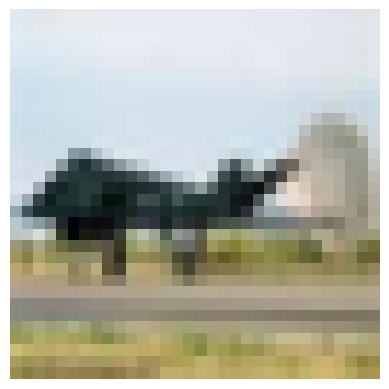}%
\xincludegraphics[width=0.098\textwidth,label=1)]{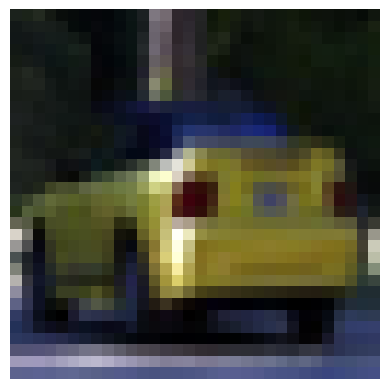}%
\xincludegraphics[width=0.098\textwidth,label=2)]{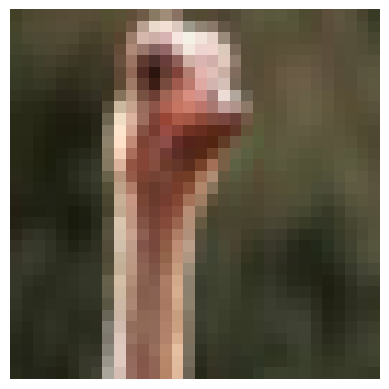}%
\xincludegraphics[width=0.098\textwidth,label=3)]{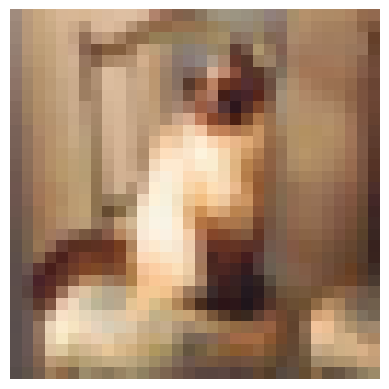}%
\xincludegraphics[width=0.098\textwidth,label=4)]{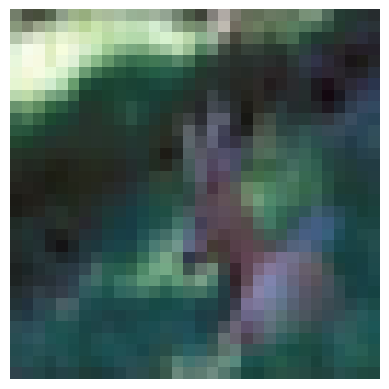}%
\xincludegraphics[width=0.098\textwidth,label=5)]{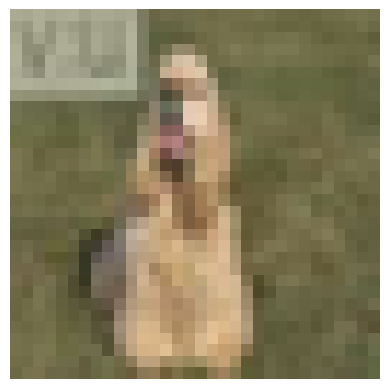}%
\xincludegraphics[width=0.098\textwidth,label=6)]{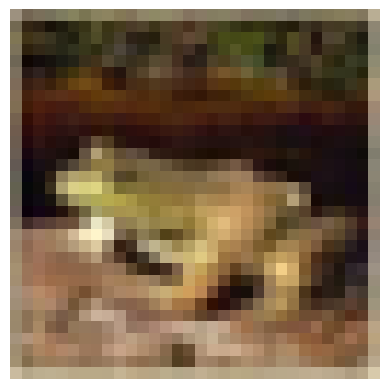}%
\xincludegraphics[width=0.098\textwidth,label=7)]{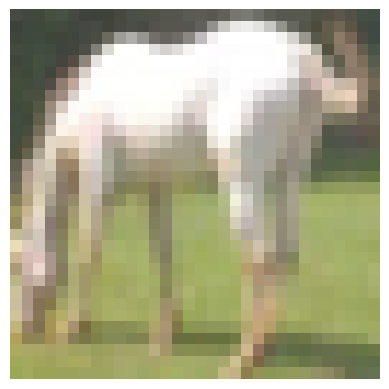}%
\xincludegraphics[width=0.098\textwidth,label=8)]{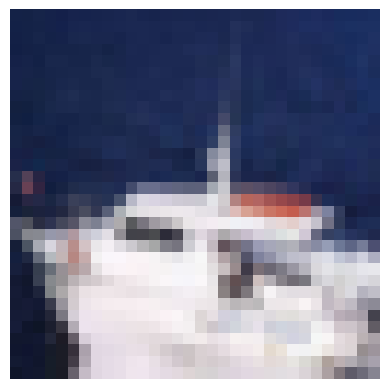}%
\xincludegraphics[width=0.098\textwidth,label=9)]{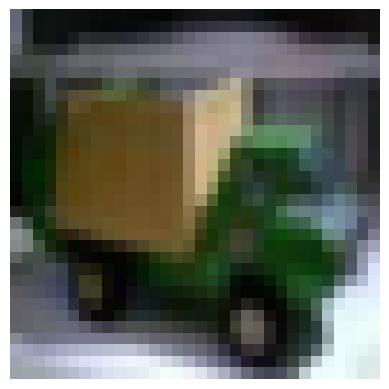}

\xincludegraphics[width=0.098\textwidth]{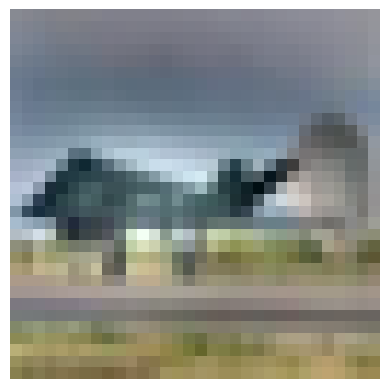}%
\xincludegraphics[width=0.098\textwidth]{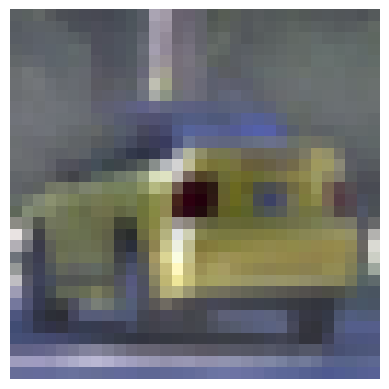}%
\xincludegraphics[width=0.098\textwidth]{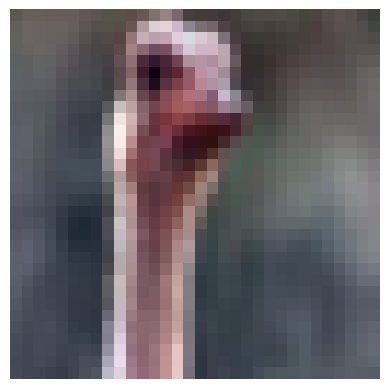}%
\xincludegraphics[width=0.098\textwidth]{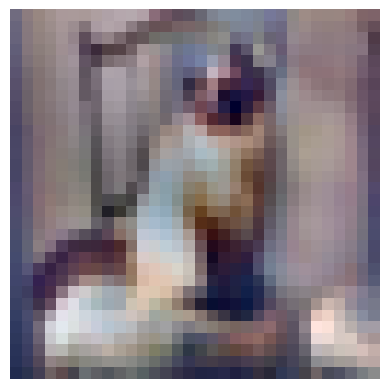}%
\xincludegraphics[width=0.098\textwidth]{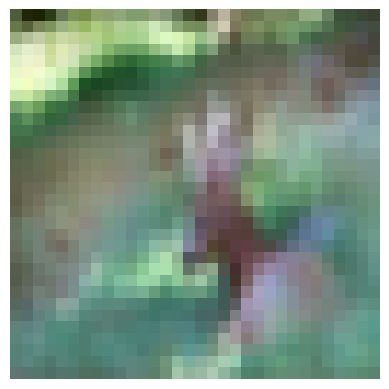}%
\xincludegraphics[width=0.098\textwidth]{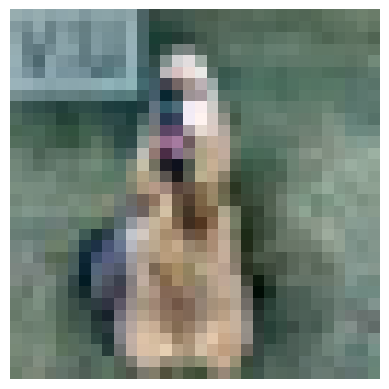}%
\xincludegraphics[width=0.098\textwidth]{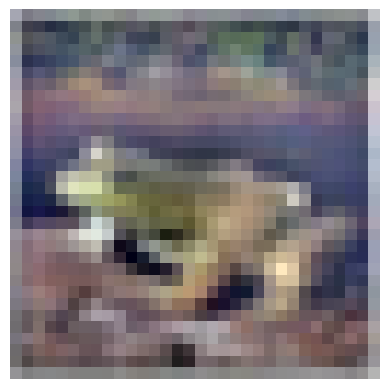}%
\xincludegraphics[width=0.098\textwidth]{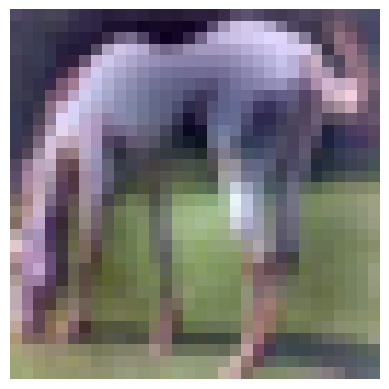}%
\xincludegraphics[width=0.098\textwidth]{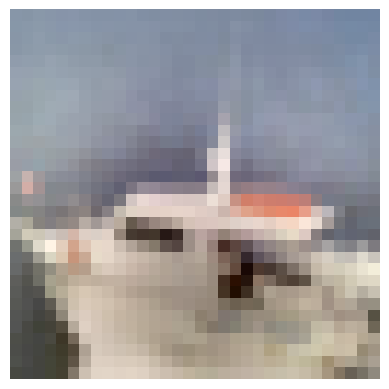}%
\xincludegraphics[width=0.098\textwidth]{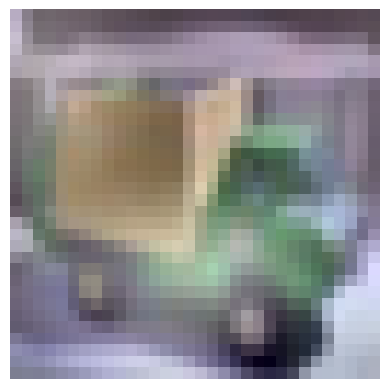}

Example CIFAR-10 training samples for each class. For each training sample, we
compare the original image (above) and the corresponding normalized image upon
removing the top 10 PCs (below). Most of the image details are preserved upon
removing these 10 PCs.
 
\subsection{Observed and limit CK spectra for all layers}\label{appendix:alllayers}

\includegraphics[width=0.33\textwidth]{figures/gaussian1_X0.png}%
\includegraphics[width=0.33\textwidth]{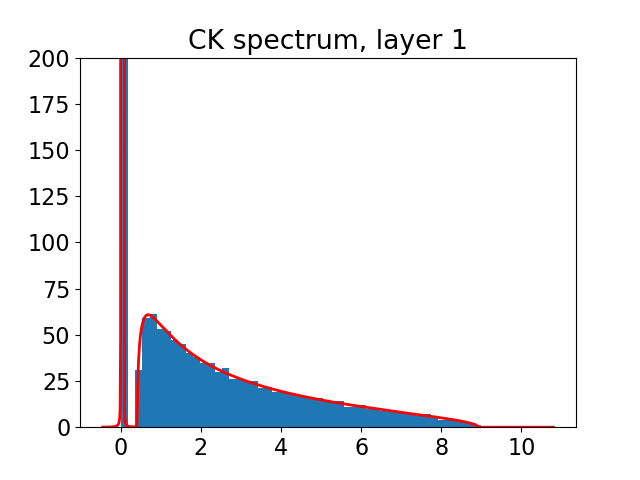}%
\includegraphics[width=0.33\textwidth]{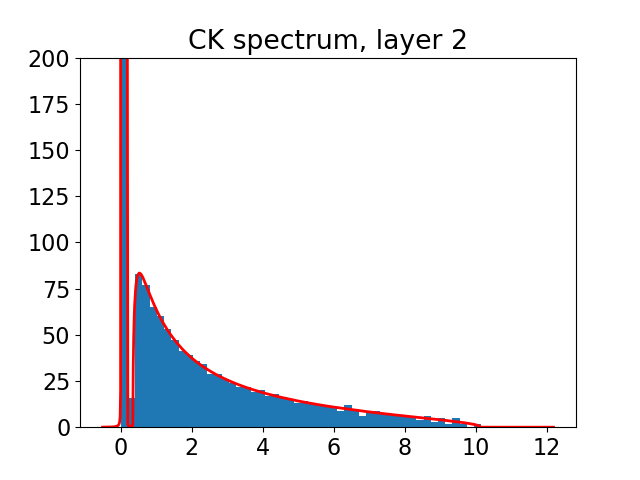}

\includegraphics[width=0.33\textwidth]{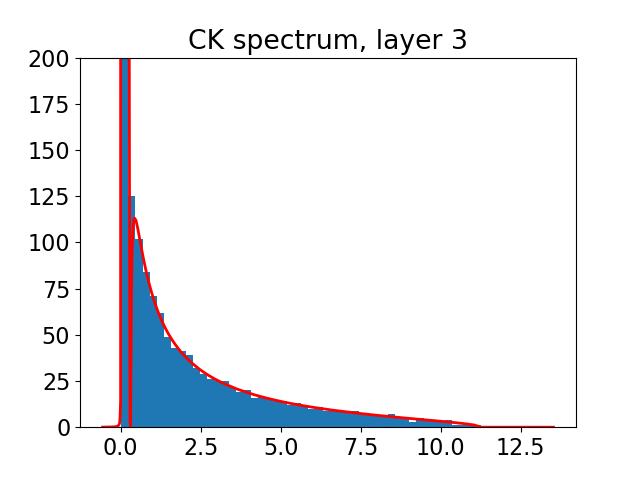}%
\includegraphics[width=0.33\textwidth]{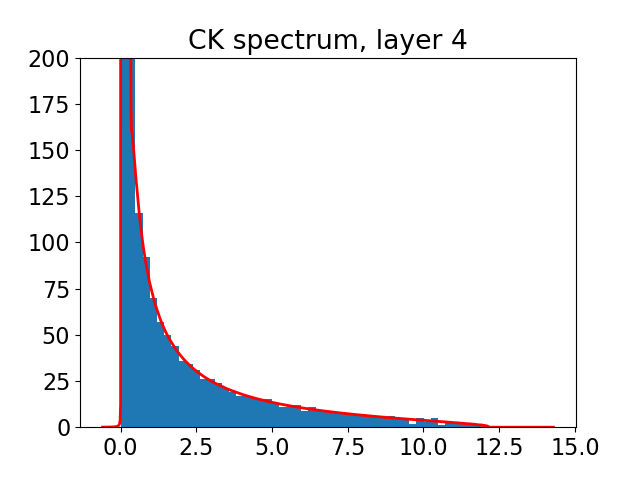}%
\includegraphics[width=0.33\textwidth]{figures/gaussian1_X5.png}

Simulated spectra of the initial CK matrices $X_\ell^\top
X_\ell$ at all intermediate
layers $\ell=1,\ldots,5$, corresponding to the i.i.d.\ Gaussian
training data example of Figure \ref{fig:gaussian}. Numerical computations
of the limit spectra from Theorem \ref{thm:CK} are overlaid in red. We observe a
merging of the two bulk spectral components and an extension of the spectral
support with increase in layer number.

\includegraphics[width=0.33\textwidth]{figures/CIFAR_raw_X0.png}%
\includegraphics[width=0.33\textwidth]{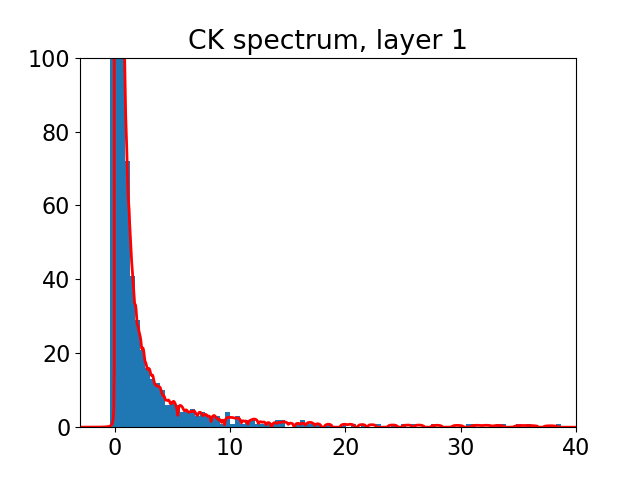}%
\includegraphics[width=0.33\textwidth]{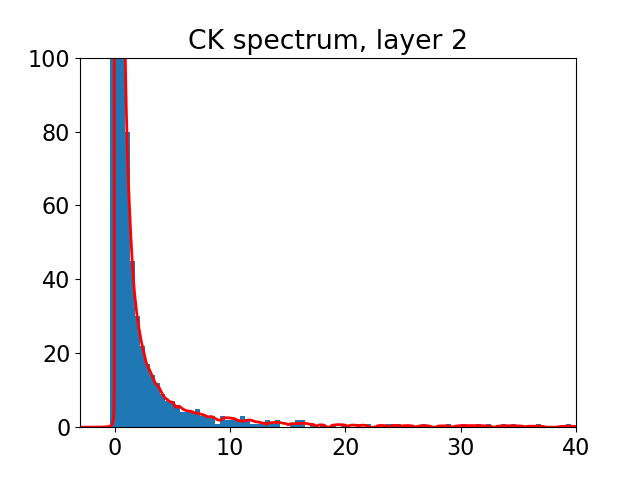}

\includegraphics[width=0.33\textwidth]{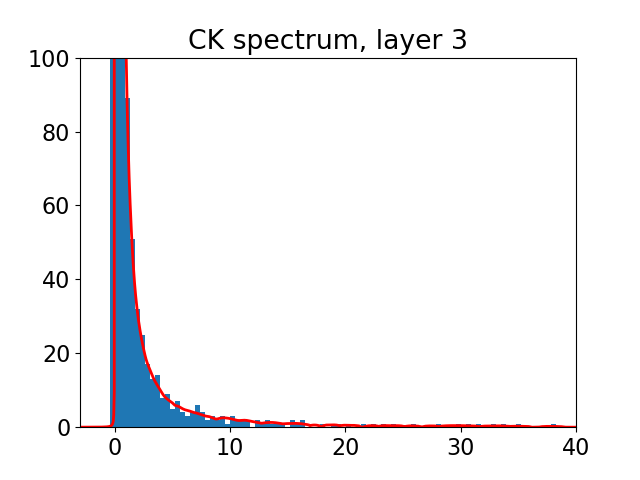}%
\includegraphics[width=0.33\textwidth]{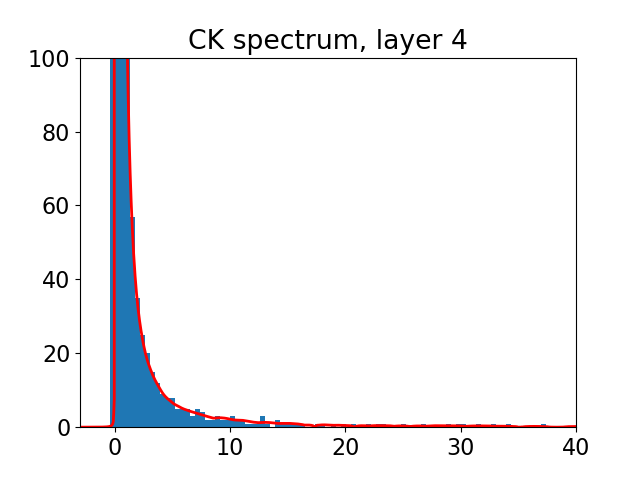}%
\includegraphics[width=0.33\textwidth]{figures/CIFAR_raw_X5.png}%

The same as above, corresponding to the CIFAR-10 training samples in Appendix
\ref{appendix:CIFARraw}. (Results with 10 PCs removed look the same.)
A close agreement with the limit spectrum described by Theorem
\ref{thm:CK} is observed at each layer.

\includegraphics[width=0.33\textwidth]{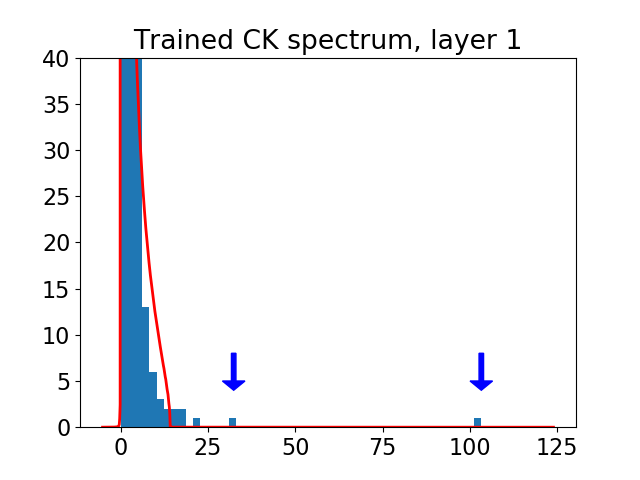}%
\includegraphics[width=0.33\textwidth]{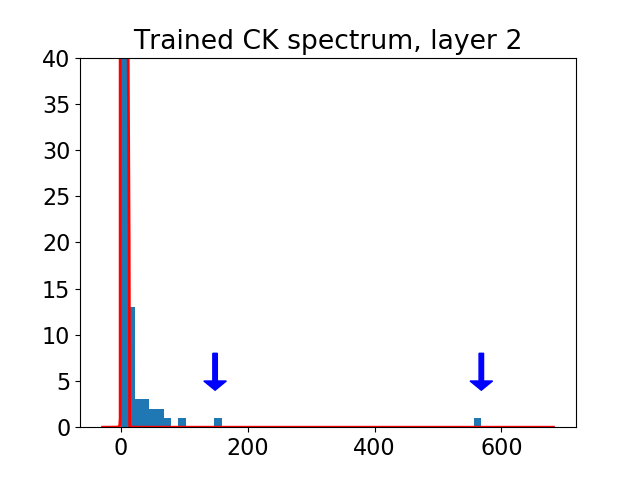}%
\includegraphics[width=0.33\textwidth]{figures/trained_X3.png}%

Spectra of the CK matrices at all three layers, corresponding to the
trained 3-layer network of Section \ref{sec:training}. The limit spectra at
random initialization of weights are depicted in red, and the two largest
eigenvalues of each matrix are depicted by blue arrows.

\subsection{CK spectrum after training on a CIFAR-10
example}\label{appendix:CIFARtraining}

We train a binary classifier on $n=10000$ training samples from CIFAR-10,
corresponding to classes 0 (airplane) and 1 (automobile). The classifier is a
fully-connected network with $L=4$ hidden layers of dimensions
$d_1=\ldots=d_4=1000$, with bias terms and a normalized sigmoid activation at each hidden
layer and also at the output layer. This network is given by
\[f_\theta(\x)=\sigma(\w^\top \x^L+b),
\qquad \x^\ell=\frac{1}{\sqrt{d_\ell}}\sigma(W_\ell \x^{\ell-1}+\mathbf{b}_\ell)
\quad \text{ for } \quad \ell=1,\ldots,L\]
where $b \in \R$ and $\mathbf{b}_\ell \in \R^{d_\ell}$ for each
$\ell=1,\ldots,L$ are the bias parameters.
The activation function $\sigma(x) \propto
(1-e^{-x})/(1+e^{-x})$ is scaled such that $\E[\sigma(\xi)^2]=1$.
Weights $\theta=(\vec(W_1),\ldots,\vec(W_4),\w)$ are initialized to independent $\N(0,1)$ for each entry,
and biases $(\mathbf{b}_1,\ldots,\mathbf{b}_4,b)$ are initialized to 0.
Hence, $K^\CK$ at random initialization has the same definition as in the
main text.

We train the weights and biases using the Adam optimizer in Keras, with
learning rate 0.01, batch size 128, and 60 training epochs. To ensure that the
leading PCs of the \emph{untrained} kernel matrix $K^\CK$ are
not too predictive of the training labels, and to better separate the original
PCs from those that emerge after training, we remove the
leading 5 PCs of the input data before training. The resulting
0--1 classification accuracy on the CIFAR-10 test set
is $85.3\%$. (Training without removing these 5 PCs yields a slightly higher
test accuracy of $90.7\%$, using the same network architecture.)

\xincludegraphics[width=0.33\textwidth,label=a)]{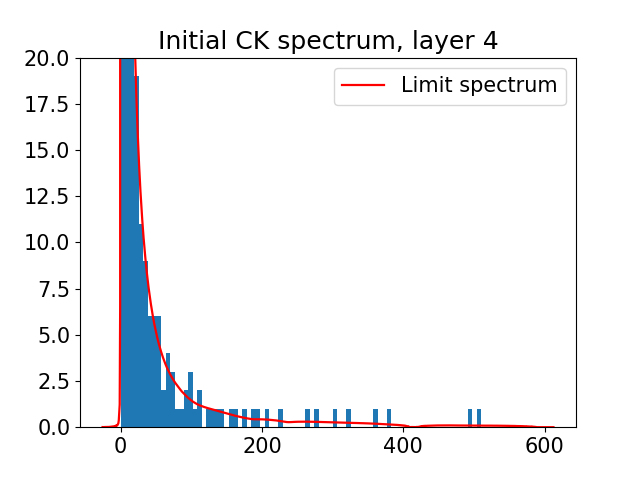}%
\xincludegraphics[width=0.33\textwidth,label=b)]{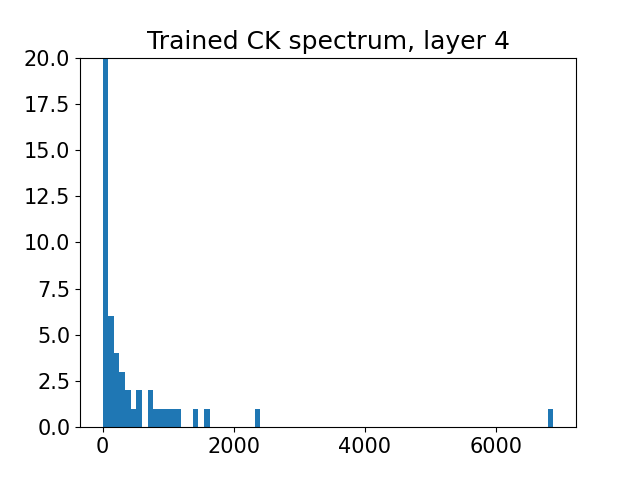}\\

Panel (a) above shows the eigenvalue distribution of $K^\CK$ at random initialization,
with the largest eigenvalue being approximately 500. We
observe a close agreement with the limit spectrum of Theorem \ref{thm:CK}.
Panel (b) shows the eigenvalues of $K^\CK$ after training. We observe
an elongation of the bulk spectral support and the emergence of large
outlier eigenvalues, analogous to the synthetic example of Section
\ref{sec:training}.

\xincludegraphics[width=0.33\textwidth,label=a)]{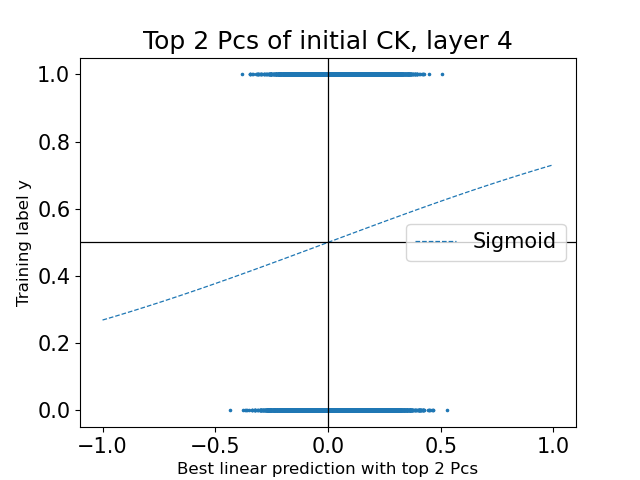}%
\xincludegraphics[width=0.33\textwidth,label=b)]{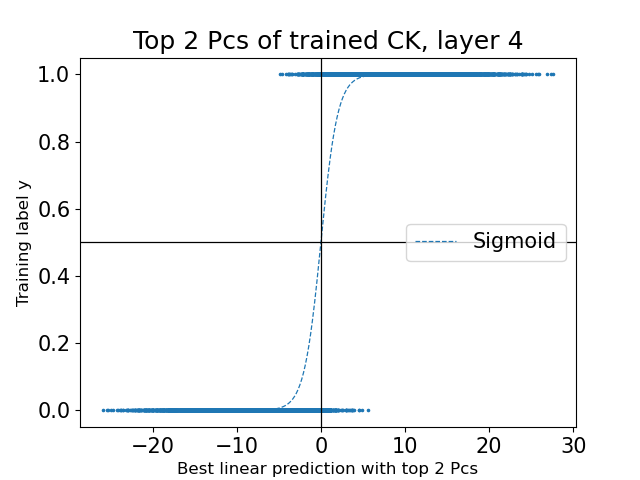}\\

The above figure depicts the information about the training labels that is
contained in the top 2 PCs of $K^\CK$, (a) before training
and (b) after training. Denoting by $\hat{X}_L$ the rank-2 approximation of
$X_L$, with columns
$\hat{\x}_1^L,\ldots,\hat{\x}_n^L$ (both before and after training),
we re-fit a linear binary classifier $y_\a=\sigma(\w^\top \hat{\x}_\a^L+b)$
of the training labels to these columns. The in-sample 0--1 training accuracy of
this classifier is 51.4\% pre-training and 96.8\% post-training, and the 
figure shows the linear predictions $\w^\top \hat{\x}_\a^L+b$ against the
training labels $y_\a$. We observe that the leading principal components of
$K^\CK$ are not predictive of the training labels before training, but
become highly predictive after training.
\end{document}